\documentclass[twoside,11pt]{article}

\usepackage[preprint]{jmlr2e}

\usepackage{lastpage}

\ShortHeadings{RKHM and KME}{Hashimoto, Ishikawa, Ikeda, Komura, Katsura and Kawahara}

\firstpageno{1}

\usepackage{url}            
\usepackage{amsfonts}       
\usepackage{amsmath,amssymb}
\usepackage{nccmath}
\usepackage{mathrsfs}
\usepackage{algorithm,algorithmicx}
\usepackage{algpseudocode}
\usepackage{ifthen}
\usepackage{graphicx}
\usepackage{subfigure}
\usepackage{tabularx}
\usepackage{color}
\usepackage[all]{xy}
\usepackage{tikz}
\usepackage{lastpage}
\jmlrheading{22}{2021}{1-\pageref{LastPage}}{11/20; Revised
8/21}{11/21}{20-1346}{Yuka Hashimoto, Isao Ishikawa, Masahiro Ikeda, Fuyuta Komura, Takeshi Katsura, and Yoshinobu Kawahara}
\ShortHeadings{Reproducing kernel Hilbert $C^*$-module and kernel mean embeddings}{Hashimoto, Ishikawa, Ikeda, Komura, Katsura, and Kawahara}
\editor{Corinna Cortes}

\makeatletter 
 \newcommand{\figcaption}[1]{\def\@captype{figure}\caption{#1}} 
 \renewcommand{\@biblabel}[1]{[#1]}

\newtheorem{prop}{Proposition}[section]
\newtheorem{lem}[prop]{Lemma}
\newtheorem{thm}[prop]{Theorem}

\newtheorem{defin}[prop]{Definition}
\newtheorem{rem}[prop]{Remark}

\newtheorem{cor}[prop]{Corollary}
\newtheorem{exam}[prop]{Example}
\newtheorem{assum}[prop]{Assumption}

\newenvironment{myproof}[1]{\par\noindent{\bf {#1}\ }}{\hfill\BlackBox\\[2mm]}

\makeatother

\def\mcl#1{\mathcal{#1}}
\def\blacket#1{\left\langle #1\right\rangle}
\def\hil{\mcl{H}}
\def\nn{\nonumber}
\def\opn{\operatorname}

\def\alg{\mcl{A}}
\def\modu{\mcl{M}}
\def\mat{\mathbb{C}^{m\times m}}
\def\blin#1{\mcl{B}(#1)}
\def\bblacket#1{\big\langle #1\big\rangle}
\def\bblacketg#1{\bigg\langle #1\bigg\rangle}
\def\Bblacket#1{\bigg\langle #1\bigg\rangle}
\def\sblacket#1{\langle #1\rangle}
\def\clch{{C}_0(\mcl{X},\alg)}
\def\total{\mcl{T}(\mcl{X},\alg)}
\def\simple{\mcl{S}(\mcl{X},\alg)}
\def\bochner#1{{L}^1_{#1}(\mcl{X},\alg)}
\def\measure{\mcl{D}(\mcl{X},\alg)}
\def\measuremat{\mcl{D}(\mcl{X},\mat)}
\def\bdmeasure#1{\mcl{D}_{#1}(\mcl{X},\alg)}
\def\regular{\mcl{R}_+(\mcl{X})}
\def\bc{\mathbf{c}}
\def\bG{\mathbf{G}}
\def\bra#1{\langle#1\vert}
\def\ket#1{\vert#1\rangle}
\def\blackett#1#2{\langle#1\vert#2\rangle}
\def\vvrkhs{\hil_k^{\opn{v}}}
\def\rkhs{\hil_{\tilde{k}}}
\def\red#1{\textcolor{black}{#1}}

\def\motiv{black}
\definecolor{dark-blue}{rgb}{0,0,0.55}
\def\exp{black}
\definecolor{dark-green}{rgb}{0.33,0.42,0.18}
\def\cont{black}

\begin{document}

\title{Reproducing kernel Hilbert $C^*$-module and kernel mean embeddings}

\author{\name Yuka Hashimoto \email yuka.hashimoto.rw@hco.ntt.co.jp \\
       \addr NTT Network Service Systems Laboratories, NTT Corporation\\
       3-9-11, Midori-cho, Musashinoshi, Tokyo, 180-8585, Japan / \\
       Graduate School of Science and Technology, Keio University\\
       3-14-1, Hiyoshi, Kohoku, Yokohama, Kanagawa, 223-8522, Japan
       \AND
       \name Isao Ishikawa \email ishikawa.isao.zx@ehime-u.ac.jp\\
       \addr Center for Data Science, Ehime University\\
       2-5, Bunkyo-cho, Matsuyama, Ehime, 790-8577, Japan / \\
       Center for Advanced Intelligence Project, RIKEN\\
       1-4-1, Nihonbashi, Chuo-ku, Tokyo 103-0027, Japan
       \AND
       \name Masahiro Ikeda \email masahiro.ikeda@riken.jp \\
       \addr Center for Advanced Intelligence Project, RIKEN\\
       1-4-1, Nihonbashi, Chuo-ku, Tokyo 103-0027, Japan / \\
       Faculty of Science and Technology, Keio University\\
       3-14-1, Hiyoshi, Kohoku, Yokohama, Kanagawa, 223-8522, Japan
       \AND
       \name Fuyuta Komura \email fuyuta.k@keio.jp \\
       \addr Faculty of Science and Technology, Keio University\\
       3-14-1, Hiyoshi, Kohoku, Yokohama, Kanagawa, 223-8522, Japan / \\ 
       Center for Advanced Intelligence Project, RIKEN\\
       1-4-1, Nihonbashi, Chuo-ku, Tokyo 103-0027, Japan
       \AND 
       \name Takeshi Katsura \email katsura@math.keio.ac.jp \\
       \addr Faculty of Science and Technology, Keio University\\
       3-14-1, Hiyoshi, Kohoku, Yokohama, Kanagawa, 223-8522, Japan / \\ 
       Center for Advanced Intelligence Project, RIKEN\\
       1-4-1, Nihonbashi, Chuo-ku, Tokyo 103-0027, Japan
       \AND
       \name Yoshinobu Kawahara \email kawahara@imi.kyushu-u.ac.jp\\
       \addr Institute of Mathematics for Industry, Kyushu University\\
       744, Motooka, Nishi-ku, Fukuoka, 819-0395, Japan / \\
       Center for Advanced Intelligence Project, RIKEN\\
       1-4-1, Nihonbashi, Chuo-ku, Tokyo 103-0027, Japan
}

\maketitle

\begin{abstract}
Kernel methods have been among the most popular techniques in machine learning, where learning tasks are solved using the property of reproducing kernel Hilbert space (RKHS). 
In this paper, we propose a novel data analysis framework with reproducing kernel Hilbert $C^*$-module (RKHM) and kernel mean embedding (KME) in RKHM. 
Since RKHM contains richer information than RKHS or vector-valued RKHS (vvRKHS), analysis with RKHM enables us to capture and extract structural properties in such as functional data. 
We show a branch of theories for RKHM to apply to data analysis, including the representer theorem, and the injectivity and universality of the proposed KME.
We also show RKHM generalizes RKHS and vvRKHS.
Then, we provide concrete procedures for employing RKHM and the proposed KME to data analysis.
\end{abstract}



\section{Introduction}\label{secintro}
Kernel methods have been among the most popular techniques in machine learning~\citep{scholkopf01}, where learning tasks are solved using the property of reproducing kernel Hilbert space (RKHS). 
RKHS is the space of complex-valued functions equipped with an inner product determined by a positive-definite kernel. 
One of the important tools with RKHS is kernel mean embedding (KME).
In KME, a probability distribution (or measure) is embedded as a function in an RKHS~\citep{solma07,kernelmean,sriperumbudur11}, which enables us to analyze distributions in RKHSs.

Whereas much of the classical literature on RKHS approaches has focused on complex-valued functions, RKHSs of vector-valued functions, i.e., vector-valued RKHSs (vvRKHSs), 
have also been proposed~\citep{micchelli05,mauricio11,lim15,quang16,kadri16}. 
This allows us to learn vector-valued functions rather than complex-valued functions.
%

In this paper, we develop a branch of theories on reproducing kernel Hilbert $C^*$-module (RKHM) and propose a generic framework for data analysis with RKHM.
RKHM is a generalization of RKHS and vvRKHS in terms of $C^*$-algebra, and we show that RKHM is a powerful tool to analyze structural properties in {such as functional data}. 
An RKHM is constructed by a $C^*$-algebra-valued positive definite kernel and characterized by a $C^*$-algebra-valued inner product (see Definition~\ref{def:pdk_rkhm}).
The theory of $C^*$-algebra has been discussed in mathematics, especially in operator algebra theory.
An important example of $C^*$-algebra is $L^{\infty}(\Omega)$, where $\Omega$ is a compact measure space.
Another important example is $\blin{\mcl{W}}$, which denotes the space of bounded linear operators on a Hilbert space $\mcl{W}$.
Note that $\blin{\mcl{W}}$ coincides with the space of matrices $\mat$ if the Hilbert space $\mcl{W}$ is finite dimensional.

Although there are several advantages for studying RKHM compared with RKHS and vvRKHS, those can be summarized into two points as follows:
First, an RKHM is a ``Hilbert $C^*$-module'', which is mathematically more general than a ``Hilbert space''. 
The inner product in an RKHM is $C^*$-algebra-valued, which captures more information than the complex-valued one in an RKHS or vvRKHS and enables us to extract richer information.
For example, 
{if we set $L^{\infty}(\Omega)$ as a $C^*$-algebra, we can control and extract features of functional data such as derivatives, total variation, and frequency components.
Also, if we set $\blin{\mcl{W}}$ as a $C^*$-algebra and the inner product is described by integral operators, we can control and extract features of continuous relationships between pairs of functional data.}
This cannot be achieved, in principle, by RKHSs and vv-RKHSs.
This is because their inner products are complex-valued, where such information degenerates into one complex value {or is lost by discretizations of function into complex values}.
Therefore, we cannot reconstruct the information from a vector in an RKHS or vvRKHS.
Second, RKHM generalizes RKHS and vvRKHS, 
that is, it can be shown that we can reconstruct RKHSs and vvRKHSs from RKHMs.
This implies that existing algorithms with RKHSs and vvRKHSs are reconstructed by using the framework of RKHM.

The theory of RKHM has been studied in mathematical physics and pure mathematics~\citep{itoh90,heo08,szafraniec10}.
On the other hand, to the best of our knowledge, as for the application of RKHM to data analysis, we can find the only literature by \citet{ye17}, where only the case of setting the space of matrices as a $C^*$-algebra is discussed.
In this paper, we develop a branch of theories on RKHM and propose a generic framework for data analysis with RKHM.
We show a theoretical property on minimization with respect to orthogonal projections and give a representer theorem in RKHMs.
These properties are fundamental for data analysis that have been investigated and applied in the cases of RKHS and vvRKHS, which has made RKHS and vvRKHS widely-accepted tools for data analysis~\citep{scholkopf01_representer}.
Moreover, we define a KME in an RKHM, and provide theoretical results about the injectivity of the proposed KME and the connection with universality of RKHM.
Note that, as is well known for RKHSs, these two properties have been actively studied to theoretically guarantee the validity of kernel-based algorithms~\citep{steinwart01,gretton07,fukumizu08,sriperumbudur11}. 
Then, we apply the developed theories to generalize kernel PCA~\citep{scholkopf01}, \textcolor{black}{analyze time-series data with the theory of dynamical system}, and analyze interaction effects for infinite dimensional data. 



The remainder of this paper is organized as follows. 
First, in Section~\ref{sec:bg}, we briefly review RKHS, vvRKHS, and \textcolor{black}{the definition of RKHM}. 
\textcolor{black}{In Section~\ref{sec:motivation}, we provide an overview of the motivation of studying RKHM for data analysis.}
In Section~\ref{sec:RKHM}, we show {general properties of RKHM for data analysis and the connection of RKHMs with RKHSs and vvRKHSs}.
In Sections~\ref{sec:kme}, we propose a KME in RKHMs, and show the connection between the injectivity of the KME and the universality of RKHM.
Then, in Section~\ref{sec:application}, we discuss applications of the developed results to kernel PCA, \textcolor{black}{time-series data analysis}, and the analysis of interaction effects in finite or infinite dimensional data.
Finally, in Section~\ref{sec:existing}, we discuss the connection of RKHMs and the proposed KME with the existing notions, and conclude the paper in Section~\ref{sec:concl}.

\paragraph{Notations}
Lowercase letters denote $\alg$-valued coefficients (often by $a,b,c,d$), vectors in a Hilbert $C^*$-module $\modu$ (often by $p,q,u,v$), or vectors in a Hilbert space $\mcl{W}$ (often by $w,h$).
Lowercase Greek letters denote measures (often by $\mu,\nu,\lambda$) or complex-valued coefficients (often by $\alpha,\beta$).
Calligraphic capital letters denote sets.
And, bold lowercase letters denote vectors in $\alg^n$ for $n\in\mathbb{N}$ (a finite dimensional Hilbert $C^*$-module).
Also, we use $\sim$ for objects related to RKHSs.
Moreover, an inner product, an absolute value, and a norm in a space or a module $\mcl{S}$ (see Definitions~\ref{def:innerproduct} and \ref{def:absolute_norm}) are denoted as $\blacket{\cdot,\cdot}_{\mcl{S}}$, $\vert\cdot\vert_{\mcl{S}}$, and $\Vert\cdot\Vert_{\mcl{S}}$, respectively.

The typical notations in this paper are listed in Table~\ref{tab1}.

\begin{table}[t]
\caption{Notation table}\label{tab1}
\vspace{.3cm}
\renewcommand{\arraystretch}{1.1}
 \begin{tabularx}{\linewidth}{|c|X|}
\hline
$\alg$ & A $C^*$-algebra\\
$1_{\alg}$ & The multiplicative identity in $\alg$\\
$\alg_+$ & The subset of $\alg$ composed of all positive elements in $\alg$\\
$\le_{\alg}$ & For $c,d\in\alg$, $c\le_{\alg}d$ means $d-c$ is positive.\\
$<_{\alg}$ & For $c,d\in\alg$, $c<d$ means $d-c$ is strictly positive, i.e., $d-c$ is positive and invertible.\\
$L^{\infty}(\Omega)$ & The space of complex-valued $L^{\infty}$ functions on a measure space $\Omega$\\
$\blin{\mcl{W}}$ & The space of bounded linear operators on a Hilbert space $\mcl{W}$\\
$\mat$  & A set of all complex-valued $m\times m$ matrix\\
$\modu$ & A Hilbert $\alg$-module\\
$\mcl{X}$ & A nonempty set for data\\
$C(\mcl{X},\mcl{Y})$ & The space of $\mcl{Y}$-valued continuous functions on $\mcl{X}$ for topological spaces $\mcl{X}$ and $\mcl{Y}$\\
$n$ & A natural number that represents the number of samples\\
$k$ & An $\alg$-valued positive definite kernel\\
$\phi$ & The feature map endowed with $k$\\
$\modu_k$ & The RKHM associated with $k$\\
$\mcl{S}^{\mcl{X}}$ & The set of all functions from a set $\mcl{X}$ to a space $\mcl{S}$\\
$\tilde{k}$ & A complex-valued positive definite kernel\\
$\tilde{\phi}$ & The feature map endowed with $\tilde{k}$\\
$\hil_{\tilde{k}}$ & The RKHS associated with $\tilde{k}$\\
$\vvrkhs$ & The vvRKHS associated with $k$\\
$\measure$ & The set of all $\alg$-valued finite regular Borel measures\\
$\Phi$ & The proposed KME in an RKHM\\
$\delta_x$ & The $\alg$-valued Dirac measure defined as $\delta_x(E)=1_{\alg}$ for $x\in E$ and $\delta_x(E)=0$ for $x\notin E$\\
$\tilde{\delta}_x$ & The complex-valued Dirac measure defined as $\tilde{\delta}_x(E)=1$ for $x\in E$ and $\tilde{\delta}_x(E)=0$ for $x\notin E$\\
$\chi_E$ & The indicator function of a Borel set $E$ on $\mcl{X}$\\
$\clch$ & The space of all continuous $\alg$-valued functions on $\mcl{X}$ vanishing at infinity\\
$\mathbf{G}$ & The $\alg$-valued Gram matrix defined as $\bG_{i,j}=k(x_i,x_j)$ for given samples $x_1,\ldots,x_n\in\mcl{X}$\\
$p_j$ & The $j$-th principal axis generated by kernel PCA with an RKHM\\
$r$ &  A natural number that represents the number of principal axes\\
$Df_{\bc}$ & The G\^{a}teaux derivative of a function $f:\modu\to\alg$ at $\bc\in\modu$\\
$\nabla f_{\bc}$ & The gradient of a function $f:\modu\to\alg$ at $\bc\in\modu$\\
\hline
 \end{tabularx}
\renewcommand{\arraystretch}{1} 
\end{table}

\section{Background}
\label{sec:bg}
We briefly review RKHS and vvRKHS in Subsections~\ref{subsec:rkhs} and \ref{subsec:vv-rkhs}, respectively.
Then, we review $C^*$-algebra and $C^*$-module in Subsection~\ref{subsec:c_alg}, Hilbert $C^*$-module in Subsection~\ref{subsec:H_c_module}, and RKHM in Subsection~\ref{sec:rkhm_review}.

\subsection{Reproducing kernel Hilbert space (RKHS)}\label{subsec:rkhs}
We review the theory of RKHS.
An RKHS is a Hilbert space and useful for extracting nonlinearity or higher-order moments of data~\citep{scholkopf01,saitoh16}.

We begin by introducing positive definite kernels.
Let $\mcl{X}$ be a non-empty set for data, and $\tilde{k}$ be a positive definite kernel, which is defined as follows:
\begin{defin}[Positive definite kernel]\label{def:pdk_rkhs}
A map $\tilde{k}:\mcl{X}\times \mcl{X}\to\mathbb{C}$ is called a {\em positive definite kernel} if it satisfies the following conditions:
\begin{enumerate}
    \item $\tilde{k}(x,y)=\overline{\tilde{k}(y,x)}$\; for $x,y\in\mcl{X}$,
    \item $\sum_{i,j=1}^{n}\overline{\alpha}_i\alpha_j\tilde{k}(x_i,x_j)\ge 0$\; for $n\in\mathbb{N}$, $\alpha_i\in\mathbb{C}$, $x_i\in\mcl{X}$.
\end{enumerate}
\end{defin}
Let $\tilde{\phi}:\mcl{X}\to\mathbb{C}^{\mcl{X}}$ be a map defined as $\tilde{\phi}(x)=\tilde{k}(\cdot,x)$.
With $\tilde{\phi}$, the following space as a subset of $\mathbb{C}^{\mcl{X}}$ is constructed: 
\begin{equation*}
\hil_{\tilde{k},0}:=\bigg\{\sum_{i=1}^{n}\alpha_i\tilde{\phi}(x_i)\bigg|\ n\in\mathbb{N},\ \alpha_i\in\mathbb{C},\ x_i\in\mcl{X}\bigg\}.
\end{equation*}
Then, a map $\blacket{\cdot,\cdot}_{\hil_{\tilde{k}}}:\hil_{\tilde{k},0}\times \hil_{\tilde{k},0}\to\mathbb{C}$ is defined as follows:
\begin{equation*}
\Bblacket{\sum_{i=1}^{n}\alpha_i\tilde{\phi}(x_i),\sum_{j=1}^{l}\beta_j\tilde{\phi}(y_j)}_{\rkhs}:=\sum_{i=1}^{n}\sum_{j=1}^{l}\overline{\alpha_i}\beta_j\tilde{k}(x_i,y_j).
\end{equation*}
By the properties in Definition~\ref{def:pdk_rkhs} of $\tilde{k}$, $\blacket{\cdot,\cdot}_{\rkhs}$ is well-defined, satisfies the axiom of inner products, and has the reproducing property, that is,
\begin{equation*}
\sblacket{\tilde{\phi}(x),v}_{\rkhs}=v(x)
\end{equation*}
for $v\in\hil_{\tilde{k},0}$ and $x\in\mcl{X}$.

The completion of $\hil_{\tilde{k},0}$ 
is called the {\em RKHS} associated with $\tilde{k}$ and denoted as $\hil_{\tilde{k}}$.
It can be shown that $\blacket{\cdot,\cdot}_{\rkhs}$ is extended continuously to $\hil_{\tilde{k}}$ and the map $\hil_{\tilde{k}}\ni v\mapsto (x\mapsto\sblacket{\tilde{\phi}(x),v}_{\rkhs})\in\mathbb{C}^{\mcl{X}}$ is injective.
Thus, $\hil_{\tilde{k}}$ is regarded to be a subset of $\mathbb{C}^{\mcl{X}}$ and has the reproducing property.
Also, $\hil_{\tilde{k}}$ is determined uniquely.

The map $\tilde{\phi}$ maps data into $\hil_{\tilde{k}}$ and is called the {\em feature map}.
Since the dimension of $\hil_{\tilde{k}}$ is higher (often infinite dimensional) than that of $\mcl{X}$, complicated behaviors of data in $\mcl{X}$ are expected to be transformed into simple ones in $\hil_{\tilde{k}}$~\citep{scholkopf01}.

\subsection{Vector-valued RKHS (vvRKHS)}\label{subsec:vv-rkhs}
We review the theory of vvRKHS. Complex-valued functions in RKHSs are generalized to vector-valued functions in vvRKHSs.
Similar to the case of RKHS, we begin by introducing positive definite kernels.
Let $\mcl{X}$ be a non-empty set for data and $\mcl{W}$ be a Hilbert space.
In addition, let $k$ be an operator-valued positive definite kernel, which is defined as follows:
\begin{defin}[Operator-valued positive definite kernel]\label{def:pdk_vv-rkhs}
A map $k:\mcl{X}\times \mcl{X}\to\mcl{B}(\mcl{W})$ is called an {\em operator-valued positive definite kernel} if it satisfies the following conditions:
\begin{enumerate}
    \item $k(x,y)=k(y,x)^*$\; for $x,y\in\mcl{X}$,
    \item $\sum_{i,j=1}^{n}\blacket{w_i,k(x_i,x_j)w_j}_{\mcl{W}}\ge 0$\; for $n\in\mathbb{N}$, $w_i\in\mcl{W}$, $x_i\in\mcl{X}$.
\end{enumerate}
Here, $^*$ represents the adjoint.
\end{defin}
Let $\phi:\mcl{X}\to\blin{\mcl{W}}^{\mcl{X}}$ be a map defined as $\phi(x)=k(\cdot,x)$.
With $\phi$, the following space as a subset of $\mcl{W}^{\mcl{X}}$ is constructed: 
\begin{equation*}
\hil_{k,0}^{\opn{v}}:=\bigg\{\sum_{i=1}^n\phi(x_i)w_i\bigg|\ n\in\mathbb{N},\ w_i\in\mcl{W},\ x_i\in\mcl{X}\bigg\}.
\end{equation*}
Then, a map $\blacket{\cdot,\cdot}_{\vvrkhs}:\hil_{k,0}^{\opn{v}}\times \hil_{k,0}^{\opn{v}}\to\mathbb{C}$ is defined as follows:
\begin{align*}
&\Bblacket{\sum_{i=1}^n\phi(x_i)w_i,\sum_{j=1}^l\phi(y_j)h_j}_{\vvrkhs}
:=\sum_{i=1}^n\sum_{j=1}^l\blacket{w_i,k(x_i,y_j)h_j}_{\mcl{W}}.
\end{align*}
By the properties in Definition~\ref{def:pdk_vv-rkhs} of $k$, $\blacket{\cdot,\cdot}_{\vvrkhs}$ is well-defined, satisfies the axiom of inner products, and has the reproducing property, that is,
\begin{equation}
\blacket{\phi(x)w,u}_{\vvrkhs}=\blacket{w,u(x)}_{\mcl{W}}\label{eq:reproducing_vvRKHS}
\end{equation}
for $u\in\hil_{k,0}^{{\opn{v}}}$, $x\in\mcl{X}$, and $w\in\mcl{W}$.

The completion of $\hil_{k,0}^{\opn{v}}$ 
is called the {\em vvRKHS} associated with $k$ and denoted as $\hil_k^{\opn{v}}$. 
Note that since an inner product in $\hil_k^{\opn{v}}$ is defined with the complex-valued inner product in $\mcl{W}$, it is complex-valued.

\subsection{$C^*$-algebra and Hilbert $C^*$-module}
\label{subsec:c_alg}
A $C^*$-algebra and a $C^*$-module are 
generalizations of the space of complex numbers $\mathbb{C}$ and a vector space, respectively.  
In this paper, we denote a $C^*$-algebra by $\mathcal{A}$ and a $C^*$-module by $\modu$, respectively. 
As we see below, many complex-valued notions can be generalized to $\mathcal{A}$-valued. 

A $C^*$-algebra is defined as a Banach space equipped with a product structure, an involution  $(\cdot)^*:\mathcal{A}\rightarrow\mathcal{A}$, \red{and additional properties}.
We denote the norm of $\mathcal{A}$ by $\Vert\cdot\Vert_\mathcal{A}$.
\color{black}
\begin{defin}[Algebra]
A set $\alg$ is called an {\em algebra} on a filed $\mathbb{F}$ if it is a vector space equipped with an operation $\cdot:\alg\times\alg\to\alg$ which satisfies the following conditions for $b,c,d\in\alg$ and $\alpha\in\mathbb{F}$:\vspace{.2cm}

 \leftskip=10pt
 $\bullet$ $(b+c)\cdot d=\red{b}\cdot d+c\cdot d$,\qquad
 $\bullet$ $b\cdot(c+d)=b\cdot c+b\cdot d$,\qquad
 $\bullet$ $(\alpha c)\cdot d=\alpha(c\cdot d)=c\cdot(\alpha d)$.\vspace{.2cm}
\leftskip=0pt

\noindent{The symbol $\cdot$ is omitted when it does not cause confusion.}
\end{defin}
\color{black}
\begin{defin}[$C^*$-algebra]~\label{def:c*_algebra}
A set $\mcl{A}$ is called a {\em $C^*$-algebra} if it satisfies the following conditions:

\begin{enumerate}
 \item $\mcl{A}$ is an algebra over $\mathbb{C}$ and \red{equipped with} a bijection $(\cdot)^*:\mcl{A}\to\mcl{A}$ that satisfies the following conditions for $\alpha,\beta\in\mathbb{C}$ and $c,d\in\mcl{A}$:

 \leftskip=10pt
 $\bullet$ $(\alpha c+\beta d)^*=\overline{\alpha}c^*+\overline{\beta}d^*$,\qquad
 $\bullet$ $(cd)^*=d^*c^*$,\qquad
 $\bullet$ $(c^*)^*=c$.

 \leftskip=0pt
 \item $\mcl{A}$ is a normed space with $\Vert\cdot\Vert_{\alg}$, and for $c,d\in\mcl{A}$, $\Vert cd\Vert_{\alg}\le\Vert c\Vert_{\alg}\Vert d\Vert_{\alg}$ holds.
 In addition, $\mcl{A}$ is complete with respect to $\Vert\cdot\Vert_{\alg}$.

 \item For $c\in\mcl{A}$, $\Vert c^*c\Vert_{\alg}=\Vert c\Vert_{\alg}^2$ holds.
\end{enumerate}
\end{defin} 
\begin{defin}[Multiplicative identity and unital $C^*$-algebra]\label{def:multiplicative_identity}
The {\em multiplicative identity} of $\alg$ is the element $a\in\alg$ which satisfies $ac=ca=c$ for any $c\in\alg$.
We denote by $1_{\alg}$ the multiplicative identity of $\alg$.
If a $C^*$-algebra $\alg$ has the multiplicative identity, then it is called a {\em unital $C^*$-algebra}.
\end{defin}
%
\begin{exam}\label{ex:vonNeumann}
Important examples of (unital) $C^*$-algebras are $L^{\infty}(\Omega)$ and $\blin{\mcl{W}}$, i.e., the space of complex-valued $L^{\infty}$ functions on a \red{$\sigma$-finite measure space} $\Omega$ and  the space of bounded linear operators on a Hilbert space $\mcl{W}$, respectively.
\begin{enumerate}
\item For $\alg=L^{\infty}(\Omega)$, the product of two functions $c,d\in\alg$ is defined as $(cd)(t)=c(t)d(t)$ for any $t\in\Omega$, the involution is defined as $c(t)=\overline{c(t)}$, the norm is the $L^{\infty}$-norm, and the multiplicative identity is the constant function whose value is $1$ at almost everywhere $t\in\Omega$.
\item For $\alg=\mcl{B}(\mcl{W})$, the product structure is the product (the composition) of operators,
the involution is the adjoint,
the norm $\Vert\cdot\Vert_\mathcal{A}$ is the operator norm, and the multiplicative identity is the identity map.
\end{enumerate}
\end{exam}
\textcolor{black}{In fact, by the Gelfand--Naimark theorem (see, for example, \citet{murphy90}), any $C^*$-algebra can be regarded as a subalgebra of $\mcl{B}(\mcl{W})$ for some Hilbert space $\mcl{W}$.
Therefore, considering the case of $\alg=\mcl{B}(\mcl{W})$ is sufficient for applications.}

{\em The positiveness} is also important in $C^*$-algebras. 
\begin{defin}[Positive]~\label{def:positive}
An element $c$ of $\mcl{A}$ is called {\em positive} if there exists $d\in\mcl{A}$ such that $c=d^*d$ holds.
\textcolor{black}{For a unital $C^*$-algebra $\alg$}, if a positive element $c\in\alg$ is invertible, i.e., there exists $d\in\alg$ such that $cd=dc=1_{\alg}$, then $c$ is called {\em strictly positive}.
For $c,d\in\alg$, we denote $c\le_{\alg} d$ if $d-c$ is positive and $c<_{\alg}d$ if $d-c$ is strictly positive.
We denote by $\alg_+$ the subset of $\alg$ composed of all positive elements in $\alg$.
\end{defin}
\begin{exam}\label{ex:positive}
\begin{enumerate}
\item For $\alg=L^{\infty}(\Omega)$, a function $c\in\alg$ is positive if and only if $c(t)\ge 0$ for almost everywhere $t\in\Omega$, and strictly positive if and only if $c(t)>0$ for almost everywhere $t\in\Omega$.
\item For $\mathcal{A}=\mcl{B}(\mcl{W})$, the positiveness is equivalent to the positive semi-definiteness of \red{self-adjoint} operators and the strictly positiveness is equivalent to the positive definiteness of \red{self-adjoint} operators.
\end{enumerate}
\end{exam}
The positiveness provides us the (pre) order in $\mathcal{A}$ and, thus, enables us to consider optimization problems in $\mathcal{A}$.
\begin{defin}[Supremum and infimum]\label{def:sup}
\begin{enumerate}
\item For a subset $\mcl{S}$ of $\alg$, $a\in\alg$ is said to be an {\em upper bound} with respect to the order $\le_{\alg}$, if $d\le_{\alg}a$ for any $d\in\mcl{S}$.
Then, $c\in\alg$ is said to be a {\em supremum} of $\mcl{S}$, if $c\le_{\alg} a$ for any upper bound $a$ of $\mcl{S}$.
\item For a subset $\mcl{S}$ of $\alg$, $a\in\alg$ is said to be a {\em lower bound} with respect to the order $\le_{\alg}$, if $a\le_{\alg} d$ for any $d\in\mcl{S}$.
Then, $c\in\alg$ is said to be a {\em infimum} of $\mcl{S}$, if $a\le_{\alg} c$ for any lower bound $a$ of $\mcl{S}$. 
\end{enumerate}
\end{defin}
%
We now introduce a $C^*$-module over $\alg$, which is a generalization of the vector space.
\begin{defin}[Right multiplication]\label{def:multiplication}
Let $\modu$ be an abelian group with operation $+$.
For $c,d\in\alg$ and $u,v\in\modu$, if an operation $\cdot:\modu\times\alg\to\modu$ satisfies
\begin{enumerate}
\item $(u+v)\cdot c=u\cdot c+v\cdot c$,
\item $u\cdot (c+d)=u\cdot c+u\cdot d$,
\item $u\cdot (cd)=(u\cdot \red{c})\cdot \red{d}$,
\item $u\cdot 1_{\alg}=u$ \textcolor{black}{if $\alg$ is unital},
\end{enumerate}
then, $\cdot$ is called a (right) {\em $\alg$-multiplication}.
The multiplication $u\cdot c$ is usually denoted as $uc$.
\end{defin}
\begin{defin}[$C^*$-module]\label{def:c*module}
Let $\modu$ be an abelian group with operation $+$.
If $\modu$ \red{is equipped with} a (right) $\alg$-multiplication, $\modu$ is called a (right) {\em $C^*$-module} over $\alg$.
\end{defin}
%
In this paper, we consider column vectors rather than row vectors for representing $\alg$-valued coefficients, and column vectors act on the right.
Therefore, we consider right multiplications.
However, considering row vectors and left multiplications instead of column vectors and right multiplications is also possible.

\subsection{Hilbert $C^*$-module}\label{subsec:H_c_module}
A Hilbert $C^*$-module is a generalization of a Hilbert space.
We first consider an $\alg$-valued inner product, which is a generalization of a complex-valued inner product, and then, introduce the definition of a Hilbert $C^*$-module.
\begin{defin}[$\alg$-valued inner product]\label{def:innerproduct}
A \red{$\mathbb{C}$-linear map with respect to the second variable} $\blacket{\cdot,\cdot}_{\modu}:\modu\times\modu\to\alg$ is called an $\alg$-valued {\em inner product} if it satisfies the following properties for $u,v,p\in\modu$ 
and $c,d\in\alg$:
\begin{enumerate}
 \item $\blacket{u,vc+pd}_{\modu}=\blacket{u,v}_{\modu}c+\blacket{u,p}_{\modu}d$,
 \item $\blacket{v,u}_{\modu}=\blacket{u,v}_{\modu}^*$,
 \item $\blacket{u,u}_{\modu}\ge_{\alg} 0$,
 \item If $\blacket{u,u}_{\modu}=0$ then $u=0$.
\end{enumerate}
\end{defin}

\begin{defin}[$\alg$-valued absolute value and norm]\label{def:absolute_norm}
For $u\in\modu$, the {\em $\alg$-valued absolute value} $\vert u\vert_{\modu}$ on $\modu$ is defined by the positive element $\vert u \vert_{\modu}$ of $\alg$ such that $\vert u\vert_{\modu}^2=\blacket{u,u}_{\modu}$.   
The (real-valued) norm $\Vert \cdot\Vert_{\modu}$ on $\modu$ is defined by $\Vert u\Vert_{\modu} =\big\Vert\vert u\vert_{\modu}\big\Vert_\alg$. 
\end{defin}
Since the absolute value $\vert\cdot\vert_{\modu}$ takes values in $\alg$, it behaves more complicatedly.
\textcolor{black}{For example, the triangle inequality does not hold for the absolute value}.
However, it provides us with more information than the norm $\Vert\cdot\Vert_{\modu}$ (which is real-valued).
For example, let $\modu={\alg=}\mat$, $c=\opn{diag}\{\alpha,0,\ldots,0\}$, and $d=\opn{diag}\{\alpha,\ldots,\alpha\}$, where $\alpha\in \mathbb{C}$. 
Then, $\Vert c\Vert_{\modu}=\Vert d\Vert_{\modu}$, but $\vert c\vert_{\modu}\neq\vert d\vert_{\modu}$.
For a self-adjoint matrix, the absolute value describes the whole spectrum of it, but the norm only describes the largest eigenvalue.
\begin{defin}[Hilbert $C^*$-module]\label{def:hil_c*module}
Let $\modu$ be a (right) $C^*$-module over $\alg$ equipped with an $\alg$-valued inner product defined in Definition~\ref{def:innerproduct}.
If $\modu$ is complete with respect to the norm $\Vert \cdot\Vert_{\modu}$, it is called a {\em Hilbert $C^*$-module} over $\alg$ or {\em Hilbert $\alg$-module}.
\end{defin}
\begin{exam}\label{ex:An}
A simple example of Hilbert $C^*$ modules over $\alg$ is $\alg^n$ for a natural number $n$.
The $\alg$-valued inner product between $\mathbf{c}=[c_1,\ldots,c_n]^T$ and $\mathbf{d}=[d_1,\ldots,d_n]^T$ is defined as $\blacket{\mathbf{c},\mathbf{d}}_{\alg^n}=\sum_{i=1}^nc_i^*d_i$.
The absolute value and norm in $\alg^n$ are given as $\vert \mathbf{c}\vert_{\alg^n}^2=(\sum_{i=1}^nc_i^*c_i)$ and $\Vert \mathbf{c}\Vert_{\alg^n}=\Vert \sum_{i=1}^nc_i^*c_i\Vert_{\alg}^{1/2}$, respectively.
\end{exam}
Similar to the case of Hilbert spaces, the following Cauchy--Schwarz inequality for $\alg$-valued inner products is available~\citep[Proposition 1.1]{lance95}.
\begin{lem}[Cauchy--Schwarz inequality]\label{lem:c-s}
For $u,v\in\modu$, the following inequality holds:
\begin{equation*}
\vert\blacket{u,v}_{\modu}\vert_{\alg}^2\;\le_{\alg}\Vert u\Vert_{\modu}^2\blacket{v,v}_{\modu}.
\end{equation*}
\end{lem}
%
An important property associated with an inner product is the orthonormality.
The orthonormality plays an important role in data analysis.
For example, an orthonormal basis constructs orthogonal projections and an orthogonally projected vector minimizes the deviation from its original vector in the projected space.
Therefore, we also introduce the orthonormality in Hilbert $C^*$-module. See, for example, Definition 1.2 in~\citep{bakic01} for more details.
\begin{defin}[Normalized]\label{def:normalized}
A vector $q\in\mcl{M}$ is {\em normalized} if 
 $0\neq\blacket{q,q}_{\modu}=\blacket{q,q}_{\modu}^2$.
\end{defin}
Note that in the case of a general $C^*$-valued inner product, for a normalized vector $q$, $\blacket{q,q}_{\modu}$ is not always equal to the identity of $\alg$ in contrast to the case of a complex-valued inner product.
\begin{defin}[Orthonormal system and basis]\label{def:orthonormal}
Let $\mcl{I}$ be an index set.
A set $\mathcal{S}=\{q_i\}_{i\in\mcl{I}}\subseteq\modu$ is called an {\em orthonormal system (ONS)} of $\mcl{M}$ if 
$q_i$ is normalized for any $i\in\mcl{I}$
and $\blacket{q_i,q_j}_{\modu}=0$ for $i\neq j$. 
We call $\mathcal{S}$ an {\em orthonormal basis (ONB)} if \red{the module generated by $\mathcal{S}$} is an ONS and dense in $\modu$.
\end{defin}
In Hilbert $C^*$-modules, $\alg$-linear is often used instead of $\mathbb{C}$-linear.
\begin{defin}[$\alg$-linear operator]\label{def:a_lin_op}
Let $\modu_1,\modu_2$ be Hilbert $\alg$-modules.
A linear map $L:\modu_1\to\modu_2$ is referred to as {\em $\alg$-linear} if it satisfies $L(uc)=(Lu)c\;$ for any $u\in\modu$ and $c\in\alg$.
\end{defin}
\begin{defin}[$\alg$-linearly independent]\label{def:a_lin_indep}
The set $\mcl{S}$ of $\modu$ is said to be {\em $\alg$-linearly independent} if it satisfies the following condition: For any finite subset $\{v_1,\ldots,v_n\}$ of $\mcl{S}$, if $\sum_{i=1}^nv_ic_i=0$ for $c_i\in\alg$, then $c_i=0$ for $i=1,\ldots,n$.
\end{defin}

For further details about $C^*$-algebra, $C^*$-module, and Hilbert $C^*$-module, refer to~\citet{murphy90,lance95}.

\subsection{Reproducing kernel Hilbert $C^*$-module (RKHM)}\label{sec:rkhm_review}
We summarize the theory of RKHM, which is discussed, for example, in~\citet{heo08}. 

Similar to the case of RKHS, we begin by introducing an $\alg$-valued generalization of a positive definite kernel on a non-empty set $\mcl{X}$ for data.
\begin{defin}[$\alg$-valued positive definite kernel]\label{def:pdk_rkhm}
 An $\mcl{A}$-valued map $k:\ \mcl{X}\times \mcl{X}\to\mcl{A}$ is called a {\em positive definite kernel} if it satisfies the following conditions: 
\begin{enumerate}
 \item $k(x,y)=k(y,x)^*$ \;for $x,y\in\mcl{X}$,
 \item $\sum_{i,j=1}^nc_i^*k(x_i,x_j)c_j\ge_{\alg} 0$ \;for $n\in\mathbb{N}$, $c_i\in\alg$, $x_i\in\mcl{X}$.
\end{enumerate}
\end{defin}
\begin{exam}\label{ex:pdk1}
\begin{enumerate}
\item Let $\mcl{X}=C([0,1]^m)$.
Let $\alg=L^{\infty}([0,1])$ and let $k:\mcl{X}\times\mcl{X}\to \alg$ be defined as $k(x,y)(t)=\int_{[0,1]^m}\red{\overline{(t-x(s))}}(t-y(s))ds$ for $t\in[0,1]$.
Then, for $x_1,\ldots,x_n\in\mcl{X}$, $c_1,\ldots,c_n\in \alg$ and $t\in[0,1]$, we have
\begin{align*}
\sum_{i,j=1}^nc_i^*(t)k(x_i,x_j)(t)c_j(t)&=\int_{[0,1]^m}\sum_{i,j=1}^n\red{\overline{c_i(t)(t-x_i(s))}}(t-x_j(s))c_j(t)ds\\
&=\int_{[0,1]^m}\sum_{i=1}^n\red{\overline{c_i(t)(t-x_i(s))}}\sum_{j=1}^n(t-x_j(s))c_j(t)ds\ge 0
\end{align*}
for $t\in[0,1]$.
Thus, $k$ is an $\alg$-valued positive definite kernel.
\item Let $\alg=L^{\infty}([0,1])$ and $k:\mcl{X}\times\mcl{X}\to \alg$ be defined such that $k(x,y)(t)$ is a complex-valued positive definite kernel for any $t\in[0,1]$.
Then, $k$ is an $\alg$-valued positive definite kernel.
\item Let $\mcl{W}$ be a separable Hilbert space and let $\{e_i\}_{i=1}^{\infty}$ be an orthonormal basis of $\mcl{W}$.
Let $\alg=\blin{\mcl{W}}$.
Let $k_i:\mcl{X}\times \mcl{X}\to\mathbb{C}$ be a complex-valued positive definite kernel for any $i=1,2,\ldots$.
\red{Assume for any $x\in\mcl{X}$, there exists $C>0$ such that for any $i=1,2,\ldots$, $\vert k_i(x,x)\vert \le C$ holds.}
Let $k:\mcl{X}\times\mcl{X}\to\alg$ be defined as $k(x,y)e_i=k_i(x,y)e_i$.
Then, for $x_1,\ldots,x_n\in\mcl{X}$, $c_1,\ldots,c_n\in\alg$ and $w\in\mcl{W}$, we have
\begin{align*}
\bblacketg{w,\bigg(\sum_{i,j=1}^nc_i^*k(x_i,x_j)c_j\bigg)w}_{\mcl{W}}&=\sum_{i,j=1}^n\sum_{l=1}^{\infty}\blacket{\alpha_{i,l}e_l,k(x_i,x_j)\alpha_{j,l}e_l}_{\mcl{W}}\\
&=\sum_{l=1}^{\infty}\sum_{i,j=1}^n\overline{\alpha_{i,l}}\alpha_{j,l}\tilde{k}_l(x_i,x_j)\ge 0,
\end{align*}
where $c_iw=\sum_{l=1}^{\infty}\alpha_{i,l}e_l$ is the expansion with respect to $\{e_i\}_{i=1}^{\infty}$.
Thus, $k$ is an $\alg$-valued positive definite kernel. 
\item Let $\mcl{X}=C(\Omega,\mcl{Y})$ and $\mcl{W}=L^2(\Omega)$ for a \red{topological space $\Omega$ with a finite Borel measure} and a topological space $\mcl{Y}$.
Let $\alg=\mcl{B}(\mcl{W})$, and $\tilde{k}:\mcl{Y}\times\mcl{Y}\to\mathbb{C}$ be a complex-valued \red{bounded} continuous positive definite kernel.
Moreover, let $k:\mcl{X}\times\mcl{X}\to\alg$ be defined as $(k(x,y)w)(s)=\int_{t\in \Omega}\tilde{k}(x(s),y(t))w(t)dt$.
Then, for $x_1,\ldots,x_n\in\mcl{X}$, $c_1,\ldots,c_n\in\alg$ and $w\in\mcl{W}$, we have
\begin{align*}
\bblacketg{w,\bigg(\sum_{i,j=1}^nc_i^*k(x_i,x_j)c_j\bigg)w}_{\mcl{W}}&=\int_{t\in\Omega}\int_{s\in\Omega}\sum_{i,j=1}^n\overline{d_i(s)}\tilde{k}(x_i(s),x_j(t))d_j(t)dsdt\ge 0,
\end{align*}
where $d_i=c_iw$.
Thus, $k$ is an $\alg$-valued positive definite kernel.
\end{enumerate}
\end{exam}

\color{black}
Let $\phi:\mcl{X}\to\alg^{\mcl{X}}$ be the {\em feature map} associated with $k$, which is defined as $\phi(x)=k(\cdot,x)$ for $x\in\mcl{X}$.
Similar to the case of RKHS, we construct the following $C^*$-module composed of $\alg$-valued functions by means of $\phi$: 
\begin{equation*}
\modu_{k,0}:=\bigg\{\sum_{i=1}^{n}\phi(x_i)c_i\bigg|\ n\in\mathbb{N},\ c_i\in\alg,\ x_i\in\mcl{X}\bigg\}.
\end{equation*}
An $\alg$-valued map $\blacket{\cdot,\cdot}_{\modu_k}:\modu_{k,0}\times \modu_{k,0}\to\alg$ is defined as follows:
\begin{equation*}
\bblacketg{\sum_{i=1}^{n}\phi(x_i)c_i,\sum_{j=1}^{l}\phi(y_j)d_j}_{\modu_k}:=\sum_{i=1}^{n}\sum_{j=1}^{l}c_i^*k(x_i,y_j)d_j.
\end{equation*}
By the properties in Definition~\ref{def:pdk_rkhm} of $k$, $\blacket{\cdot,\cdot}_{\modu_k}$ is well-defined and has the reproducing property
\begin{equation*}
\blacket{\phi(x),v}_{\modu_k}=v(x)
\end{equation*}
for $v\in\modu_{k,0}$ and $x\in\mcl{X}$.
Also, it satisfies the properties in Definition~\ref{def:innerproduct}. 
As a result, $\blacket{\cdot,\cdot}_{\modu_k}$ is shown to be an $\alg$-valued inner product.

The {\em reproducing kernel Hilbert $\alg$-module (RKHM)} associated with $k$ is defined as the completion of $\modu_{k,0}$.  
We denote by $\modu_k$ the RKHM associated with $k$.  

\color{\cont}
\citet{heo08} focused on the case where a group acts on $\mcl{X}$ and investigated corresponding actions on RKHMs.
Moreover, he considered the space of operators on Hilbert $\alg$-module and proved that for each operator-valued positive definite kernel associated with a group and cocycle, there is a corresponding representation on the Hilbert $C^*$-module associated with the positive definite kernel.

\color{\motiv}
\section{Application of RKHM to functional data}\label{sec:motivation}
In this section, we provide an overview of the motivation for studying RKHM for data analysis.
We especially focus on the application of RKHM to functional data.

Analyzing functional data has been researched to take advantage of the additional information implied by the smoothness of functions underlying data~\citep{ramsay05,levitin07,wang16}. 
By describing data as functions, we obtain information as functions such as derivatives.
Applying kernel methods to functional data is also proposed~\citep{kadri16}.
In these frameworks, the functions are assumed to be vectors in a Hilbert space such as $L^2(\Omega)$ for a measure space $\Omega$, or they are embedded in an RKHS or vvRKHS.
Then, analyses are addressed in these Hilbert spaces.

However, since functional data itself is infinite-dimensional data, Hilbert spaces are not always sufficient for extracting its continuous behavior.
This is because the inner products in Hilbert spaces are complex-valued, degenerating or failing to capture the continuous behavior of the functional data.
We compare algorithms in Hilbert spaces and those in Hilbert $C^*$-modules and show advantages of algorithms in Hilbert $C^*$-modules over those in Hilbert spaces, which are summarized in Figure~\ref{fig:advantages}.
We first consider algorithms in Hilbert spaces for analyzing functional data $x_1,x_2,\ldots\in C(\Omega,\mcl{X})$, where $\Omega$ is a \red{$\sigma$-finite measure space} and $\mcl{X}$ is a Hilbert space.
There are two possible typical patterns of algorithms in Hilbert spaces.
The first pattern (Pattern 1 in Fig.~\ref{fig:advantages}) is regarding each function $x_i$ as a vector in a Hilbert space $\hil$ containing $C(\Omega,\mcl{X})$.
In this case, the inner product $\blacket{x_i,x_j}_{\hil}$ between two functions $x_i$ and $x_j$ is single  complex-valued although $x_i$ and $x_j$ are functions.
Therefore, information of the value of functions at each point degenerates into a complex value.
The second pattern (Pattern 2 in Fig.~\ref{fig:advantages}) is discretizing each function $x_i$ as $x_i(t_0),x_i(t_1),\ldots$ for $t_0,t_1,\ldots\in\Omega$ and regarding each discretized value $x_i(t_l)$ as a vector in the Hilbert space $\mathcal{X}$.
In this case, we obtain the complex-valued inner product $\blacket{x_i(t_l),x_j(t_l)}_{\mcl{X}}$ at each point $t_l\in\Omega$.
However, because of the discretization, continuous behaviors, for example, derivatives, total variation, and frequency components, of the function $x_i$ are lost.
Algorithms of both patterns in the Hilbert spaces proceed by using the computed complex-valued inner products.
As a result, capturing features of functions with the algorithms in the Hilbert spaces is difficult.
On the other hand, if we regard each function $x_i$ as a vector in a Hilbert $C^*$-module $\modu$ (the rightmost picture in Fig.~\ref{fig:advantages}), then the inner product $\blacket{x_i,x_j}_{\modu}$ between two functions $x_i$ and $x_j$ in the Hilbert $C^*$-module is $C^*$-algebra-valued.
Thus, if we set the $C^*$-algebra as a function space such as $L^{\infty}(\Omega)$, the inner product $\blacket{x_i,x_j}_{\modu}$ is function-valued.
Therefore, algorithms in Hilbert $C^*$-modules enable us to capture and extract continuous behaviors of functions.
Moreover, in the case of the outputs are functions, we can control the outputs according to the features of the functions.

Since RKHM is a generalization of RKHS and vvRKHS (see Subsection~\ref{subsec:rkhm_vvrks} for further details), the framework of RKHMs (Hilbert $C^*$-modules) allows us to generalize kernel methods in RKHSs and vvRKHSs (Hilbert spaces) to those in Hilbert $C^*$-modules.
Therefore, by using RKHM, we can capture and extract features of functions in kernel methods.
The remainder of this paper is devoted to developing the theory of applying RKHMs to data analysis and showing examples of practical applications of data analysis in RKHMs (PCA, time-series data analysis, and analysis of interaction effects).
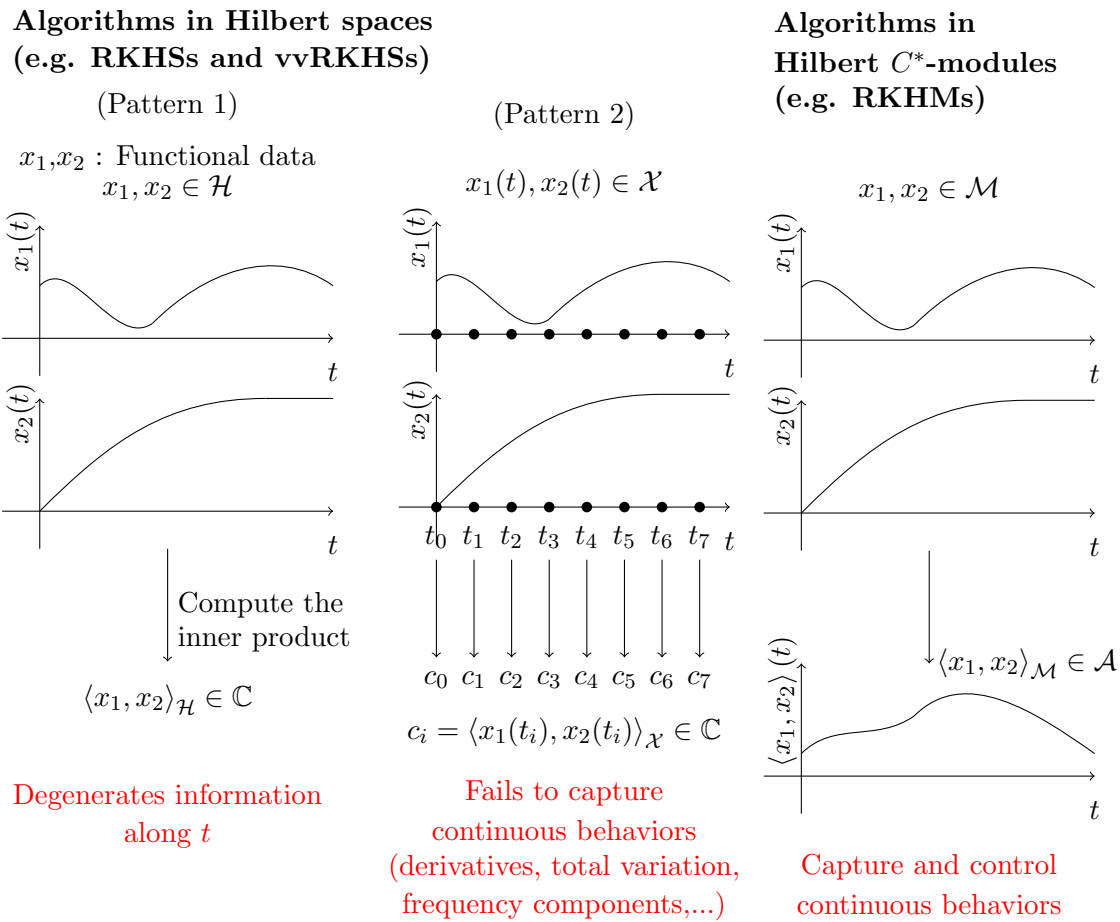
\begin{figure}
\begin{minipage}{0.33\textwidth}
\centering
\begin{tikzpicture}
\draw[->] (1,-0.5) -- (5.4,-0.5);
\draw[->] (1.5,-1) -- (1.5,1);
\node at (1.25,0.8) [rotate=90] {$x_2(t)$};
\draw (1.5,-0.5) to [out=45,in=180] (4.5,1);
\draw (4.5,1) -- (5.4,1);
\node at (5.4,-0.7) [below] {$t$};
\draw[->] (1,1.8) -- (5.4,1.8);
\draw[->] (1.5,1.3) -- (1.5,3.3);
\node at (1.25,3.1) [rotate=90] {$x_1(t)$};
\draw (1.5,2.5) to [out=45,in=-145] (3,2);
\draw (3,2) to [out=45,in=145] (5.4,2.5);
\node at (5.4,1.6) [below] {$t$};
\draw[->] (3.2,-1) to node [right]{Compute the} (3.2,-2.5);
\node at (3.2,-2.2) [right] {inner product};
\node at (3.2,-3) [] {$\blacket{x_1,x_2}_{\hil}\in\mathbb{C}$};
\node at (3.2,-3.5) [] {\phantom{$c_i=\blacket{x_1(t_i),x_2(t_i)}\in\mathbb{C}$}};
\node at (3.2,-4.3) [] {\textcolor{red}{Degenerates information}};
\node at (3.2,-4.8) [] {\textcolor{red}{along $t$}};
\node at (3.2,-5.3) [] {\textcolor{red}{\phantom{a}}};
\node at (3.2,-5.8) [] {\textcolor{red}{\phantom{a}}};
\node at (3.2,4.2) [] {$x_1$,$x_2$ : Functional data};
\node at (3.2,3.8) [] {$x_1,x_2\in\hil$};
\node at (1,6) [right] {\textbf{Algorithms in Hilbert spaces}};
\node at (1,5.5) [right] {\textbf{(e.g. RKHSs and vvRKHSs)}};
\node at (3.2,4.9) [thick] {(Pattern 1)};
\end{tikzpicture}
\end{minipage}%
\begin{minipage}{0.33\textwidth}
\centering
\begin{tikzpicture}
\draw[->] (1,-0.5) -- (5.4,-0.5);
\draw[->] (1.5,-1) -- (1.5,1);
\node at (1.25,0.8) [rotate=90] {$x_2(t)$};
\draw (1.5,-0.5) to [out=45,in=180] (4.5,1);
\draw (4.5,1) -- (5.4,1);
\node at (5.4,-0.7) [below] {$t$};
\draw[->] (1,1.8) -- (5.4,1.8);
\draw[->] (1.5,1.3) -- (1.5,3.3);
\node at (1.25,3.1) [rotate=90] {$x_1(t)$};
\draw (1.5,2.5) to [out=45,in=-145] (3,2);
\draw (3,2) to [out=45,in=145] (5.4,2.5);
\node at (5.4,1.6) [below] {$t$};
\draw[->] (1.5,-1.2) to node [right]{} (1.5,-2.5);
\draw[->] (2,-1.2) to node [right]{} (2,-2.5);
\draw[->] (2.5,-1.2) to node [right]{} (2.5,-2.5);
\draw[->] (3,-1.2) to node [right]{} (3,-2.5);
\draw[->] (3.5,-1.2) to node [right]{} (3.5,-2.5);
\draw[->] (4,-1.2) to node [right]{} (4,-2.5);
\draw[->] (4.5,-1.2) to node [right]{} (4.5,-2.5);
\draw[->] (5,-1.2) to node [right]{} (5,-2.5);
\node at (1.5,-0.9) [] {$t_0$};
\fill[] (1.5,-0.5) circle (2pt);
\fill[] (2,-0.5) circle (2pt);
\fill[] (2.5,-0.5) circle (2pt);
\fill[] (3,-0.5) circle (2pt);
\fill[] (3.5,-0.5) circle (2pt);
\fill[] (4,-0.5) circle (2pt);
\fill[] (4.5,-0.5) circle (2pt);
\fill[] (5,-0.5) circle (2pt);
\fill[] (1.5,1.8) circle (2pt);
\fill[] (2,1.8) circle (2pt);
\fill[] (2.5,1.8) circle (2pt);
\fill[] (3,1.8) circle (2pt);
\fill[] (3.5,1.8) circle (2pt);
\fill[] (4,1.8) circle (2pt);
\fill[] (4.5,1.8) circle (2pt);
\fill[] (5,1.8) circle (2pt);
\node at (2,-0.9) [] {$t_1$};
\node at (2.5,-0.9) [] {$t_2$};
\node at (3,-0.9) [] {$t_3$};
\node at (3.5,-0.9) [] {$t_4$};
\node at (4,-0.9) [] {$t_5$};
\node at (4.5,-0.9) [] {$t_6$};
\node at (5,-0.9) [] {$t_7$};
\node at (1.5,-2.8) [] {$c_0$};
\node at (2,-2.8) [] {$c_1$};
\node at (2.5,-2.8) [] {$c_2$};
\node at (3,-2.8) [] {$c_3$};
\node at (3.5,-2.8) [] {$c_4$};
\node at (4,-2.8) [] {$c_5$};
\node at (4.5,-2.8) [] {$c_6$};
\node at (5,-2.8) [] {$c_7$};
\node at (3.2,-3.5) [] {$c_i=\blacket{x_1(t_i),x_2(t_i)}_{\mcl{X}}\in\mathbb{C}$};
\node at (3.2,-4.3) [] {\textcolor{red}{Fails to capture}};
\node at (3.2,-4.8) [] {\textcolor{red}{continuous behaviors}};
\node at (3.2,-5.3) [] {\textcolor{red}{(derivatives, total variation,}};
\node at (3.2,-5.8) [] {\textcolor{red}{frequency components,...)}};
\node at (3.2,4) [] {\phantom{Functional data}};
\node at (1,6) [right] {\phantom{\textbf{Algorithms}}};
\node at (1,5.5) [right] {\phantom{\textbf{(e.g. RKHSs)}}};
\node at (3.2,4.7) [thick] {(Pattern 2)};
\node at (3.2,3.8) [] {$x_1(t),x_2(t)\in\mcl{X}$};
\end{tikzpicture}
\end{minipage}%
\begin{minipage}{0.33\textwidth}
\centering
\begin{tikzpicture}
\draw[->] (1,-0.5) -- (5.4,-0.5);
\draw[->] (1.5,-1) -- (1.5,1);
\node at (1.25,0.8) [rotate=90] {$x_2(t)$};
\draw (1.5,-0.5) to [out=45,in=180] (4.5,1);
\draw (4.5,1) -- (5.4,1);
\node at (5.4,-0.7) [below] {$t$};
\draw[->] (1,1.8) -- (5.4,1.8);
\draw[->] (1.5,1.3) -- (1.5,3.3);
\node at (1.25,3.1) [rotate=90] {$x_1(t)$};
\draw (1.5,2.5) to [out=45,in=-145] (3,2);
\draw (3,2) to [out=45,in=145] (5.4,2.5);
\node at (5.4,1.6) [below] {$t$};
\draw[->] (3.2,-1) to node [right]{} (3.2,-2.5);
\node at (3.2,-2.2) [right] {};
\draw[->] (1,-4) -- (5.4,-4);
\draw[->] (1.5,-4.5) -- (1.5,-2.5);
\node at (1.25,-3) [rotate=90] {$\blacket{x_1,x_2}(t)$};
\node at (5.4,-4.2) [below] {$t$};
\draw (1.5,-3.7) to [out=45,in=-145] (3,-3.2);
\draw (3,-3.2) to [out=45,in=145] (5.4,-3.7);
\node at (4.5,-2.5) [] {$\blacket{x_1,x_2}_{\modu}\in\alg$};
\node at (3.2,-5.2) [] {\textcolor{red}{Capture and control}};
\node at (3.2,-5.7) [] {\textcolor{red}{continuous behaviors}};
\node at (3.2,4) [] {\phantom{Functional data}};
\node at (1,6) [right] {\textbf{Algorithms in}};
\node at (1,5.5) [right] {\textbf{Hilbert $C^*$-modules}};
\node at (1,5) [right] {\textbf{(e.g. RKHMs)}};
\node at (3.2,3.8) [] {$x_1,x_2\in\modu$};
\end{tikzpicture}
\end{minipage}%
\caption{Advantages of algorithms in Hilbert $C^*$-modules over those in Hilbert spaces}\label{fig:advantages}
\end{figure}

\section{RKHM for data analysis}\label{sec:RKHM}
{As we mentioned in Section~\ref{secintro}, RKHM has been studied in mathematical physics and pure mathematics.
In existing studies, mathematical properties of RKHM such as the relationship between group actions and RKHMs (see the last paragraph of Subsection~\ref{sec:rkhm_review}) have been discussed. 
However, these studies have not been focused on data and algorithms for analyzing it.}
Therefore, we fill the gaps between the existing theory of RKHM and its application to data analysis in this section.
We develop theories for the validity to applying it to data analysis in Subsection~\ref{subsec:foundation_rkhm}.
Also, we investigate the connection of RKHM with RKHS and vvRKHS in Subsection~\ref{subsec:rkhm_vvrks}.

\color{\cont}
Generalizations of theories of Hilbert space and RKHS are quite nonobvious for general $C^*$-algebras since fundamental properties in Hilbert spaces such as the Riesz representation theorem and orthogonal complementedness are not always obtained in Hilbert $C^*$-modules.
Therefore, we consider limiting $C^*$-algebras to an appropriate class of $C^*$-algebras.
In fact, von Neumann-algebras satisfy desired properties.
\begin{defin}[von Neumann-algebra]\label{def:vonNeumann}
A C*-algebra $\mathcal{A}$ is called a von Neumann-algebra if $\mathcal{A}$ is isomorphic to the dual Banach space of some Banach space. 
\end{defin}
The following propositions are fundamental for deriving useful properties for data analysis in Hilbert $C^*$-modules and RKHMs~\citep[Theorem 4.16]{skeide00}, \citep[Proposition 2.3.3]{manuilov00}.
\begin{prop}[The Riesz representation theorem for Hilbert $\alg$-modules]\label{thm:riesz}
Let\\ $\alg\red{\subseteq\mcl{B}(\mcl{W})}$ be a von Neumann-algebra and \red{let $\modu$ be a Hilbert $\alg$-module}.
Let $\hil=\modu\otimes_{\alg}\mcl{W}$ (see Definition~\ref{def:interior_tensor} for the definition of the product $\modu\otimes_{\alg}\mcl{W}$).
Then, every $v\in\modu$ can be regarded as an operator in $\mcl{B}(\mcl{W}, \hil)$, the set of bounded linear operators from $\mcl{W}$ to $\hil$.
If $\modu\subseteq\mcl{B}(\mcl{W}, \hil)$ is strongly closed (in this case, we say that $\modu$ is a von Neumann $\alg$-module),
then for a bounded $\alg$-linear map $L:\modu\to\alg$ (see Definition~\ref{def:a_lin_op}), there exists a unique $u\in\modu$ such that $Lv=\blacket{u,v}_{\modu}$ for all $v\in\modu$.
\end{prop}
Let $\alg$ be a von Neumann-algebra.
We remark that the Hilbert $\alg$-module $\alg^n$ for some $n\in\mathbb{N}$ is a von Neumann $\alg$-module.
Moreover, for an $\alg$-valued positive definite kernel defined as $\tilde{k}1_{\alg}$, where $\tilde{k}$ is a (standard) positive definite kernel, 
the RKHM $\modu_k$ is a von Neumann $\alg$-module. 
(Generally, the Hilbert $\alg$-module represented as $\hil\otimes\alg$ for a Hilbert space $\hil$ is a von Neumann $\alg$-module.
Here, $\otimes$ represents the tensor product of a Hilbert space and $C^*$-module.
See~\citet[p.6]{lance95} for further details about the tensor product.)
\begin{prop}[Orthogonal complementedness in Hilbert $\alg$-modules]\label{prop:orthcomp}
Let $\alg$ be a \red{unital $C^*$-algebra} and let $\modu$ be a Hilbert $\alg$-module.
Let $\mcl{V}$ be a \red{finitely (algebraically) generated} closed submodule of $\modu$. 
Then, any $u\in\modu$ is decomposed into $u=u_1+u_2$ where $u_1\in\mcl{V}$ and $u_2\in\mcl{V}^{\perp}$.
Here, $\mcl{V}^{\perp}$ is the orthogonal complement of $\mcl{V}$ defined as $\{u\in\modu\mid\ \blacket{u,v}_{\modu}=0\}$.
\end{prop}
\red{Let $\alg$ be a unital $C^*$-algebra, let $\modu$ be a Hilbert $\alg$-module, and let $\{q_1,\ldots,q_n\}$ be an ONS of $\modu$. 
Then, the submodule $\mcl{V}$ generated by $\{q_1,\ldots,q_n\}$ is isomorphic to $\bigoplus_{i=1}^n\mcl{V}_i$, where $\mcl{V}_i=\{\blacket{q_i,q_i}_{\modu}c\mid\ c\in\alg\}$ is a closed submodule of $\alg$.
Thus, we have $\modu=\mcl{V}\oplus\mcl{V}^{\perp}$.}

\color{black}
Therefore, we set $\alg$ as a von Neumann-algebra to derive useful properties of RKHM for data analysis.
Note that every von Neumann-algebra is unital (see Definition~\ref{def:multiplicative_identity}).
\begin{assum}
We assume $\alg$ is a von Neumann-algebra throughout this paper.
\end{assum}
\color{\cont}
$C^*$-algebras in Example~\ref{ex:vonNeumann} are also von Neumann-algebras.
As we noted after Example~\ref{ex:vonNeumann}, any $C^*$-algebra can be regarded as a subalgebra of $\mcl{B}(\mcl{W})$.
Thus, this fact implies setting the range of the positive definite kernel as $\mcl{B}(\mcl{W})$ rather than general $C^*$-algebras is effective for data analysis. 
\color{black}
\subsection{General properties of RKHM for data analysis}\label{subsec:foundation_rkhm}
\color{black}
\subsubsection{Fundamental properties of RKHM}
\color{black}
%
%
Similar to the cases of RKHSs, we show RKHMs constructed by $\alg$-valued positive definite kernels have the reproducing property.
Also, we show that the RKHM associated with an $\alg$-valued positive definite kernel $k$ is uniquely determined.
\begin{prop}\label{prop:reproducing}
The map $\blacket{\cdot,\cdot}_{\modu_k}$ defined on $\modu_{k,0}$ is extended continuously to $\modu_k$ and the map $\modu_k\ni v\mapsto (x\mapsto\blacket{\phi(x),v}_{\modu_k})\in\alg^{\mcl{X}}$ is injective.
Thus, $\modu_k$ is regarded to be the subset of $\alg^{\mcl{X}}$ and has the reproducing property.
\end{prop}
\begin{prop}\label{prop:rkhm_unique}
Assume a Hilbert $C^*$-module $\mcl{M}$ over $\mcl{A}$ and a map $\psi:\mcl{X}\to\mcl{M}$ satisfy the following conditions:
\begin{enumerate}
 \item $^{\forall} x,y\in\mcl{X}$, $\blacket{\psi(x),\psi(y)}_{\mcl{M}}=k(x,y)$
 \item $\overline{\{\sum_{i=1}^n\psi(x_i)c_i\mid\ x_i\in\mcl{X},\ c_i\in\mcl{A}\}}=\mcl{M}$
\end{enumerate}
Then, there exists a unique $\mcl{A}$-linear bijection map $\Psi:\mcl{M}_k\to\mcl{M}$ that preserves the inner product and satisfies the following commutative diagram:
\begin{equation*}
 \xymatrix{
\mcl{M}_k \ar[rr]^{\Psi}&  &\mcl{M}\\
&\mcl{X} \ar[lu]^{\phi} \ar[ru]_{\psi}\ar@{}[u]|{\circlearrowright} &
}
\end{equation*}
\end{prop}
We give the proofs for the above propositions in Appendix~\ref{ap:proof_3_1}.

\color{black}
\subsubsection{Minimization property and representer theorem in RKHMs}
\color{black}
We now develop some theories for the validity to apply RKHM to data analysis.
First, we show a minimization property of orthogonal projection operators, which is a fundamental property in Hilbert spaces, is also available in Hilbert $C^*$-modules.
\begin{thm}[Minimization property of orthogonal projection operators]\label{prop:min_projection}
Let $\mcl{I}$ be a \red{finite} index set.
Let $\{q_i\}_{i\in\mcl{I}}$ be an ONS of $\modu$ and 
$\mcl{V}$ be the submodule of $\modu$ spanned by $\{q_i\}_{i\in\mcl{I}}$.
For $u\in\modu_k$, let $P:\modu\to\mcl{V}$ be the projection operator defined as $Pu:=\sum_{i\in\mcl{I}}q_i\blacket{q_i,u}_{\modu}$.
Then $Pu$ is the unique solution of the following minimization problem, where the minimum is taken with respect to a (pre) order in $\alg$ (see Definition~\ref{def:sup}): 
\begin{equation}
\min_{v\in\mcl{V}}\vert u-v\vert_{\modu}^2.\label{eq:min_projection}
\end{equation}
\end{thm}
\begin{proof}
By Proposition~\ref{prop:orthcomp}, $u\in\mcl{M}$ is decomposed into $u=u_1+u_2$, where $u_1=Pu\in\mcl{V}$ and $u_2=u-u_1\in\mcl{V}^{\perp}$.
Let $v\in\mcl{V}$.
Since $u_1-v\in\mcl{V}$, the identity $\blacket{u_2,u_1-v}_{\modu}=0$ holds.
Therefore, we have
\begin{equation}
 \vert u-v\vert_{\modu}^2=\vert u_2+(u_1-v)\vert_{\modu}^2=\vert u_2\vert_{\modu}^2+\vert u_1-v\vert_{\modu}^2\label{eq:deviation},
\end{equation}
which implies $\vert u-v\vert_{\modu}^2-\vert u-u_1\vert_{\modu}^2\ge_{\alg} 0$.
Since $v\in\mcl{V}$ is arbitrary, $u_1$ is a solution of $\min_{v\in\mcl{V}}\vert u-v\vert_{\modu}$.

Moreover, if there exists $u'\in\mcl{V}$ such that $\vert u-u_1\vert_{\modu}^2=\vert u-u'\vert_{\modu}^2$, then letting $v=u'$ in Eq.~\eqref{eq:deviation} derives $\vert u-u'\vert_{\modu}^2=\vert u_2\vert_{\modu}^2+\vert u_1-u'\vert_{\modu}^2$, which implies $\vert u_1-u'\vert_{\modu}^2=0$.
As a result, $u_1=u'$ holds and the uniqueness of $u_1$ has been proved. 
\end{proof}
Proposition~\ref{prop:min_projection} shows the orthogonally projected vector uniquely minimizes the deviation from an original vector in $\mcl{V}$.
Thus, 
we can generalize methods related to orthogonal projections in Hilbert spaces to Hilbert $C^*$-modules.

Next, we show the representer theorem in RKHMs.
\begin{thm}[Representer theorem]\label{thm:representation}
Let $x_1,\ldots,x_n\in\mcl{X}$ and $a_1,\ldots,a_n\in\alg$.
Let $h:\mcl{X}\times\alg^2\to\alg_+$ be an error function and let $g:\alg_+\to\alg_+$ satisfy $g(c)\le_{\alg} g(d)$ for $c\le_{\alg} d$.
\red{Assume the module spanned by $\{\phi(x_i)\}_{i=1}^n$ is closed.}
Then, any $u\in\modu_k$ minimizing $\sum_{i=1}^nh(x_i,a_i,u(x_i))+g(\vert u\vert_{\modu_k})$ admits a representation of the form $\sum_{i=1}^n\phi(x_i)c_i$ for some $c_1,\ldots,c_n\in\alg$.
\end{thm}
\begin{proof}
Let $\mcl{V}$ be the module spanned by $\{\phi(x_i)\}_{i=1}^n$.
By Proposition~\ref{prop:orthcomp}, $u\in\modu_k$ is decomposed into $u=u_1+u_2$, where  $u_1\in\mcl{V}$, $u_2\in\mcl{V}^{\perp}$.
By the reproducing property of $\modu_k$, the following equalities are derived for $i=1,\ldots,n$:
\begin{align*}
&u(x_i)=\blacket{\phi(x_i),u}_{\modu_k}=\blacket{\phi(x_i),u_1+u_2}_{\modu_k}=\blacket{\phi(x_i),u_1}_{\modu_k}.
\end{align*}
Thus, $\sum_{i=1}^nh(x_i,a_i,u(x_i))$ is independent of $u_2$.
As for the term $g(\vert u\vert_{\modu_k})$, since $g$ satisfies $g(c)\le_{\alg} g(d)$ for $c\le_{\alg} d$, we have
\begin{align*}
g(\vert u\vert_{\modu_k})=g(\vert u_1+u_2\vert_{\modu_k})=g\Big(\big(\vert u_1\vert_{\modu_k}^2+\vert u_2\vert_{\modu_k}^2\big)^{1/2}\Big)\ge_{\alg} g(\vert u_1\vert_{\modu_k}).
\end{align*}
Therefore, setting $u_2=0$ does not affect the term $\sum_{i=1}^nh(x_i,a_i,u(x_i))$, while strictly reducing the term $g(\vert u\vert_{\modu_k})$, which implies any minimizer
must have $u_2=0$.
As a result, any minimizer takes the form $\sum_{i=1}^n\phi(x_i)c_i$.
\end{proof}


\subsection{Connection with RKHSs and vvRKHSs}\label{subsec:rkhm_vvrks}
We show that the framework of RKHM is more general than those of RKHS and vvRKHS.
Let $\tilde{k}$ be a complex-valued positive definite kernel and let $\hil_{\tilde{k}}$ be the RKHS associated with $\tilde{k}$.
In addition, let $k$ be an $\alg$-valued positive definite kernel and $\modu_k$ be the RKHM associated with $k$.
The following proposition is derived by the definitions of RKHSs and RKHMs.
\begin{prop}[Connection between RKHMs with RKHSs]
If $\alg=\mathbb{C}$ and $k=\tilde{k}$, then $\hil_{\tilde{k}}=\modu_k$.
\end{prop}

As for the connection between vvRKHSs and RKHMs, we first remark that in the case of $\alg=\blin{\mcl{W}}$, Definition~\ref{def:pdk_rkhm} is equivalent to the operator valued positive definite kernel (Definition~\ref{def:pdk_vv-rkhs}) for the theory of vv-RKHSs.
\begin{lem}[Connection between Definition~\ref{def:pdk_rkhm} and Definition~\ref{def:pdk_vv-rkhs}]\label{lem:pdk_equiv}
If $\alg=\mcl{B}(\mcl{W})$, then, the $\alg$-valued positive definite kernel defined in Definition~\ref{def:pdk_rkhm} is equivalent to the operator valued positive definite kernel defined in Definition~\ref{def:pdk_vv-rkhs}.
\end{lem}
The proof for Lemma~\ref{lem:pdk_equiv} is given in Appendix~\ref{ap:proof_3_1}.

\color{black}
Let $\alg=\blin{\mcl{W}}$ and let $\hil_{k}^{\opn{v}}$ be the vvRKHS associated with $k$.
To investigate further connections between vvRKHSs and RKHMs, we introduce the notion of interior tensor~\citep[Chapter 4]{lance95}.
\begin{prop}
Let $\modu$ be a Hilbert $\blin{\mcl{W}}$-module \textcolor{black}{and let $\modu\otimes\mcl{W}$ be the tensor product of $\modu$ and $\mcl{W}$ as vector spaces}.
The map $\blacket{\cdot,\cdot}_{\modu\otimes\mcl{W}}:\modu\otimes \mcl{W}\;\times\;\modu\otimes \mcl{W}\to\mathbb{C}$ defined as
\begin{equation*}
\blacket{v\otimes w,u\otimes h}_{\modu\otimes\mcl{W}}=\blacket{w,\blacket{v,u}_{\modu}h}_{\mcl{W}}
\end{equation*}
is a complex-valued pre inner product on $\modu\otimes \mcl{W}$.
\end{prop}
\begin{defin}[Interior tensor]\label{def:interior_tensor}
The completion of $\modu\otimes \mcl{W}$ with respect to the pre inner product $\blacket{\cdot,\cdot}_{\modu\otimes\mcl{W}}$ is referred to as the {\em interior tensor} between $\modu$ and $\mcl{W}$, and denoted as $\modu\otimes_{\blin{\mcl{W}}}\mcl{W}$. 
\end{defin}
Note that $\modu\otimes_{\blin{\mcl{W}}}\mcl{W}$ is a Hilbert space.
We now show vvRKHSs are reconstructed by the interior tensor between RKHMs and $\mcl{W}$.
\begin{thm}[Connection between RKHMs and vvRKHSs]\label{thm:rkhm_vvrkhs}
If $\alg=\blin{\mcl{W}}$, then two Hilbert spaces $\hil_k^{\opn{v}}$ and $\modu\otimes_{\blin{\mcl{W}}}\mcl{W}$ are isomorphic.
\end{thm}
Theorem~\ref{thm:rkhm_vvrkhs} is derived by the following lemma.
\begin{lem}
There exists a unique unitary map $U\colon \modu_k\otimes_{\blin{\mcl{W}}}\mcl{W}\to \hil_k^{\opn{v}}$ such that $U(\phi(x)c \otimes w)=\phi(x)(cw)$ holds for all $x\in \mcl{X}$, $c\in \blin{\mcl{W}}$ and $w\in\mcl{W}$.
\end{lem}
\begin{proof}
	First,
	we show that
	\begin{align*}
	\bblacketg{\sum_{i=1}^n\phi({x_i})c_i\otimes w_i, \sum_{j=1}^l\phi({y_j})d_j\otimes h_j}_{\modu_k\otimes\mcl{W}}
	=\bblacketg{\sum_{i=1}^n\phi({x_i})(c_iw_i), \sum_{j=1}^l\phi({y_j})(d_jh_j)}_{\vvrkhs}
	\end{align*}
	holds for all $\sum_{i=1}^n\phi({x_i})c_i\otimes w_i,\sum_{j=1}^l\phi({y_j})d_j\otimes h_j\in \modu_k\otimes_{\blin{\mcl{W}}}\mcl{W}$.
	This follows from the straightforward calculation.
	Indeed, we have
	\begin{align*}
	&\bblacketg{\sum_{i=1}^n\phi({x_i})c_i\otimes w_i, \sum_{j=1}^l\phi({y_j})d_j\otimes h_j}_{\modu_k\otimes\mcl{W}}
	=\sum_{i=1}^n\sum_{j=1}^l\blacket{w_i,\blacket{\phi(x_i)c_i,\phi(y_j)d_j}_kh_j}_{\mcl{W}} \\
	&\qquad=\sum_{i=1}^n\sum_{j=1}^l\blacket{w_i,c_i^*k(x_i,y_j)d_jh_j}_{\mcl{W}} 
	=\sum_{i=1}^n\sum_{j=1}^l\blacket{c_iw_i,k(x_i,y_j)d_jh_j}_{\mcl{W}} \\
	&\qquad=\bblacketg{\sum_{i=1}^n\phi({x_i})(c_iw_i), \sum_{j=1}^l\phi({y_j})(d_jh_j)}_{\vvrkhs}.
	\end{align*}
	Therefore, by the standard functional analysis argument, it turns out that there exists an isometry $U\colon\modu_k\otimes_{\blin{\mcl{W}}}\mcl{W}\to \hil_k^{\opn{v}}$ such that $U(\phi(x)c \otimes w)=\phi(x)(cw)$ holds for all $x\in \mcl{X}$, $c\in \blin{\mcl{W}}$ and $w\in\mcl{W}$.
	Since the image of $U$ is closed and dense in $\hil_k^{\opn{v}}$,
	$U$ is surjective.
	Thus $U$ is a unitary map.
\end{proof}

\section{Kernel mean embedding in RKHM}\label{sec:kme}
We generalize KME in RKHSs, which is widely used in analyzing distributions, to RKHMs.
By using the framework of RKHM, we can embed $\alg$-valued measures instead of probability measures (more generally, complex-valued measures).
\textcolor{black}{We provide a brief review of $\alg$-valued measures and the integral with respect to $\alg$-valued measures in Appendix~\ref{subsec:vv_measure}}.
We define a KME in RKHMs in Subsection~\ref{subsec:kme} and show its theoretical properties in Subsection~\ref{subsec:injectivity}.

To define a KME by using $\alg$-valued measures and integrals, we first define $c_0$-kernels.
\begin{defin}[Function space $\clch$]\label{def:c0_function}
For a locally compact Hausdorff space $\mcl{X}$, the set of all $\alg$-valued continuous functions on $\mcl{X}$ vanishing at infinity is denoted as  $C_0(\mcl{X},\alg)$.
Here, an $\alg$-valued continuous function $u$ is said to vanish at infinity if the set $\{x\in\mcl{X}\mid\ \Vert u(x)\Vert_{\alg}\ge\epsilon\}$ is compact for any $\epsilon>0$.
The space $\clch$ is a Banach $\alg$-module with respect to the sup norm.
\end{defin}
Note that if $\mcl{X}$ is compact, any continuous function is contained in $\clch$.
\begin{defin}[$c_0$-kernel]\label{def:c0_kernel}
Let $\mcl{X}$ be a locally compact Hausdorff space.
An $\alg$-valued positive definite kernel $k:\mcl{X}\times\mcl{X}\to\alg$ is referred to as a {\em $c_0$-kernel} if $k$ is bounded and $\phi(x)=k(\cdot,x)\in\clch$ for any $x\in\mcl{X}$. 
\end{defin}
In this section, we impose the following assumption.
\begin{assum}\label{assum:kme}
We assume $\mcl{X}$ is a locally compact Hausdorff space and $k$ is an $\alg$-valued $c_0$-positive definite kernel.
In addition, we assume $\modu_k$ is a von Neumann $\alg$-module (see Proposition~\ref{thm:riesz}).
\end{assum}
For example, we often consider $\mcl{X}=\mathbb{R}^d$ in practical situations.
Also, we provide examples of $c_0$-kernels as follows.
\begin{exam}
\begin{enumerate}
\item Let $\alg=L^{\infty}([0,1])$ and $k$ is an $\alg$-valued positive definite kernel defined such that $k(x,y)(t)$ is a complex-valued positive definite kernel for $t\in[0,1]$ (see Example~\ref{ex:pdk1}.2).
\red{Assume there exists a complex-valued $c_0$-positive definite kernel $\tilde{k}$ such that for any $t\in[0,1]$, $\vert k(x,y)(t)\vert\le \vert \tilde{k}(x,y)\vert$ holds.}
If $\Vert k(x,y)\Vert_{\alg}$ is continuous with respect to $y$ for any $x\in\mcl{X}$, then the inclusion 
\begin{equation*}
  \{y\in\mcl{X}\mid\ \Vert k(x,y)\Vert_{\alg}\ge \epsilon\}\subseteq \{y\in\mcl{X}\mid\ \vert \tilde{k}(x,y)\vert\ge \epsilon\} 
\end{equation*}
holds for $x\in\mcl{X}$ and $\epsilon>0$.
Since $\tilde{k}$ is a $c_0$-kernel, the set $\{y\in\mcl{X}\mid\ \vert \tilde{k}(x,y)\vert\ge \epsilon\}$ is compact (see Definition~\ref{def:c0_function}). 
Thus, $\{y\in\mcl{X}\mid\ \Vert k(x,y)\Vert_{\alg}\ge \epsilon\}$ is also compact and $k$ is an $\alg$-valued $c_0$-positive definite kernel.
Examples of complex-valued $c_0$-positive definite kernels are Gaussian, Laplacian and $B_{2n+1}$-spline kernels.
\item Let $\mcl{W}$ be a separable Hilbert space and let $\{e_i\}_{i=1}^{\infty}$ be an orthonormal basis of $\mcl{W}$.
Let $\alg=\blin{\mcl{W}}$ and let $k:\mcl{X}\times\mcl{X}\to\alg$ be defined as $k(x,y)e_i=k_i(x,y)e_i$, where $k_i:\mcl{X}\times \mcl{X}\to\mathbb{C}$ is a complex-valued positive definite kernel for any $i=1,2,\ldots$ (see Example~\ref{ex:pdk1}.3).
\red{Assume there exists a complex-valued $c_0$-positive definite kernel $\tilde{k}$ such that for any $i=1,2,\ldots$, $\vert k_i(x,y)\vert\le \vert \tilde{k}(x,y)\vert$ holds.}

If $\Vert k(x,y)\Vert_{\alg}$ is continuous with respect to $y$ for any $x\in\mcl{X}$, then $k$ is shown to be an $\alg$-valued $c_0$-positive definite kernel in the same manner as the above example.
\end{enumerate}
\end{exam}

We introduce $\alg$-valued measure and integral in preparation for defining a KME in RKHMs.
\textcolor{black}{They are special cases of vector measure and integral~\citep{dinculeanu67,dinculeanu00}, respectively.
We review vector measure and integral as $\alg$-valued ones in Appendix~\ref{subsec:vv_measure}.}
The notions of measure and the Lebesgue integral are generalized to $\alg$-valued.
\subsection{Kernel mean embedding of $C^*$-algebra-valued measures}\label{subsec:kme}
%
We now define a KME in RKHMs.
\begin{defin}[KME in RKHMs]
Let $\mcl{D}(\mcl{X},\alg)$ be the set of all $\alg$-valued finite regular Borel measures.
A {\em kernel mean embedding} in an RKHM $\modu_k$ is a map $\Phi: \mcl{D}(\mcl{X},\alg)\rightarrow \modu_k$ defined by
\begin{equation}
 \Phi(\mu):=\int_{x\in\mcl{X}}\phi(x)d\mu(x).\label{eq:km_rkhm}
\end{equation}
\end{defin}
We emphasize that the well-definedness of $\Phi$ is not trivial, and von Neumann $\alg$-module is adequate to show it. More precisely, the following theorem derives the well-definedness.
\begin{thm}[Well-definedness for the KME in RKHMs]\label{thm:kme}
Let $\mu\in\mcl{D}(\mcl{X},\alg)$.
Then, $\Phi(\mu)\in\modu_k$. 
In addition, the following equality holds for any $v\in\modu_k$:
\begin{equation}
\blacket{\Phi(\mu),v}_{\modu_k}=\int_{x\in\mcl{X}}d\mu^*(x)v(x).\label{eq:reproducing_km}
\end{equation}
\end{thm}
To show Theorem~\ref{thm:kme}, we use the Riesz representation theorem for Hilbert $\alg$-modules (Proposition~\ref{thm:riesz}).

\begin{proof} 
Let $L_{\mu}:\modu_k\to\alg$ be an $\alg$-linear map defined as $L_{\mu}v:=\int_{x\in\mcl{X}}d\mu^*(x)v(x)$.
The following inequalities are derived by the reproducing property and the Cauchy--Schwarz inequality (Lemma~\ref{lem:c-s}):
\begin{align}
 \Vert L_{\mu}v\Vert_{\alg}
&\le\int_{x\in\mcl{X}}\Vert v(x)\Vert_{\alg}d\vert \mu\vert(x)
=\int_{x\in\mcl{X}}\Vert \blacket{\phi(x),v}_{\modu_k}\Vert_{\alg}d\vert \mu\vert(x)\nn\\
&\le \Vert v\Vert_{\modu_k}\int_{x\in\mcl{X}}\Vert \phi(x)\Vert_{\modu_k} d\vert \mu\vert(x)
\le \vert\mu\vert(\mcl{X})\Vert v\Vert_{\modu_k}\sup_{x\in\mcl{X}}\Vert\phi(x)\Vert_{\modu_k},\label{eq:bounded}
\end{align}
where the first inequality is easily checked for a step function $s(x):=\sum_{i=1}^nc_i\chi_{E_i}(x)$ as follows:
\begin{align*}
\bigg\Vert \int_{x\in\mcl{X}}d\mu^*(x)s(x)\bigg\Vert_{\alg}
&=\bigg\Vert \sum_{i=1}^n\mu(E_i)^*c_i\bigg\Vert_{\alg}
\le \sum_{i=1}^n\Vert \mu(E_i)\Vert_{\alg}\Vert c_i\Vert_{\alg}\\
&\le \sum_{i=1}^n\vert \mu\vert(E_i)\Vert c_i\Vert_{\alg}
=\int_{x\in\mcl{X}}\Vert s(x)\Vert_{\alg}d\vert \mu\vert(x).
\end{align*}
Thus, it holds for any totally measurable functions.
Since both $\vert{\mu}\vert(\mcl{X})$ and $\sup_{x\in\mcl{X}}\Vert\phi(x)\Vert_{\modu_k}$ are finite, inequality~\eqref{eq:bounded} means $L_{\mu}$ is bounded.
Thus, by the Riesz representation theorem for Hilbert $\alg$-modules (Proposition~\ref{thm:riesz}), there exists $u_{\mu}\in\modu_k$ such that $L_{\mu}v=\blacket{u_{\mu},v}_{\modu_k}$.
By setting $v=\phi(y)$, for $y\in\mcl{X}$, we have $u_{\mu}(y)=L_{\mu}\phi(y)^*=\int_{x\in\mcl{X}}k(y,x)d\mu(x)$.
Therefore, $\Phi(\mu)=u_{\mu}\in\modu_k$ and $\blacket{\Phi(\mu),v}_{\modu_k}=\int_{x\in\mcl{X}}d\mu^*(x)v(x)$.
\end{proof}

\begin{cor}
 For $\mu,\nu\in\mcl{D}(\mcl{X},\alg)$, the inner product between $\Phi(\mu)$ and $\Phi(\nu)$ is given as follows:
\begin{equation*}
\blacket{\Phi(\mu),\Phi(\nu)}_{\modu_k}=\int_{x\in\mcl{X}}\int_{y\in\mcl{X}}d\mu^*(x)k(x,y)d\nu(y).
\end{equation*}
\end{cor}
Moreover, many basic properties for the existing KME in RKHS are generalized to the proposed KME as follows.
\begin{prop}[Basic properties of the KME $\Phi$]\label{prop:kme_lin}
For $\mu,\nu\in\mcl{D}(\mcl{X},\alg)$ and $c\in\alg$, $\Phi(\mu+\nu)=\Phi(\mu)+\Phi(\nu)$ and $\Phi(\mu c)=\Phi(\mu)c$ (i.e., $\Phi$ is $\alg$-linear, see Definition~\ref{def:a_lin_op}) hold.
In addition, for $x\in\mcl{X}$, $\Phi(\delta_x)=\phi(x)$ (see Definition~\ref{def:dirac} for the definition of the $\alg$-valued Dirac measure $\delta_x$).
\end{prop}
This is derived from Eqs.~\eqref{eq:km_rkhm} and \eqref{eq:reproducing_km}.
Note that if $\alg=\mathbb{C}$, then the proposed KME~\eqref{eq:km_rkhm} is equivalent to the existing KME in RKHS considered in~\cite{sriperumbudur11}.

\subsection{Injectivity and universality}\label{subsec:injectivity}
Here, we show the connection between the injectivity of the KME and the universality of RKHM. The proofs of the propositions in this subsection are given in Appendix~\ref{ap:universal}.

\subsubsection{injectivity}
In practice, the injectivity of $\Phi$ is important to transform problems in $\mcl{D}(\mcl{X},\alg)$ into those in $\modu_k$.
This is because if a KME $\Phi$ in an RKHM is injective, then
$\alg$-valued measures are embedded into $\modu_k$ through $\Phi$ without loss of information.
Note that, for probability measures, the injectivity of the existing KME is also referred to as the ``characteristic'' property.
The injectivity of the existing KME in RKHS has been discussed in, for example, \cite{fukumizu08,sriperumbudur10,sriperumbudur11}.  
These studies give criteria for the injectivity of the KMEs associated with important complex-valued kernels such as transition invariant kernels and radial kernels.
Typical examples of these kernels are Gaussian, Laplacian, and inverse multiquadratic kernels.
Here, we define the transition invariant kernels and radial kernels for $\alg$-valued measures, and generalize their criteria to RKHMs associated with $\alg$-valued kernels. 

To characterize transition invariant kernels, we first define a Fourier transform and support of an $\alg$-valued measure.
\begin{defin}[Fourier transform and support of an $\alg$-valued measure]
For an $\alg$-valued measure $\lambda$ on $\mathbb{R}^d$, the {\em Fourier transform} of $\lambda$, denoted as $\hat{\lambda}$, is defined as 
\begin{equation*}
\hat{\lambda}(x)=\int_{\omega\in\mathbb{R}^d}e^{-\sqrt{-1}x^T\omega}d\lambda(\omega).
\end{equation*}
In addition, the {\em support} of $\lambda$ is defined as
\begin{equation*}
\opn{supp}(\lambda)=\{x\in\mathbb{R}^d\mid\ \lambda(\mcl{U})>_{\alg}0\mbox{ for any open set $\mcl{U}$ such that }x\in \mcl{U}\}.
\end{equation*}
\end{defin}
\begin{defin}[Transition invariant kernel and radial kernel]\label{def:transition_invariant}
\begin{enumerate}
\item An $\alg$-valued positive definite kernel $k:\mathbb{R}^d\times\mathbb{R}^d\to\alg$ is called a {\em transition invariant kernel} if it is represented as $k(x,y)=\hat{\lambda}(y-x)$ for a positive $\alg$-valued measure $\lambda$.
\item An $\alg$-valued positive definite kernel $k:\mathbb{R}^d\times\mathbb{R}^d\to\alg$ is called a {\em radial kernel} if it is represented as $k(x,y)=\int_{[0,\infty)}e^{-t\Vert x-y\Vert^2}d\eta(t)$ for a positive $\alg$-valued measure $\eta$.
\end{enumerate}
Here, an $\alg$-valued measure $\mu$ is said to be positive if $\mu(E)\ge_{\alg} 0$ for any Borel set $E$.
\end{defin}
We show transition invariant kernels and radial kernels induce injective KMEs. 
\begin{prop}[The injectivity for transition invariant kernels]\label{thm:characteristic}
 Let $\alg=\mat$ and $\mcl{X}=\mathbb{R}^d$.
 Assume $k:\mcl{X}\times\mcl{X}\to\alg$ is a transition invariant kernel with a positive $\alg$-valued measure $\lambda$ that satisfies $\opn{supp}(\lambda)=\mcl{X}$.
Then, 
KME $\Phi:\mcl{D}(\mcl{X},\alg)\to\modu_k$ defined as Eq.~\eqref{eq:km_rkhm} is injective.
\end{prop}
\begin{prop}[The injectivity for radial kernels]\label{thm:characteristic2}
 Let $\alg=\mat$ and $\mcl{X}=\mathbb{R}^d$.
 Assume $k:\mcl{X}\times\mcl{X}\to\alg$ is a radial kernel with a positive definite $\alg$-valued measure $\eta$ that satisfies $\opn{supp}(\eta)\neq \{0\}$.
 Then, KME $\Phi:\mcl{D}(\mcl{X},\alg)\to\modu_k$ defined as Eq.~\eqref{eq:km_rkhm} is injective.
\end{prop}
%
\begin{exam}\label{ex:injective1}
\begin{enumerate}
\item If $k:\mathbb{R}^d\times \mathbb{R}^d\to\mat$ is a matrix-valued kernel whose diagonal elements are Gaussian, Laplacian, or $B_{2n+1}$-spline and nondiagonal elements are $0$, then $k$ is a $c_0$-kernel (See Example~\ref{ex:pdk1}.1).
There exists a matrix-valued measure $\lambda$ that satisfies $k(x,y)=\hat{\lambda}(y-x)$ and whose diagonal elements are nonnegative and supported by $\mathbb{R}^d$ (c.f. Table 2 in~\cite{sriperumbudur10}) and nondiagonal elements are $0$.
Thus, by Proposition~\ref{thm:characteristic}, $\Phi$ is injective.
\item If $k$ is a matrix-valued kernel whose diagonal elements are inverse multiquadratic and nondiagonal elements are $0$, then $k$ is a $c_0$-kernel.
There exists a matrix-valued measure $\eta$ that satisfies $k(x,y)=\int_{[0,\infty)}e^{-t\Vert x-y\Vert^2}d\eta(t)$, and whose diagonal elements are nonnegative and $\opn{supp}(\eta)\neq\{0\}$ and nondiagonal elements are $0$ 
(c.f. Theorem 7.15 in~\cite{wendland04}).
Thus, by Proposition~\ref{thm:characteristic2}, $\Phi$ is injective.
\end{enumerate}
\end{exam}

\subsubsection{Connection with universality}\label{subsec:universality}
Another important property for kernel methods is universality, which ensures that
kernel-based algorithms approximate 
each continuous target function arbitrarily well.
For RKHS, \cite{sriperumbudur11} showed the equivalence of the injectivity of the existing KME in RKHSs and universality of RKHSs.
We define a universality of RKHMs as follows.
\begin{defin}[Universality]
An RKHM is said to be {\em universal} if it is dense in $\clch$.
\end{defin}
We show the above equivalence holds also for RKHM in the case of $\alg=\mat$.
\begin{prop}[Equivalence of the injectivity and universality for $\alg=\mat$]\label{thm:universal_finitedim}
Let $\alg=\mat$. Then, $\Phi:\measure\to\modu_k$ is injective if and only if $\modu_k$ is dense in $\clch$.
\end{prop}
By Proposition~\ref{thm:universal_finitedim}, if $k$ satisfies the condition in Proposition~\ref{thm:characteristic} 
or \ref{thm:characteristic2}, 
then $\modu_k$ is universal.

For the case where $\alg$ is infinite dimensional, the universality of $\modu_k$ in $\clch$ is a sufficient condition for the injectivity of the proposed KME.
%
\begin{thm}[Connection between the injectivity and universality for general $\alg$]\label{thm:universal}
If $\modu_k$ is dense in $\clch$, then 
$\Phi:\measure\to\modu_k$ is injective.
\end{thm}
However, the equivalence of the injectivity and universality, and the injectivity for transition invariant kernels and radial kernels are open problems.
This is because their proofs strongly depend on the Hahn--Banach theorem and Riesz--Markov representation theorem, and generalizations of these theorems to $\alg$-valued functions and measures are challenging problems due to the situation peculiar to the infinite dimensional spaces.
Further details of the proofs of propositions in this section are given in Appendix~\ref{ap:universal}.

\section{Applications}\label{sec:application}
We apply the framework of RKHM described in Sections~\ref{sec:RKHM} and \ref{sec:kme} to problems in data analysis.
We propose kernel PCA in RKHMs in Subsection~\ref{subsec:pca_rkhm}, {time-series data analysis in RKHMs in Subsection~\ref{subsec:pf}}, and analysis of interaction effects in finite or infinite dimensional data with the proposed KME in RKHMs in Subsection~\ref{subsec:interacting}.
Then, we discuss further applications in Subsection~\ref{subsec:others}.

\subsection{PCA in RKHMs}\label{subsec:pca_rkhm}
Principal component analysis (PCA) is a fundamental tool for describing data in a low dimensional space.
Its implementation in RKHSs has also been proposed~(c.f. \citet{scholkopf01}).
It enables us to deal with the nonlinearlity of data by virtue of the high expressive power of RKHSs.
Here, we generalize the PCA in RKHSs to capture more information in data, such as multivariate data and functional data, by using the framework of RKHM.

\paragraph{Applying RKHM to PCA}
In the existing framework of PCA in Hilbert spaces, the following reconstruction error is minimized with respect to vectors $p_1,\ldots,p_r$:
\begin{equation}
    \sum_{i=1}^n\bigg\Vert x_i-\sum_{j=1}^rp_j\blacket{p_j,x_i}\bigg\Vert^2,\label{eq:pca_normal}
\end{equation}
where $x_1,\ldots,x_n$ are given samples in a Hilbert space and $p_1,\ldots,p_r$ are called principal axes.
Here, the complex-valued inner product $\blacket{p_j,x_i}$ is the weight with respect to the principal axis $p_j$ for representing the sample $x_i$.
PCA for functional data (functional PCA) has also investigated~\citep{ramsay05}.
For example, in standard functional PCA settings, we set the Hilbert space as $L^2(\Omega)$ for a \red{$\sigma$-finite} measure space $\Omega$.
However, if samples $x_1,\ldots,x_n$ are finite dimensional vectors or functions, Eq.~\eqref{eq:pca_normal} fails to describe their element wise or continuous dependencies on the principal axes.
For $d$-dimensional (finite dimensional) vectors, we can just split $x_i=[x_{i,1},\ldots,x_{i,d}]$ into $d$ vectors $[x_{i,1},0,\ldots,0],\ldots,[0,\ldots,0,x_{i,d}]$.
Then, we can understand which element is dominant for representing $x_i$ by using the principal axis $p_j$.
On the other hand, for functional data, the situation is completely different.
For example, assume samples are in $L^2(\Omega)$.
Since delta functions are not contained in $L^2(\Omega)$, we cannot split a sample $x_i=x_i(t)$ into discrete functions. 
In this case, how can we understand the continuous dependencies on the principal axes with respect to the variable $t\in \Omega$?
One possible way to answer this question is to employ Hilbert $C^*$-modules instead of Hilbert spaces.
We consider the same type of reconstruction error as Eq.~\eqref{eq:pca_normal} in Hilbert $C^*$-modules.
In this case, the inner product $\blacket{p_j,x_i}_{\mcl{W}}$ is $C^*$-algebra-valued, which allows us to provide more information than the complex-valued one.
If we set the $C^*$-algebra as the function space on $\Omega$ such as $L^{\infty}(\Omega)$ and define a $C^*$-algebra-valued inner product which depends on $t\in \Omega$, then, the weight $\blacket{p_j,x_i}_{\mcl{W}}$ depends on $t$.
As a result, we can extract continuous dependencies of samples on the principal axes.
More generally, PCA is often considered in an RKHS $\hil_{\tilde{k}}$.
In this case, $x_i$ in Eq.~\eqref{eq:pca_normal} is replaced with $\tilde{\phi}(x_i)$, where $\tilde{\phi}$ is the feature map, and the inner product and norm are replaced with those in the RKHS.
We can extract continuous dependencies of samples on the principal axes by generalizing RKHS to RKHM.

\color{black}
\subsubsection{Generalization of the PCA in RKHSs to RKHMs}
Let $x_1,\ldots,x_n\in\mcl{X}$ be given samples.
Let $k:\mcl{X}\times \mcl{X}\to\alg$ be an $\alg$-valued positive definite kernel on $\mcl{X}$ and let $\modu_k$ be the RKHM associated with $k$.
We explore a useful set of axes $p_1,\ldots,p_r$ in $\modu_k$, which are referred to as principal axes, to describe the feature of given samples $x_1,\ldots,x_n$.
The corresponding components $p_j\blacket{p_j,\phi(x_i)}_{\modu_k}$ are referred to as principal components.
We emphasize our proposed PCA in RKHM provides weights of principal components contained in $\alg$, not in complex numbers.
This is a remarkable difference between our method and existing PCAs.
When samples have some structures such as among variables or in functional data, $\alg$-valued weights provide us richer information than complex-valued ones.
For example, if $\mcl{X}$ is the space of functions of multi-variables and if we set $\alg$ as $L^{\infty}([0,1])$, then we can reduce multi-variable functional data to functions in $L^{\infty}([0,1])$, functions of single variable (as illustrated in Section~\ref{sec:pca_numexp}).

To obtain $\alg$-valued weights of principal components, we consider the following minimization problem regarding the following reconstruction error \textcolor{black}{(see Definition~\ref{def:orthonormal} for the definition of ONS)}:
\begin{equation}
\inf_{\{p_j\}_{j=1}^r\subseteq \modu_k\mbox{\footnotesize : ONS}}\;\sum_{i=1}^n\bigg\vert \phi(x_i)-\sum_{j=1}^rp_j\blacket{p_j,\phi(x_i)}_{\modu_k}\bigg\vert_{\modu_k}^2,\label{eq:pca_min}
\end{equation}
where the infimum is taken with respect to a (pre) order in $\alg$ (see Definition~\ref{def:sup}).
Since the identity $\vert \phi(x_i)-\sum_{j=1}^rp_j\blacket{p_j,\phi(x_i)}_{\modu_k}\vert_{\modu_k}^2=k(x_i,x_i)-\sum_{j=1}^r\blacket{\phi(x_i),p_j}_{\modu_k}\blacket{p_j,\phi(x_i)}_{\modu_k}$ holds
and $\blacket{\phi(x_i),p_j}_{\modu_k}$ is represented as $p_j(x_i)$ by the reproducing property, the problem~\eqref{eq:pca_min} can be reduced to the minimization problem
\begin{equation}
\inf_{\{p_j\}_{j=1}^r\subseteq \modu_k\mbox{\footnotesize : ONS}}\;\sum_{i=1}^n\sum_{j=1}^r-p_j(x_i)p_j(x_i)^*.\label{eq:pca_min4}
\end{equation}
In the case of RKHS, i.e., $\alg=\mathbb{C}$, the solution of the problem~\eqref{eq:pca_min4} is obtained by computing eigenvalues and eigenvectors of Gram matrices (see, for example, \citet{scholkopf01}).
Unfortunately, we cannot extend their procedure to RKHM straightforwardly.
Therefore, we develop two methods to obtain approximate solutions of the problem~\eqref{eq:pca_min4}: by gradient descents on Hilbert $C^*$-modules, and by the minimization of the trace of the $\alg$-valued objective function.

\subsubsection{Gradient descent on Hilbert $C^*$-modules}\label{subsec:gradient}
We propose a gradient descent method on Hilbert $\alg$-module for the case where $\alg$ is commutative.
An important example of commutative von Neumann-algebra is $L^{\infty}([0,1])$. 
The gradient descent for a real-valued function on a Hilbert space has been proposed~\citep{smyrlis04}.
However, in our situation, the objective function of the problem~\eqref{eq:pca_min4} is an $\alg$-valued function in a Hilbert $C^*$-module $\alg^n$.
Thus, the existing gradient descent is not applicable to our situation.
Therefore, we generalize the existing gradient descent algorithm to $\alg$-valued functions on Hilbert $C^*$-modules.

Let $\alg$ be a commutative von Neumann-algebra.
Assume the positive definite kernel $k$ takes its values in $\alg_r:=\{c-d\in\alg\mid\ c,d\in\alg_+\}$.
For example, for $\alg=L^{\infty}([0,1])$,  $\alg_r$ is the space of real-valued $L^{\infty}$ functions on $[0,1]$.
Based on the representer theorem (Theorem~\ref{thm:representation}), \red{we find a solution of the problem~\eqref{eq:pca_min4} which is represented as $p_j=\sum_{i=1}^n\phi(x_i)c_{j,i}$ for some $c_{j,i}\in\alg$}.
Moreover, since $\alg$ is commutative, $p_j(x_i)p_j(x_i)^*$ is equal to $p_j(x_i)^*p_j(x_i)$.
Therefore, the problem~\eqref{eq:pca_min} on $\modu_k$ is equivalent to the following problem on the Hilbert $\alg$-module $\alg^n$ (see Example~\ref{ex:An} about $\alg^n$):
\begin{equation}
\inf_{\mathbf{c}_j\in\alg^n,\ \{\sqrt{\mathbf{G}}\mathbf{c}_j\}_{j=1}^r\mbox{\footnotesize : ONS}}-\sum_{j=1}^r\mathbf{c}_j^*\mathbf{G}^2\mathbf{c}_j,\label{eq:pca_min2}
\end{equation}
where $\mathbf{G}$ is the $\alg$-valued Gram matrix defined as $\mathbf{G}_{i,j}=k(x_i,x_j)$.
For simplicity, we assume $r=1$, i.e., the number of principal axes is $1$.
We rearrange the problem~\eqref{eq:pca_min2} to the following problem by adding a penalty term:
\begin{equation}
\inf_{\mathbf{c}\in\alg^n}(-\mathbf{c}^*\mathbf{G}^2\mathbf{c}+\lambda\vert \mathbf{c}^*\mathbf{Gc}-1_{\alg}\vert_{\alg}^2),\label{eq:pca_min3}
\end{equation}
where $\lambda$ is a real positive weight for the penalty term.
For $r>1$, let $\bc_1$ be a solution of the problem~\eqref{eq:pca_min2}.
Then, we solve the same problem in the orthogonal complement of the module spanned by $\{\bc_1\}$ and set the solution of this problem as $\bc_2$.
Then, we solve the same problem in the orthogonal complement of the module spanned by $\{\bc_1,\bc_2\}$ and repeat this procedure to obtain solutions $\bc_1,\ldots\bc_r$.
The problem~\eqref{eq:pca_min3} is the minimization problem of an $\alg$-valued function defined on the Hilbert $\alg$-module $\alg^n$.
We search a solution of the problem~\eqref{eq:pca_min3} along the steepest descent directions.
\textcolor{black}{To calculate the steepest descent directions, 
we introduce a derivative $Df_{\bc}$ of an $\alg$-valued function $f$ on a Hilbert $C^*$-module at $\bc\in\modu$.
It is defined as the derivative on Banach spaces (c.f. \citet{blanchard15}).
The definition of the derivative is included in Appendix~\ref{ap:gataux}.}
The following gives the derivative of the objective function in problem~\eqref{eq:pca_min3}.
\begin{prop}[Derivative of the objective function]
Let $f:\alg^n\to\alg$ be defined as
\begin{equation}
f(\bc)=-\bc^*\bG^2\bc+\lambda \vert \bc^*\bG\bc-1_{\alg}\vert_{\alg}^2.\label{eq:objective}
\end{equation}
Then, $f$ is infinitely differentiable and the first derivative of $f$ is calculated as
\begin{equation*}
Df_{\bc}(u)=-2\bc^*\bG^2u-4\lambda\bc^*\bG u+4\lambda\bc^*\bG\bc\bc^*\bG u.
\end{equation*}
Moreover, for each $\bc\in\alg^n$, there exists a unique $\mathbf{d}\in\alg^n$ such that $\blacket{\mathbf{d},u}_{\alg^n}=Df_{\bc}(u)$ for any $u\in\alg^n$.
The vector $\mathbf{d}$ is calculated as 
\begin{equation}
\mathbf{d}=-2\bG^2\bc-4\lambda\bG\bc +4\lambda\bG\bc\bc^*\bG\bc.\label{eq:gradient}
\end{equation}
\end{prop}
\begin{proof}
The derivative of $f$ is calculated by the definition and the assumption that $\alg$ is commutative.
Since $Df_{\bc}$ is a bounded $\alg$-linear operator, by the Riesz representation theorem (Proposition~\ref{thm:riesz}), there exists a unique $\mathbf{d}\in\alg^n$ such that $\blacket{\mathbf{d},u}_{\alg^n}=Df_{\bc}(u)$.
\end{proof}
\begin{defin}[Gradient of $\alg$-valued functions on Hilbert $C^*$-modules]
Let $f:\modu\to\alg$ be a differentiable function.
Assume for each $\bc\in\modu$, there exists a unique $\mathbf{d}\in\modu$ such that $\blacket{\mathbf{d},u}_{\alg^n}=Df_{\bc}(u)$ for any $u\in\modu$.
In this case, we denote $\mathbf{d}$ by $\nabla f_{\bc}$ and call it the {\em gradient} of $f$ at $\bc$.
\end{defin}

We now develop an $\alg$-valued gradient descent scheme.
\begin{thm}\label{thm:gradient_decent}
Assume $f:\modu\to\alg$ is differentiable.
Moreover, assume there exists $\nabla f_{\bc}$ for any $\bc\in\modu$.
Let $\eta_t>0$.
Let $\bc_0\in\modu$ and 
\begin{equation}
\bc_{t+1}=\bc_t-\eta_t\nabla f_{\bc_t}\label{eq:gradient_decent}
\end{equation}
for $t=0,1,\ldots$.
Then, we have
\begin{equation}
f(\bc_{t+1})=f(\bc_{t})-\eta_t\vert \nabla f_{\bc_{t}}\vert^2_{\modu}+S(\bc_t,\eta_t),\label{eq:gd_taylor}
\end{equation}
where $S(\bc,\eta)$ satisfies $\lim_{\eta\to 0}\Vert S(\bc,\eta)\Vert_{\alg}/\eta=0$.
\end{thm}
The statement is derived by the definition of the derivative (Definition~\ref{def:derivative}).
The following examples show the scheme~\eqref{eq:gradient_decent} is valid to solve the problem~\eqref{eq:pca_min3}.
\begin{exam}\label{ex:gd}
Let $\alg=L^{\infty}([0,1])$, let $a_t=\vert \nabla f_{\bc_{t}}\vert^2_{\alg^n}\in\alg$ and let $b_{t,\eta}=S(\bc_t,\eta)\in\alg$.
If $a_t\ge_{\alg}\delta 1_{\alg}$ for some positive real value $\delta$, 
then the function $a_t$ on $[0,1]$ satisfies $a_t(s)>0$ for almost everywhere $s\in[0,1]$.
On the other hand, since $b_{t,\eta}$ satisfies $\lim_{\eta\to 0}\Vert b_{t,\eta}\Vert_{\alg}/\eta^2=0$, there exists sufficiently small positive real value $\eta_{t,0}$ such that for almost everywhere $s\in[0,1]$, $b_{t,\eta_{t,0}}(s)\le \Vert b_{t,\eta_{t,0}}\Vert_{\alg}\le \eta^2_{t,0} \delta\le \eta_{t,0}(1-\xi_1)\delta$ hold for some positive real value $\xi_1$.
As a result, $-\eta_{t,0}\vert \nabla f_{\bc_{t}}\vert^2_{\alg^n}+S(\bc_t,\eta_{t,0})\le_{\alg} -\eta_{t,0}\xi_1\vert \nabla f_{\bc_{t}}\vert^2_{\alg^n}$ holds and by the Eq.~\eqref{eq:gd_taylor}, we have
\begin{equation}
f(\bc_{t+1})<_{\alg}f(\bc_t)\label{eq:gd_ineq}
\end{equation}
for $t=0,1,\ldots$.
As we mentioned in Example~\ref{ex:positive}, the inequality~\eqref{eq:gd_ineq} means the function $f(\bc_{t+1})\in L^{\infty}([0,1])$ is smaller than the function $f(\bc_{t})\in L^{\infty}([0,1])$ at almost every points on $[0,1]$, i.e., 
\begin{equation*}
 f(\bc_{t+1})(s)<f(\bc_t)(s)
\end{equation*}
for almost every $s\in [0,1]$.
\end{exam}
\begin{exam}
Assume $\alg$ is a finite dimensional space.
If $\vert \nabla f_{\bc_t}\vert_{\alg}^2\ge_{\alg}\delta 1_{\alg}$ for some positive real value $\delta$, the inequality $f(\bc_{t+1})\le_{\alg}f(\bc_{t})-\eta_{t}\xi_1\vert \nabla f_{\bc_{t}}\vert^2_{\alg^n}$ holds for $t=0,1,\ldots$ and some $\eta_t$ {and $\xi_1$} in the same manner as Example~\ref{ex:gd}.
Moreover, the function $f$ defined as Eq.~\eqref{eq:objective} is bounded below and $\nabla f_{\bc_t}$ is Lipschitz continuous on the set $\{\bc\in\alg^n\mid\ f(\bc)\le_{\alg}f(\bc_0)\}$.
In this case, if there exists a positive real value $\xi_2$ such that $\Vert  \nabla f_{\bc_{t+1}}- \nabla f_{\bc_t}\Vert_{\alg^n} \ge \xi_2\Vert \nabla f_{\bc_t}\Vert_{\alg^n}$, then we have
\begin{equation*}
  \xi_2 \Vert \nabla f_{\bc_t}\Vert_{\alg^n}\le L\Vert \bc_{t+1}-\bc_t\Vert_{\alg^n}
  \le L \eta_t \Vert \nabla f_{\bc_t}\Vert_{\alg^n},
\end{equation*}
where $L$ is a Liptschitz constant of $\nabla f_{\bc_t}$.
As a result, we have
\begin{equation*}
f(\bc_{t+1})\le_{\alg}f(\bc_t)-\eta_{t}\xi_1\vert \nabla f_{\bc_{t}}\vert^2_{\alg^n}
\le_{\alg} f(\bc_t)-\frac{\xi_1\xi_2}{L}\vert \nabla f_{\bc_{t}}\vert^2_{\alg^n},
\end{equation*}
which implies $\sum_{t=1}^T\vert \nabla f_{\bc_{t}}\vert^2_{\alg^n}\le_{\alg}L/(\xi_1\xi_2)(f(\bc_1)-f(\bc_{T+1}))$.
Since $f$ is bounded below, the sum $\sum_{t=1}^{\infty}\vert \nabla f_{\bc_{t}}\vert^2_{\alg^n}$ converges.
Therefore, $\vert \nabla f_{\bc_{t}}\vert^2_{\alg^n}\to 0$ as $t\to\infty$, i.e., the gradient $\nabla f_{\bc_{t}}$ in Eq.~\eqref{eq:gradient_decent} converges to $0$.
\end{exam}

\color{\motiv}
\begin{rem}\label{rem:gd_noncommutative}
It is possible to generalize the above method to the case where the objective function $f$ has the form $f(\bc)=\bc^*\mathbf{G}\bc$ for $\mathbf{G}\in\alg^{n\times n}$ {and $\alg$ is noncommutative}.
In this case, the derivative $Df_{\bc}$ is calculated as
\begin{equation*}
Df_{\bc}(u)=u^*\mathbf{G}\bc+\bc^*\mathbf{G}u.
\end{equation*}
Therefore, defining the gradient $\nabla f_{\bc}$ as $\nabla f_{\bc}=\mathbf{G}\bc$ results in $Df_{\bc}(-\eta\nabla f_{\bc})=-2\eta\bc^*\bG^2\bc\le_{\alg}0$ for a real positive value $\eta$, which allows us to derive the same result as Theorem~\ref{thm:gradient_decent}.
\end{rem}

\begin{rem}
The computational complexity of the PCA in RKHMs is higher than the standard PCA in RKHSs.
Indeed, in the case of RKHSs, the minimization problem is reduced to an eigenvalue problem of the Gram matrix with respect to given samples.
On the other hand, we solve the minimization problem~\eqref{eq:pca_min} by the gradient descent, and in each iteration step, we compute the gradient $\mathbf{d}$ in Eq.~\eqref{eq:gradient}.
Since the elements of $\mathbf{G}$ and $\mathbf{c}$ are in $\alg$, the computation of $\mathbf{d}$ involves the multiplication in $\alg$ such as multiplication of functions.
Even though we compute the multiplication in $\alg$ approximately in practice (see Subsection~\ref{sec:pca_numexp}), its computational cost is much higher than the multiplication in $\mathbb{C}$.
\end{rem}

\color{black}
\subsubsection{Minimization of the trace}\label{subsec:pca_trace}
In the case of $\alg=\blin{\mcl{W}}$, $p_j(x_i)$ and $p_j(x_i)^*$ in the problem~\eqref{eq:pca_min4} do not always commute.
\textcolor{\motiv}{Therefore, we restrict the solution to the form $p_j(x_i)=\sum_{i=1}^n\phi(x_i)c_i$ where each $c_i$ is a Hilbert--Schmidt operator and minimize the trace of the objective function of the problem~\eqref{eq:pca_min4} as follows:} 
\begin{equation}
\inf_{\mathbf{c}_j\in F,\ \{\sqrt{\mathbf{G}}\mathbf{c}_j\}_{j=1}^r\mbox{\footnotesize : ONS}}-\opn{tr}\bigg(\sum_{j=1}^r\mathbf{c}_j^*\mathbf{G}^2\mathbf{c}_j\bigg),\label{eq:pca_min6}
\end{equation}
where \textcolor{\motiv}{$F=\{\bc=[c_1,\ldots,c_n]\in\alg^n\mid\ c_i\mbox{ is a Hilbert--Schmidt operator for }i=1,\ldots,n\}$}.
If $\alg=\mat$, i.e., $\mcl{W}$ is a finite dimensional space, then we
solve the problem~\eqref{eq:pca_min6} by regarding $\bG$ as an $mn\times mn$ matrix and computing the eigenvalues and eigenvectors of $\bG$.
\begin{prop}\label{prop:min_trace}
Let $\alg=\mat$.
Let $\lambda_1,\ldots,\lambda_r\in\mathbb{C}$ and $\mathbf{v}_1,\ldots,\mathbf{v}_r\in\mathbb{C}^{mn}$ be the largest $r$ eigenvalues and the corresponding orthonormal eigenvectors of $\bG\in\mathbb{C}^{mn\times mn}$.
Then, $\bc_j=[\mathbf{v}_j,0,\ldots,0]\lambda_j^{-1/2}$ is a solution of the problem~\eqref{eq:pca_min6}. 
\end{prop}
\begin{proof}
Since the identity $\sum_{j=1}^r\mathbf{c}_j^*\mathbf{G}^2\mathbf{c}_j=\sum_{j=1}^r(\sqrt{\bG}\mathbf{c}_j)^*\mathbf{G}(\sqrt{\bG}\mathbf{c}_j)$ holds, any solution $\mathbf{c}_j$ of the problem~\eqref{eq:pca_min6} satisfies $\sqrt{\mathbf{G}}\mathbf{c}_j=\mathbf{v}_ju^*$ for a normalized vector $u\in\mathbb{C}^m$.
Thus, $p_j=\sum_{i=1}^n\phi(x_i)c_{i,j}$, where ${c}_{i,j}$ is the $i$-th element of $\lambda_j^{-1/2}[\mathbf{v}_j,0,\ldots,0]$, is a solution of the problem.
\end{proof}

\color{\motiv}
If $\mcl{W}$ is an infinite dimensional space, we rewrite the problem~\eqref{eq:pca_min6} with the Hilbert--Schmidt norm as follows:
\begin{equation}
\inf_{\mathbf{c_j}\in F,\ \{\sqrt{\mathbf{G}}\mathbf{c}_j\}_{j=1}^r\mbox{\footnotesize : ONS}}-\sum_{j=1}^r\Vert\mathbf{G}\mathbf{c}_j\Vert^2_{F},\label{eq:pca_min8}
\end{equation}
where $\Vert \mathbf{c}\Vert_{F}^2=\sum_{i=1}^n\Vert c_i\Vert_{\opn{HS}}^2$ and $\Vert\cdot\Vert_{\opn{HS}}$ is the Hilbert--Schmidt norm for Hilbert--Schmidt operators.
Similar to Eq.~\eqref{eq:pca_min3}, we rearrange the problem~\eqref{eq:pca_min8} to the following problem by adding a penalty term:
\begin{equation}
\inf_{\mathbf{c}\in F}-\Vert\mathbf{G}\mathbf{c}\Vert^2_{F}+\lambda\big\vert\big\Vert \sqrt{\mathbf{Gc}}\big\Vert^2_{F}-1\big\vert,\label{eq:pca_min7}
\end{equation}
where $\lambda$ is a real positive weight for the penalty term.
Then, we can apply the standard gradient descent method in Hilbert spaces to the problem in $F$~\citep{smyrlis04} since $F$ is the Hilbert space equipped with the Hilbert--Schmidt inner product.
Similar to the case of Eq.~\eqref{eq:pca_min3}, for $r>1$, let $\bc_1$ be a solution of the problem~\eqref{eq:pca_min7}.
Then, we solve the same problem in the orthogonal complement of the space spanned by $\{\bc_1\}$ and set the solution of this problem as $\bc_2$.
Then, we solve the same problem in the orthogonal complement of the space spanned by $\{\bc_1,\bc_2\}$ and repeat this procedure to obtain solutions $\bc_1,\ldots\bc_r$.
\color{black}

\subsubsection{Numerical examples}\label{sec:pca_numexp}
\paragraph{Experiments with synthetic data}
We applied the above PCA with $\alg=L^{\infty}([0,1])$ to functional data.
We randomly generated three kinds of sample-sets from the following functions of two variables on $[0,1]\times [0,1]$:
\begin{align*}
y_1(s,t)=e^{10(s-t)},\quad y_2(s,t)=10st,\quad y_3(s,t)=\cos(10(s-t)).
\end{align*}
Each sample-set $i$ is composed of 20 samples with random noise.
We denote these samples by $x_1,\ldots,x_{60}$.
The noise was randomly drawn from the Gaussian distribution with mean $0$ and standard deviation $0.3$. 
Since $L^{\infty}([0,1])$ is commutative, we applied the gradient descent proposed in Subsection~\ref{subsec:gradient} to solve the problem~\eqref{eq:pca_min}.
The parameters were set as $\lambda=0.1$ and $\eta_t=0.01$.
We set the $L^{\infty}([0,1])$-valued positive definite kernel $k$ as $(k(x_i,x_j))(t)=\int_0^1\int_0^1(t-x_i(s_1,s_2))(t-x_j(s_1,s_2))ds_1ds_2$ (see Example~\ref{ex:pdk1}.1).
\textcolor{\exp}{Since $(k(x_i,x_j))(t)$ is a polynomial of $t$, all the computations on $\alg$ result in polynomials.
Thus, the results are obtained by keeping coefficients of the polynomials.}
Moreover, we set $\bc_0$ as the constant function $[1,\ldots,1]^T\in\alg^n$ and computed $\bc_1,\bc_2,\ldots$ according to Eq.~\eqref{eq:gradient_decent}.
For comparison, we also vectorized the samples \textcolor{\exp}{by discretizing $y_i$ at $121=11\times 11$ points composed of 11 equally spaced points in $[0,1]$ ($0,0.1,\ldots,1$)} and applied the standard kernel PCA in the RKHS associated with the Laplacian kernel on $\mathbb{R}^{121}$.
The results are illustrated in Figure~\ref{fig:pca}.
Since the samples are contaminated by the noise, the PCA in the RKHS cannot separate three sample-sets.
On the other hand, the $L^{\infty}([0,1])$-valued weights of principal components obtained by the proposed PCA in the RKHM reduce the information of the samples as functions.
As a result, it clearly separates three sample-sets.

Figure~\ref{fig:gd} shows the convergence of the proposed gradient descent.
In this example, we only compute the first principal components, hence $r$ is set as $1$.
For the objective function $f$ defined as $f(\bc)=-\bc^*\bG^2\bc+\lambda\bc\bG\bc\bc^*\bG\bc+\lambda\bc^*\bG\bc$, functions $f(\bc_t)\in L^{\infty}([0,1])$ for $t=0,\ldots,9$ are illustrated.
We can see $f(\bc_{t+1})<f(\bc_t)$ and $f(\bc_t)$ gradually approaches a certain function as $t$ grows.
\begin{figure}[t]
    \centering
    \includegraphics[scale=0.4]{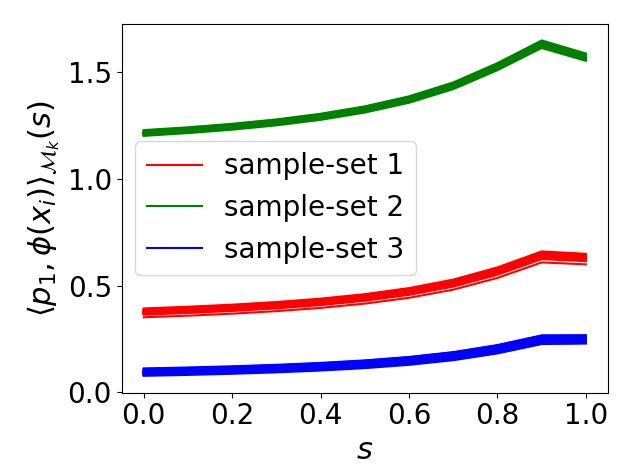}
    \includegraphics[scale=0.4]{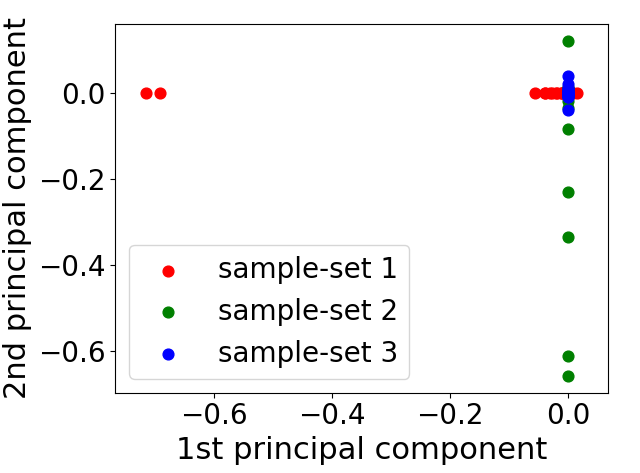}
    \caption{The $L^{\infty}([0,1])$-valued first principal components obtained by the proposed PCA in an RKHM (left) and the real-valued first and second principal components obtained by the standard PCA in an RKHS (right)}
    \label{fig:pca}
\end{figure}

\begin{figure}[t]
    \centering
    \includegraphics[scale=0.3]{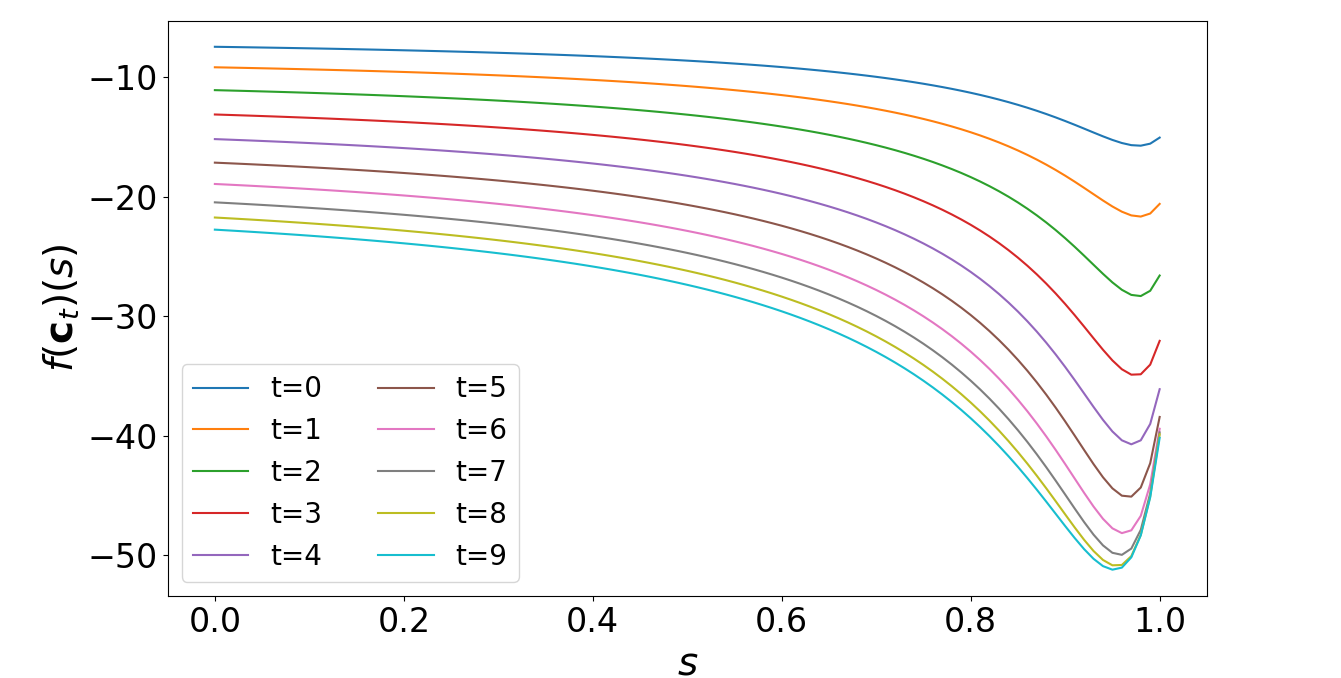}
    \caption{The convergence of the function $f(\bc_t)$ along $t$.}
    \label{fig:gd}
\end{figure}

\color{\exp}
\paragraph{Experiments with real-world data}
To show the proposed PCA with RKHMs extracts the continuous dependencies of samples on the principal axes as we insisted in Section~\ref{sec:motivation}, we conducted experiments with climate data in Japan\footnote{available at \url{https://www.data.jma.go.jp/gmd/risk/obsdl/}}.
The data is composed of the maximum and minimum daily temperatures at 47 prefectures in Japan in 2020.
The original data is illustrated in Figure~\ref{fig:data_climate_ori}.
The red line represents the temperature at Hokkaido, the northernmost prefecture in Japan and the blue line represents that at Okinawa, the southernmost prefecture in Japan. 
We respectively fit the maximum and minimum temperatures at each location to the Fourier series $a_0+\sum_{i=1}^{10}(a_i\cos(it)+b_i\sin(it))$.
The fitted functions $x_1,\ldots,x_{47}\in C([0,366],\mathbb{R}^2)$ are illustrated in Figure~\ref{fig:data_climate}.
Then, we applied the PCA with the RKHM associated with the $L^{\infty}([0,366])$-valued positive definite kernel $(k(x,y))(t)=e^{-\Vert x(t)-y(t)\Vert_2^2}$.
Let $\mcl{F}=\{a_0+\sum_{i=1}^{10}(a_i\cos(it)+b_i\sin(it))\mid\ a_i,b_i\in\mathbb{R}\}\subseteq L^2([0,366])$.
We project $k(x,y)$ onto $\mcl{F}$.
Then, for $c,d\in\mcl{F}$, $c+d\in\mcl{F}$ is satisfied, but $cd\in\mcl{F}$ is not always satisfied.
Thus, we approximate $cd$ with $a_0+\sum_{i=1}^{N}(a_i\cos(it)+b_i\sin(it))$ for $N\le 10$ to restrict all the computations in $\mcl{F}$ in practice.
{Here, to remove high frequency components corresponding to noise and extract essential information, we set $N=3$.}
Figure~\ref{fig:pca_climate_rkhm} shows the computed $L^{\infty}([0,366])$-valued weights of the first principal axis in the RKHM, which continuously depends on time.
The red and blue lines correspond to Hokkaido and Okinawa, respectively.
We see these lines are well-separated from other lines corresponding to other prefectures.
For comparison, we also applied the PCA in RKHSs to discrete time data.
First, we respectively applied the standard kernel PCA with RKHSs to the original temperature each day and obtained real-valued weights of the first principal components.
Here, we used the complex-valued Gaussian kernel $\tilde{k}(x,y)=e^{-\Vert x-y\Vert_2^2}$.
Then, we connected the results and obtained Figure~\ref{fig:pca_climate_rkhs_ori}. 
Since the original data is not smooth, the PCA amplifies the non-smoothness, which provides meaningless results.
Next, we respectively applied the standard kernel PCA with the RKHS to the value of the fitted Fourier series each day and obtained real-valued weights of the first principal components.
Then, similar to the case of Figure~\ref{fig:pca_climate_rkhs_ori}, we connected the results and obtained Figure~\ref{fig:pca_climate_rkhs}.
In this case, the extracted features somewhat capture the continuous behaviors of the temperatures.
{However, the PCA in the RKHS amplifies high frequency components, which correspond to noise.}
Therefore, the result fails to separate the temperatures of Hokkaido and Okinawa, whose behaviors are significantly different as illustrated in Figure~\ref{fig:data_climate_ori}.
On the other hand, the PCA in the RKHM captures the feature of each sample as a function {and removes nonessential high frequency components}, which results in separating functional data properly.

\begin{figure}[t]
    \centering
    \includegraphics[scale=0.4]{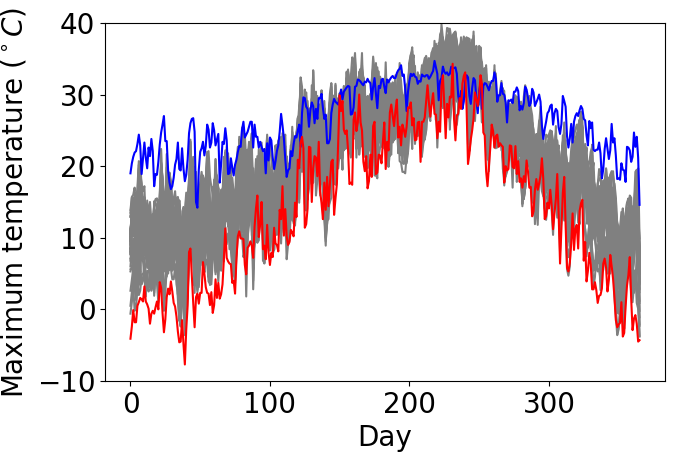}
    \includegraphics[scale=0.4]{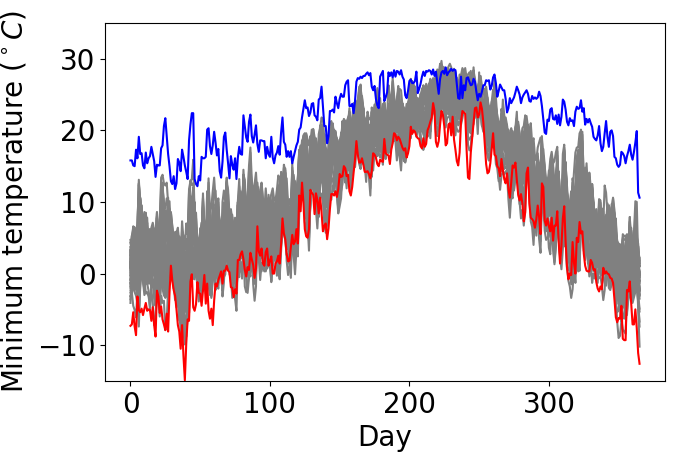}
    \caption{Original climate data at 47 locations}
    \label{fig:data_climate_ori}
\end{figure}

\begin{figure}[t]
    \centering
    \includegraphics[scale=0.4]{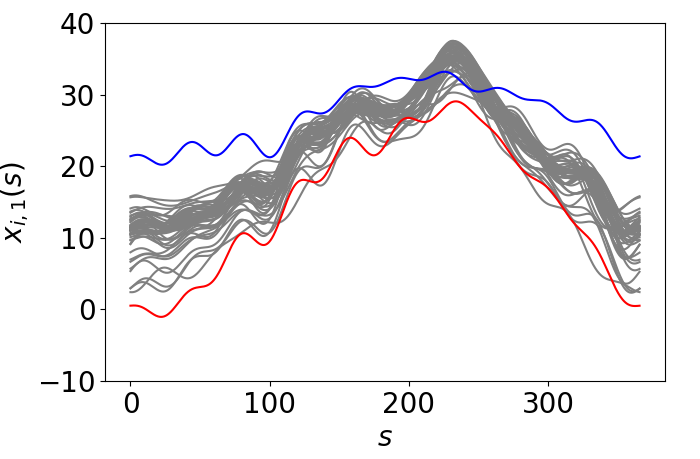}
    \includegraphics[scale=0.4]{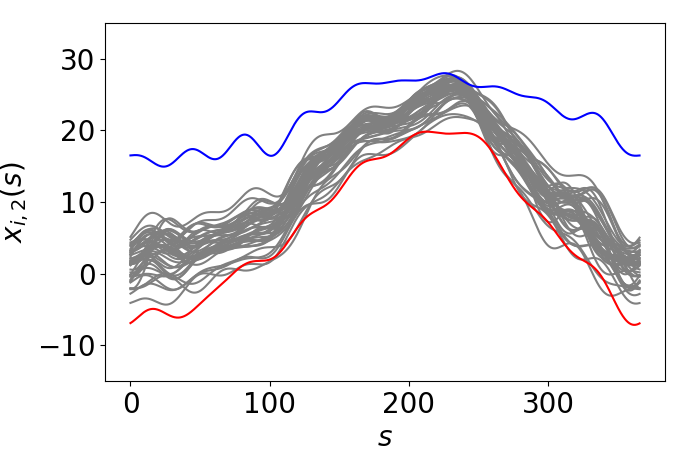}
    \caption{Fitted Fourier series}
    \label{fig:data_climate}
\end{figure}

\begin{figure}[t]
    \centering
    \subfigure[PCA with RKHMs for the fitted Fourier series]{\includegraphics[scale=0.4]{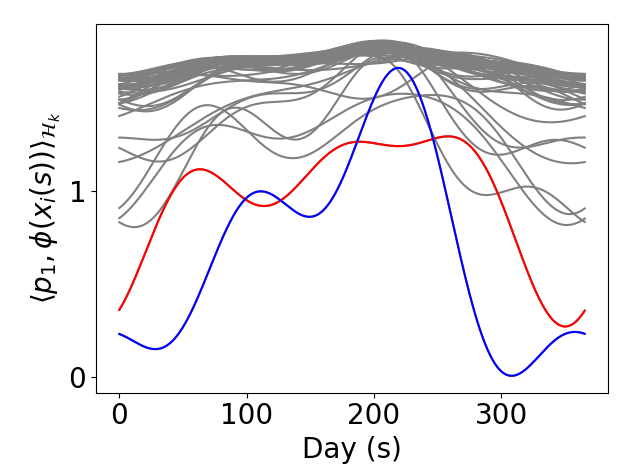}\label{fig:pca_climate_rkhm}}
    \subfigure[PCA with RKHSs for the original data]{\includegraphics[scale=0.4]{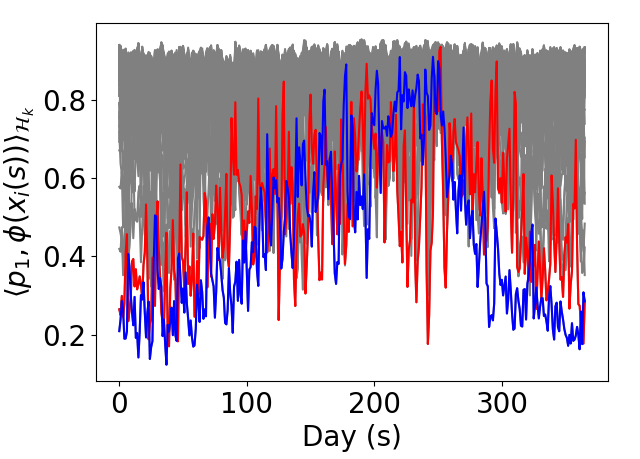}\label{fig:pca_climate_rkhs_ori}}\\
    \subfigure[PCA with RKHSs for the fitted Fourier series]{\includegraphics[scale=0.4]{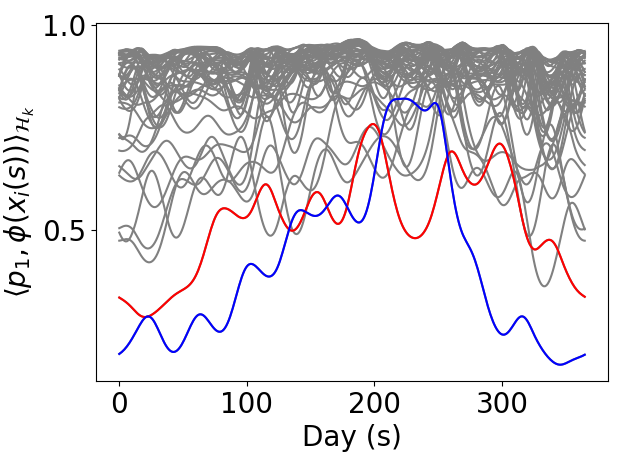}\label{fig:pca_climate_rkhs}}
    \caption{Principal components of PCA for climate data}
    \label{fig:pca_climate}
\end{figure}

\color{black}
\subsection{Time-series data analysis}\label{subsec:pf}
\color{\motiv}
The problem of analyzing dynamical systems from data by using Perron--Frobenius operators and their adjoints (called Koopman operators), which are linear operators expressing the time evolution of dynamical systems, has recently attracted attention in various fields~\citep{mezic12,mezic17,takeishi17,takeishi17-2,lusch17}.
And, several methods for this problem using RKHSs have also been proposed~\citep{kawahara16,klus17,ishikawa18,hashimoto19,fujii19}.
In these methods, sequential data is supposed to be generated from dynamical systems and is analyzed through Perron--Frobenius operators in RKHSs.
{To analyze the time evolution of functional data, we generalize Perron--Frobenius operators defined in RKHSs to those in RKHMs by using an operator-valued positive definite kernel describing similarities between pairs of functions.}

\paragraph{Defining Perron--Frobenius operators in RKHMs}
We consider the RKHM and vvRKHS associated with an operator-valued positive definite kernel.
VvRKHSs are associated with operator-valued kernels, and as we stated in Lemma~\ref{lem:pdk_equiv}, those operator-valued kernels are special cases of $C^*$-algebra-valued positive definite kernels.
Here, we discuss the advantage of RKHMs over vvRKHSs.
Comparing with vvRKHSs, RKHMs have enough representation power for preserving continuous behaviors of infinite dimensional operator-valued kernels, while vvRKHSs are not sufficient for preserving such behaviors.
Let $\mcl{W}$ be a Hilbert space, let $k:\mcl{X}\times \mcl{X}\to\mcl{B}(\mcl{W})$ be an operator-valued positive definite kernel on a data space $\mcl{X}$, and let $\hil_k^{\opn{v}}$ be the vvRKHS associated with $k$.
Since the inner products in vvRKHSs have the form $\blacket{w,k(x,y)h}$ for $w,h\in\mcl{W}$ and $x,y\in\mcl{X}$, if $\mcl{W}$ is a $d$-dimensional space, putting $w$ as $d$ linearly independent vectors in $\mcl{W}$ reconstructs $k(x,y)$.
However, if $\mcl{W}$ is an infinite dimensional space, we need infinitely many $w$ to reconstruct $k(x,y)$, and we cannot recover the continuous behavior of the operator $k(x,y)$ with finitely many $w$.
For example, let $\mcl{X}=C(\Omega,\mcl{Y})$ and $\mcl{W}=L^2(\Omega)$ for a compact measure space $\Omega$ and a topological space $\mcl{Y}$.
Let $(k(x,y)w)(s)=\int_{t\in \Omega}\tilde{k}(x(s),y(t))w(t)dt$, where $\tilde{k}$ is a complex-valued positive definite kernel on $\mcl{Y}$ (see Example~\ref{ex:pdk1}.4).
The operator $k(x,y)$ for functional data $x$ and $y$ describes the continuous changes of similarities between function $x$ and $y$.
However, the estimation or prediction of the operator $k(x,y)$ in vvRKHSs fails to extract the continuous behavior of the function $\tilde{k}(x(s),y(t))$ in the operator $k(x,y)$ since vectors in vvRKHSs have the form $k(\cdot,y)w$  and we cannot completely recover $k(x,y)$ with finitely many vectors in the vvRKHS.
On the other hand, RKHMs have enough information to recover $k(x,y)$ since it is just the inner product between two vectors $\phi(x)$ and $\phi(y)$.
\color{black}

\subsubsection{Perron--Frobenius operator in RKHSs}\label{subsec:pf_rkhs_review}
We briefly review the definition of the Perron-Frobenius operator on RKHS and existing methods for analysis of time-series data through Perron--Frobenius operators and construction of their estimations~\citep{kawahara16,hashimoto19} .
First, we define Perron--Frobenius operators in RKHSs.
Let $\{x_0,x_1,\ldots\}\subseteq \mcl{X}$ be time-series data.
We assume it is generated from the following deterministic dynamical system:
\begin{equation}
x_{i+1}=f(x_i),\label{eq:ds_rkhs}
\end{equation}
where $f:\mcl{X}\to\mcl{X}$ is a map.
By embedding $x_i$ and $f(x_i)$ in an RKHS $\hil_{\tilde{k}}$ associated with a positive definite kernel $\tilde{k}$ and the feature map $\tilde{\phi}$, dynamical system~\eqref{eq:ds_rkhs} in $\mcl{X}$ is transformed into that in the RKHS as
\begin{equation*}
\tilde{\phi}(x_{i+1})=\tilde{\phi}(f(x_i)).
\end{equation*}
The Perron--Frobenius operator $\tilde{K}$ in the RKHS is defined as a linear operator on $\hil_{\tilde{k}}$ satisfying
\begin{equation*}
\tilde{K}\tilde{\phi}(x):=\tilde{\phi}(f(x))
\end{equation*}
for $x\in\mcl{X}$.
If $\{\tilde{\phi}(x)\mid\ x\in\mcl{X}\}$ is linearly independent, $\tilde{K}$ is well-defined as a linear map in the RKHS.
For example, if $\tilde{k}$ is a universal kernel~\citep{sriperumbudur11} such as the Gaussian or Laplacian kernel on $\mcl{X}=\mathbb{R}^d$, $\{\tilde{\phi}(x)\mid\ x\in\mcl{X}\}$ is linearly independent.

By considering eigenvalues and the corresponding {eigenvectors} of $\tilde{K}$, we can understand the long-time behavior of the dynamical system.
For example, let $v_1,\ldots,v_m$ be the eigenvectors with respect to eigenvalue 1 of $\tilde{K}$.  
We project the vector $\tilde{\phi}(x_0)$ onto the subspace spanned by $v_1,\dots,v_m$.
We denote the projected vector by $v$.
Then, for $\alpha=1,2,\ldots$, we have
\begin{equation*}
\tilde{\phi}(x_{\alpha})=\tilde{K}^{\alpha}(v+v^{\perp})=v+\tilde{K}^{\alpha}v^{\perp},
\end{equation*}
where $v^{\perp}=\tilde{\phi}(x_0)-v$.
Therefore, by calculating a pre-image of $v$, we can extract the time-invariant component of the dynamical system with the initial value $x_0$.

For practical uses of the above discussion, we construct an estimation of $\tilde{K}$ only with observed data $\{x_0,x_1,\ldots\}\subseteq \mcl{X}$ as follows: 
We project $\tilde{K}$ onto the finite dimensional subspace spanned by $\{\tilde{\phi}(x_0),\ldots,\tilde{\phi}(x_{T-1})\}$.
Let $\tilde{W}_T:=[\tilde{\phi}(x_0),\ldots,\tilde{\phi}(x_{T-1})]$ and $\tilde{W}_T=\tilde{Q}_T\tilde{\mathbf{R}}_T$ be the QR decomposition of $\tilde{W}_T$ in the RKHS.
Then, the Perron--Frobenius operator $\tilde{K}$ is estimated by projecting $\tilde{K}$ onto the space spanned by  $\{\tilde{\phi}(x_0),\ldots,\tilde{\phi}(x_{T-1})\}$.
Since $\tilde{K}\tilde{\phi}(x_i):=\tilde{\phi}(f(x_i))=\tilde{\phi}(x_{i+1})$ holds, we construct an estimation $\tilde{\mathbf{K}}_T$ of $\tilde{K}$ as follows:
\begin{align*}
\tilde{\mathbf{K}}_T:&= \tilde{Q}_T^*\tilde{K}\tilde{Q}_T=\tilde{Q}_T^*\tilde{K}\tilde{W}_T\tilde{\mathbf{R}}_T^{-1}
=\tilde{Q}_T^*[\tilde{\phi}(x_1),\ldots,\tilde{\phi}(x_{T})]\tilde{\mathbf{R}}_T^{-1},
\end{align*}
which can be computed only with observed data.

\subsubsection{Perron--Frobenius operator in RKHMs}\label{subsec:pf_rkhm}
Existing analyses \citep{kawahara16, hashimoto19} of time-series data with Perron--Frobenius operators are addressed only in RKHSs.
In the remaining parts of this section, we generalize the existing analyses to RKHM to extract continuous behaviors of functional data.
We consider the case where time-series is functional data.
Let $\Omega$ be a compact measure space, $\mcl{Y}$ be a topological space, $\mcl{X}=C(\Omega,\mcl{Y})$, $\alg=\mcl{B}(L^2(\Omega))$, and $\{x_0,x_1,\ldots\}\subseteq\mcl{X}$ be functional time-series data.
Let $k:\mcl{X}\times\mcl{X}\to \alg$ be defined as $(k(x,y)w)(s)=\int_{t\in\Omega}\tilde{k}(x(s),y(t))w(t)dt$, where $\tilde{k}:\mcl{Y}\times \mcl{Y}\to\mathbb{C}$ is a complex-valued positive definite kernel (see Example~\ref{ex:pdk1}.4 and the last paragraph of Section~\ref{sec:motivation}).
The operator $k(x,y)$ is the integral operator whose integral kernel is $\tilde{k}(x(s),y(t))$.
We define a Perron--Frobenius operator in the RKHM $\modu_k$ associated with the above kernel $k$ as an $\alg$-linear operator satisfying 
\begin{equation*}
{K}{\phi}(x)={\phi}(f(x))
\end{equation*}
for $x\in\mcl{X}$.
We assume $K$ is well-defined on a dense subset of $\modu_k$.
Then, for $\alpha,\beta=1,2,\ldots$, we have
\begin{equation*}
k(x_{\alpha},x_{\beta})=\blacket{\phi(x_{\alpha}),\phi(x_{\beta})}_{\modu_k}=\bblacket{K^{\alpha}\phi(x_0),K^{\beta}\phi(x_0)}_{\modu_k}.
\end{equation*}
Therefore, by estimating $K$ in the RKHM $\modu_k$, we can extract the similarity between arbitrary points of functions $x_{\alpha}$ and $x_{\beta}$.
Moreover, the eigenvalues and eigenvectors of $K$ provide us a decomposition of the similarity $k(x_{\alpha},x_{\beta})$ into a time-invariant term and time-dependent term.
Since $K$ is a linear operator on a Banach space $\modu_k$, eigenvalues and eigenvectors of $K$ are available. 
Let $v_1,\ldots,v_m\in\modu_k$ be the eigenvectors with respect to eigenvalue $1$ of $K$.
We project the vector $\phi(x_0)$ onto the submodule spanned by $v_1,\ldots,v_m$, which is denoted by $\mcl{V}$.
Let $\{q_1,\ldots,q_m\}\subseteq\modu_k$ be an orthonormal basis of $\mcl{V}$ and let $v=\sum_{i=1}^mq_i\blacket{q_i,\phi(x_0)}_{\modu_k}$.
Then, we have
\begin{equation}
k(x_{\alpha},x_{\beta})=\bblacket{K^{\alpha}(v+v^{\perp}),K^{\beta}(v+v^{\perp})}_{\modu_k}=\blacket{v,v}_{\modu_k}+r({\alpha},\beta),\label{eq:modedec_rkhm}
\end{equation}
where $v^{\perp}=\phi(x_0)-v$ and $r(\alpha,\beta)=\blacket{K^{\alpha}v,K^{\beta}v^{\perp}}_{\modu_k}+\blacket{K^{\alpha}v^{\perp},K^{\beta}v}_{\modu_k}+\blacket{K^{\alpha}v^{\perp},K^{\beta}v^{\perp}}_{\modu_k}$.
Therefore, the term $\blacket{v,v}_{\modu_k}$ provides us with the information about time-invariant similarities.

\begin{rem}\label{rmk:vvrkhs}
We can also consider the vvRKHS $\hil_k^{\opn{v}}$ with respect to the operator-valued kernel $k$.
Here, we discuss the difference between the case of vvRKHS and RKHM.
The Perron--Frobenius operator $K^{\opn{v}}$ in a vvRKHS $\vvrkhs$~\citep{fujii19} is defined as a linear operator satisfying
\begin{equation*}
K^{\opn{v}}\phi(x)w=\phi(f(x))w
\end{equation*}
for $x\in\mcl{X}$ and $w\in\mcl{W}$.
However, with finitely many vectors in $\hil_k^{\opn{v}}$, we can only recover an projected operator $UU^*k(x_{\alpha},x_{\beta})UU^*$, where $N\in\mathbb{N}$, $U=[u_1,\ldots,u_N]$, and $\{u_1,\ldots,u_N\}$ is an orthonormal system on $\mcl{W}$ as follows:
\begin{equation}
U^*k(x_{\alpha},x_{\beta})U=\big[\blacket{\phi(x_s)u_i,\phi(x_t)u_j}_{\vvrkhs}\big]_{i,j}=\big[\bblacket{(K^{\opn{v}})^{\alpha}\phi(x_0)u_i,(K^{\opn{v}})^{\beta}\phi(x_0)u_j}_{\vvrkhs}\big]_{i,j}.\label{eq:time_invariant_rkhm}
\end{equation}
Furthermore, let $v_1,\ldots,v_m\in\modu_k$ be the eigenvectors with respect to eigenvalue $1$ of $K^{\opn{v}}$.
Let $\{q_1,\ldots,q_m\}\subseteq\vvrkhs$ be an orthonormal basis of the subspace spanned by $v_1,\ldots,v_m$ and let $\tilde{v}_j=\sum_{i=1}^mq_i\blacket{q_i,\phi(x_0)u_j}_{\vvrkhs}$.
Then, we have
\begin{equation}
U^*k(x_{\alpha},x_{\beta})U=\big[\sblacket{(K^{\opn{v}})^{\alpha}(\tilde{v}_i+\tilde{v}_i^{\perp}),(K^{\opn{v}})^{\beta}(\tilde{v}_j+\tilde{v}_j^{\perp})}_{\vvrkhs}\big]_{i,j}=[\blacket{\tilde{v}_i,\tilde{v}_j}_{\vvrkhs}]_{i,j}+\tilde{r}(\alpha,\beta),\label{eq:time_invariant_vvrkhs}
\end{equation}
where $\tilde{v}_i^{\perp}=\phi(x_0)u_i-\tilde{v}_i$ and $\tilde{r}(\alpha,\beta)=[\sblacket{(K^{\opn{v}})^{\alpha}\tilde{v}_i,(K^{\opn{v}})^{\beta}\tilde{v}_j^{\perp}}_{\vvrkhs}+\sblacket{(K^{\opn{v}})^{\alpha}\tilde{v}_i^{\perp},(K^{\opn{v}})^{\beta}\tilde{v}_j}_{\vvrkhs}+\sblacket{(K^{\opn{v}})^{\alpha}\tilde{v}_i^{\perp},(K^{\opn{v}})^{\beta}\tilde{v}_j^{\perp}}_{\vvrkhs}]_{i,j}$.
Therefore, with vvRKHSs, we cannot recover the continuous behavior of the operator $k(x,y)$ which encodes similarities between functions $x$ and $y$.
\end{rem}

\subsubsection{Estimation of Perron--Frobenius operators in RKHMs}
In practice, we only have time-series data but do not know the underlying dynamical system and its Perron--Frobenius operator in an RKHM.
Therefore, we consider estimating the Perron--Frobenius operator only with the data.
To do so, we generalize the Gram--Schmidt orthonormalization algorithm to Hilbert $C^*$-modules to apply the QR decomposition and project Perron--Frobenius operators onto the submodule spanned by $\{\phi(x_0),\ldots,\phi(x_{T-1})\}$.
The Gram--Schmidt orthonormalization in Hilbert modules is theoretically investigated by~\citet{cnops92}.
Here, we develop a practical method for our settings.
Then, we can apply the  decomposition~\eqref{eq:modedec_rkhm}, proposed in Subsection~\ref{subsec:pf_rkhm}, of the estimated operator regarding eigenvectors.
Since we are considering the RKHM associated with the integral operator-valued positive definite kernel defined in the first part of Subsection~\ref{subsec:pf_rkhm}, we assume $\alg=\mcl{B}(\mcl{W})$ and we denote by $\modu$ a Hilbert $C^*$-module over $\alg$ throughout this subsection.
Note that integral operators are compact.

We first develop a normalization method for Hilbert $C^*$-modules.
In $C^*$-algebras, nonzero elements are not always invertible, which is the main difficulty of the normalization in Hilbert $C^*$-modules.
However, by carefully applying the definition of normalized (see Definition~\ref{def:normalized}), we can construct a normalization method.
\begin{prop}[Normalization]\label{prop:normalized_property}
Let $\epsilon\ge 0$ and let $\hat{q}\in\mcl{M}$ satisfy $\Vert \hat{q}\Vert_{\modu}>\epsilon$.
Assume $\sblacket{\hat{q},\hat{q}}_{\modu}$ is compact.
Then, there exists $\hat{b}\in\alg$ such that $\Vert \hat{b}\Vert_{\alg}<1/\epsilon$ and $q:=\hat{q}\hat{b}$ is normalized.
In addition, there exists $b\in\alg$ such that 
$\Vert \hat{q}-qb\Vert_{\modu}\le\epsilon$.
\end{prop}
\begin{proof}
Let $\lambda_1\ge\lambda_2\ge\cdots\ge 0$ be the eigenvelues of the compact operator $\blacket{\hat{q},\hat{q}}_{\modu}$, and $m':=\max\{j\mid\ \lambda_j>\epsilon^2\}$.
Since $\sblacket{\hat{q},\hat{q}}_{\modu}$ is positive and compact,
it admits the spectral decomposition $\sblacket{\hat{q},\hat{q}}_{\modu}=\sum_{i=1}^{\infty}\lambda_iv_iv_i^*$, where $v_i$ is the orthonormal eigenvector with respect to $\lambda_i$.
Also, since $\lambda_1=\Vert\hat{q}\Vert_{\modu}^2>\epsilon^2$, we have $m'\ge 1$.
Let $\hat{b}=\sum_{i=1}^{m'}1/\sqrt{\lambda_i}v_iv_i^*$. 
By the definition of $\hat{b}$, $\Vert \hat{b}\Vert_{\alg}=1/\sqrt{\lambda_{m'}}<1/\epsilon$ holds.
Also, we have
\begin{align*}
\sblacket{\hat{q}\hat{b},\hat{q}\hat{b}}_{\modu}&=\hat{b}^*\blacket{\hat{q},\hat{q}}_{\modu}\hat{b}
=\sum_{i=1}^{m'}\frac{1}{\sqrt{\lambda_i}}v_iv_i^*\sum_{i=1}^{\infty}\lambda_iv_iv_i^*\sum_{i=1}^{m'}\frac{1}{\sqrt{\lambda_i}}v_iv_i^*
=\sum_{i=1}^{m'}v_iv_i^*.
\end{align*}
Thus, $\sblacket{\hat{q}\hat{b},\hat{q}\hat{b}}_{\modu}$ is a nonzero orthogonal projection.

In addition, let $b=\sum_{i=1}^{m'}\sqrt{\lambda_i}v_iv_i^*$.
Since $\hat{b}b=\sum_{i=1}^{m'}v_iv_i^*$, the identity $\sblacket{\hat{q},\hat{q}\hat{b}b}=\sblacket{\hat{q}\hat{b}b,\hat{q}\hat{b}b}$ holds, and 
we obtain
\begin{align*}
\sblacket{\hat{q}-qb,\hat{q}-qb}_{\modu}
&=\sblacket{\hat{q}-\hat{q}\hat{b}b,\hat{q}-\hat{q}\hat{b}b}_{\modu}
=\sblacket{\hat{q},\hat{q}}-\sblacket{\hat{q}\hat{b}b,\hat{q}\hat{b}b}_{\modu}\nn\\
&=\sum_{i=1}^{\infty}{\lambda_i}v_iv_i^*-\sum_{i=1}^{m'}{\lambda_i}v_iv_i^*=\sum_{i=m'+1}^{\infty}{\lambda_i}v_iv_i^*.
\end{align*}
Thus, $\Vert \hat{q}-q\hat{b}\Vert_{\modu}=\sqrt{\lambda_{m'+1}}\le\epsilon$ holds, which completes the proof of the proposition.
\end{proof}
Proposition~\ref{prop:normalized_property} and its proof provide a concrete procedure to obtain normalized vectors in $\modu$.
This enables us to compute an orthonormal basis practically by applying Gram-Schmidt orthonormalization with respect to $\alg$-valued inner product.
\begin{prop}[Gram-Schmidt orthonormalization]\label{prop:gram-schmidt}
Let $\{w_i\}_{i=1}^{\infty}$ be a sequence in $\modu$.
Assume $\blacket{w_i,w_j}_{\modu}$ is compact for any $i,j=1,2,\ldots$.
Consider the following scheme for $i=1,2,\ldots$ and $\epsilon\ge 0$:
\begin{equation}
\begin{aligned}
\hat{q}_j&=w_j-\sum_{i=1}^{j-1}q_i\blacket{q_i,w_j}_{\modu},\quad
q_j=\hat{q}_j\hat{b}_j\quad \mbox{if }\;\Vert \hat{q}_j\Vert_{\modu}>\epsilon,\\
q_j&=0\quad\mbox{o.w.},
\end{aligned}\label{eq:gram-schmidt}
\end{equation}
where $\hat{b}_j$ is defined as $\hat{b}$ in Proposition~\ref{prop:normalized_property} by setting $\hat{q}=\hat{q}_j$.
Then, $\{q_j\}_{j=1}^{\infty}$ is an orthonormal basis in $\modu$ such that any $w_j$ is contained in the $\epsilon$-neighborhood of the module spanned by $\{ q_j\}_{j=1}^{\infty}$.
\end{prop}
\begin{rem}\label{rem:trade_off}
We give some remarks about the role of $\epsilon$ in Propositions~\ref{prop:normalized_property}.  
The vector $\hat{q}_i$ can always be reconstructed by $w_i$ only when $\epsilon=0$.
This is because the information of the spectrum of $\blacket{\hat{q}_i,\hat{q}_i}_{\modu}$ may be lost if $\epsilon>0$.
However, if $\epsilon$ is sufficiently small, 
we can reconstruct $\hat{q}_i$ with a small error.
On the other hand, the norm of $\hat{b}_i$ can be large if $\epsilon$ is small, and the computation of $\{q_i\}_{i=1}^{\infty}$ can become numerically unstable.
This corresponds to the trade-off between the theoretical accuracy and numerical stability.
\end{rem}
To prove Proposition~\ref{prop:gram-schmidt}, we first prove the following lemmas.
\begin{lem}\label{lem:inprodequiv}
For $c\in\mcl{A}$ and $v\in\mcl{M}$, if $\blacket{v,v}_{\modu}c=\blacket{v,v}_{\modu}$, then $vc=v$ holds.
\end{lem}
\begin{proof}
If $\blacket{v,v}_{\modu}c=\blacket{v,v}_{\modu}$, then $c^*\blacket{v,v}_{\modu}=\blacket{v,v}_{\modu}$ and we have
\begin{align*}
 &\blacket{vc-v,vc-v}_{\modu}=c^*\blacket{v,v}_{\modu}c-c^*\blacket{v,v}_{\modu}-\blacket{v,v}_{\modu}c+\blacket{v,v}_{\modu}=0,
\end{align*}
 which implies $vc=v$.
\end{proof}
\begin{lem}\label{cor:inprodequiv}
If $q\in\mcl{M}$ is normalized, then $q\blacket{q,q}_{\modu}=q$ holds.
\end{lem}
\begin{proof}
Since $\blacket{q,q}_{\modu}$ is a projection, $\blacket{q,q}_{\modu}\blacket{q,q}_{\modu}=\blacket{q,q}_{\modu}$ holds.
Therefore, letting $c=\blacket{q,q}_{\modu}$ and $v=q$ in Lemma~\ref{lem:inprodequiv} completes the proof of the lemma.
\end{proof}
\begin{myproof}{Proof of Proposition~\ref{prop:gram-schmidt}}
By Proposition~\ref{prop:normalized_property}, $q_j$ is normalized, and for $\epsilon\ge 0$, there exists $b_j\in\alg$ such that $\Vert \hat{q}_j-q_jb_j\Vert_{\modu}\le \epsilon$.
Therefore, by the definition of $\hat{q}_j$, 
$\Vert w_j-v_j\Vert_{\modu}\le\epsilon$ holds, where $v_j$ is a vector in the module spanned by $\{q_j\}_{j=0}^{\infty}$ which is defined as $v_j=\sum_{i=1}^{j-1}q_i\blacket{q_i,w_j}_{\modu}-q_jb_j$.
This means that the $\epsilon$-neighborhood of the space spanned by $\{q_j\}_{j=1}^{\infty}$ contains $\{w_j\}_{j=1}^{\infty}$.
Next, we show the orthogonality of $\{q_j\}_{j=1}^{\infty}$.
Assume $q_1,\ldots,q_{j-1}$ are orthogonal to each other.
For $i<j$, the following identities are deduced by Lemma~\ref{cor:inprodequiv}:
\begin{align*}
\blacket{q_j,q_i}_{\modu}&=\hat{b}_t^*\blacket{\hat{q}_j,q_i}_{\modu}
=\hat{b}_j^*\Bblacket{w_j-\sum_{l=1}^{j-1}q_l\blacket{q_l,w_j},q_i}_{\modu}\\
&=\hat{b}_j^*\left(\blacket{w_j,q_i}_{\modu}-\blacket{q_i\blacket{q_i,w_j}_{\modu},q_i}\right)
=\hat{b}_j^*\left(\blacket{w_j,q_i}_{\modu}-\blacket{w_j,q_i}_{\modu}\right)=0.
\end{align*}
Therefore, $q_1,\ldots,q_j$ are also orthogonal to each other, which completes the proof of the proposition.
\end{myproof}

In practical computations, the scheme~\eqref{eq:gram-schmidt} should be represented with matrices.
For this purpose, we derive the following QR decomposition from Proposition~\ref{prop:gram-schmidt}.
This is a generalization of the QR decomposition in Hilbert spaces.
\begin{cor}[QR decomposition]\label{prop:qr}
For $n\in\mathbb{N}$, let $W:=[w_1,\ldots,w_{n}]$ and $Q:=[q_1,\ldots,q_{n}]$.
Let $\epsilon\ge 0$.
Then, there exist $\mathbf{R},\mathbf{R}_{\opn{inv}}\in\alg^{n\times n}$ that satisfy
\begin{equation}
Q=W\mathbf{R}_{\opn{inv}},\quad
\Vert W-Q\mathbf{R}\Vert\le \epsilon.\label{eq:qr_dec}
\end{equation}
Here, $\Vert W\Vert$ for a $\alg$-linear map $W:\alg^n\to\modu$ is defined as $\Vert W\Vert:=\sup_{\Vert v\Vert_{\alg^n}=1}\Vert Wv\Vert_{\modu}$.
\end{cor}
\begin{proof}
Let $\mathbf{R}=[r_{i,j}]_{i,j}$ be an $n\times n$ $\alg$-valued matrix.
Here, $r_{i,j}$ is defined by $r_{i,j}=\blacket{q_i,w_j}_{\modu}\in\alg$ for $i<j$, $r_{i,j}=0$ for $i>j$, and $r_{j,j}=b_j$, where $b_j$ is defined as $b$ in Proposition~\ref{prop:normalized_property} by setting $\hat{q}=\hat{q}_j$.
In addition, let $\hat{\mathbf{B}}=\opn{diag}\{\hat{b}_1,\ldots,\hat{b}_{n}\}$, $\mathbf{B}=\opn{diag}\{{b}_1,\ldots,{b}_{n}\}$, and $\mathbf{R}_{\opn{inv}}=\mathbf{\hat{B}}(I+(\mathbf{R}-\mathbf{B})\mathbf{\hat{B}})^{-1}$ be $n\times n$ $\alg$-valued matrices.
The equality $Q=W\mathbf{R}_{\opn{inv}}$ is derived directly from scheme~\eqref{eq:gram-schmidt}.
In addition, by the scheme~\eqref{eq:gram-schmidt}, for $t=1,\ldots,n$, we have
\begin{align*}
w_j&=\sum_{i=1}^{j-1}q_i\blacket{q_i,w_j}_{\modu}+\hat{q}_j
=\sum_{i=1}^{j-1}q_i\blacket{q_i,w_j}_{\modu}+q_jb_j+\hat{q}_j-q_jb_j
=Q\mathbf{r}_j+\hat{q}_j-q_jb_j,
\end{align*}
where $\mathbf{r}_j\in\alg^n$ is the $i$-th column of $\mathbf{R}$.
Therefore, by Proposition~\ref{prop:normalized_property}, $\Vert w_j-Q\mathbf{r}_j\Vert_{\modu}=\Vert \hat{q}_j-q_jb_j\Vert_{\modu}\le\epsilon$ holds for $j=1,\ldots,n$, which implies $\Vert W-Q\mathbf{R}\Vert\le \epsilon$. 
\end{proof}
We call the decomposition~\eqref{eq:qr_dec} as the QR decomposition in Hilbert $C^*$-modules.
Although we are handling vectors in $\modu$, by applying the QR decomposition, we only have to compute $\mathbf{R}_{\opn{inv}}$ and $\mathbf{R}$.

We now consider estimating the Perron--Frobenius operator $K$ with observed time-series data $\{x_0,x_1,\ldots\}$. 
Let ${W}_T=[{\phi}(x_0),\ldots,{\phi}(x_{T-1})]$.
We are considering an integral operator-valued positive definite kernel (see the first part of Subsection~\ref{subsec:pf_rkhm} and the last paragraph in Section~\ref{sec:motivation}).
Since integral operators are compact, $W_T$ satisfies the assumption in Corollary~\ref{prop:gram-schmidt}.
Thus, let ${W}_T{\mathbf{R}}_{\opn{inv},T}={Q}_T$ be the QR decomposition \eqref{eq:qr_dec} of ${W}_T$ in the RKHM $\modu_k$.
The Perron--Frobenius operator $K$ is estimated by projecting ${K}$ onto the module spanned by $\{{\phi}(x_0),\ldots,\phi(x_{T-1})\}$.
We define ${\mathbf{K}}_T$ as the estimation of $K$.
Since ${K}{\phi}(x_i)={\phi}(f(x_i))={\phi}(x_{i+1})$ hold, ${\mathbf{K}}_T$ can be computed only with observed data as follows:
\begin{align*}
{\mathbf{K}}_T&= {Q}_T^*{K}{Q}_T={Q}_T^*{K}{W}_T{\mathbf{R}}_{\opn{inv},T}
={Q}_T^*[{\phi}(x_1),\ldots,{\phi}(x_{T})]{\mathbf{R}}_{\opn{inv},T}.
\end{align*}
\color{\exp}
\begin{rem}\label{rem:computation}
In practical computations, we only need to keep the integral kernels to implement the Gram--Schmidt orthonormalization algorithm and estimate Perron--Frobenius operators in the RKHM associated with the integral operator-valued kernel $k$.
Therefore, we can directly access integral kernel functions of operators, which is not achieved by vvRKHS as we stated in Remark~\ref{thm:rkhm_vvrkhs}.
Indeed, the operations required for estimating Perron--Frobenius operators are explicitly computed as follows:
Let $c,d\in\mcl{B}(L^2(\Omega))$ be integral operators whose integral kernels are $f(s,t)$ and $g(s,t)$.
Then, the integral kernels of the operator $c+d$ and $cd$ are $f(s,t)+g(s,t)$ and $\int_{r\in\Omega}f(s,r)g(r,t)dr$, respectively.
And that of $c^*$ is $f(t,s)$.
Moreover, if $c$ is positive, let $c_{\epsilon}^+$ be $\sum_{\lambda_i>\epsilon}1/\sqrt{\lambda_i}v_i{v_i}^*$, where $\lambda_i$ are eigenvalues of the compact positive operator $c$ and $v_i$ are corresponding orthonormal eigenvectors.
Then, the integral kernel of the operator $c_{\epsilon}^+$ is $\sum_{\lambda_i>\epsilon}1/\sqrt{\lambda_i}v_i(s)\overline{v_i(t)}$.
\end{rem}

\color{\exp}
\subsubsection{Numerical examples}
To show the proposed analysis with RKHMs captures continuous changes of values of kernels along functional data as we insisted in Section~\ref{sec:motivation}, we conducted experiments with river flow data of the Thames River in London\footnote{available at \url{https://nrfa.ceh.ac.uk/data/search}}.
The data is composed of daily flow at 10 stations. 
We used the data for 51 days beginning from January first, 2018.
We regard every daily flow as a function of the ratio of the distance from the most downstream station and fit it to a polynomial of degree 5 to obtain time series $x_0,\ldots,x_{50}\in C([0,1],\mathbb{R})$.
Then, we estimated the Perron--Frobenius operator which describes the time evolution of the series $x_0,\ldots,x_{50}$ in the RKHM associated with the $\mcl{B}(L^2([0,1]))$-valued positive definite kernel $k(x,y)$ defined as the integral operator whose integral kernel is $\tilde{k}(s,t)=e^{-\vert x(s)-y(t)\vert^2}$ for $x,y\in C([0,1],\mathbb{R})$.
In this case, $T=50$.
As we noted in Remark~\ref{rem:computation}, all the computations in $\alg=\mcl{B}(L^2([0,1]))$ are implemented by keeping integral kernels of operators.
Let $\mcl{F}$ be the set of polynomials of the form $x_i(s,t)=\sum_{j,l=0}^5\eta_{j,l}s^jt^l$, where $\eta_{j,l}\in\mathbb{R}$.
We project $\tilde{k}$ onto $\mcl{F}$.
Then, for $c,d\in\mcl{F}$, $c+d\in\mcl{F}$ is satisfied, but $cd\in\mcl{F}$ is not always satisfied.
Thus, we project $cd$ onto $\mcl{F}$ to restrict all the computations in $\mcl{F}$ in practice.
We computed the time-invariant term $\blacket{v,v}_{\modu_k}$ in Eq.~\eqref{eq:time_invariant_rkhm}.
Regarding the computation of eigenvectors with respect to the eigenvalue $1$, we consider the following minimization problem for the estimated Perron--Frobenius operator $\mathbf{K}_T$:
\begin{equation}
\inf_{\mathbf{v}\in\alg^T}\vert \mathbf{K}_T\mathbf{v}-\mathbf{v}\vert_{\alg^T}^2-\lambda\vert \mathbf{v}\vert_{\alg^T}^2\label{eq:eig_min}.
\end{equation}
Here, $-\lambda\vert \mathbf{v}\vert_{\alg^T}^2$ is a penalty term to keep $\mathbf{v}$ not going to $0$.
Since the objective function of the problem~\eqref{eq:eig_min} is represented as $\mathbf{v}^*(\mathbf{K}_T^*\mathbf{K}_T-\mathbf{K}_T^*-\mathbf{K}_T+(1-\lambda)\mathbf{I})\mathbf{v}$, where $\mathbf{I}$ is the identity operator on $\alg^T$, we apply the gradient descent on $\alg^T$ (see Remark~\ref{rem:gd_noncommutative}).
Figure~\ref{fig:river_rkhm} shows the heat map representing the integral kernel of $\blacket{v,v}_{\modu_k}$.

For comparison, we also applied the similar analysis in a vvRKHS.
We computed the time-invariant term $[\blacket{\tilde{v}_i,\tilde{v}_j}_{\vvrkhs}]_{i,j}$ in Eq.~\eqref{eq:time_invariant_vvrkhs} by setting $u_i$ as orthonormal polynomials of the form $u_i(s)=\sum_{j=1}^5\eta_j s^j$, where $\eta_j\in\mathbb{R}$.
Let $c_{\opn{inv}}=[\blacket{\tilde{v}_i,\tilde{v}_j}_{\vvrkhs}]_{i,j}$.
In this case, we cannot obtain the integral kernel of the time-invariant term of the operator $k(x_{\alpha},x_{\beta})$, which is denoted by $\tilde{k}_{\opn{inv}}$ here.
Instead, by approximating $k(x_{\alpha},x_{\beta})$ by $UU^*k(x_{\alpha},x_{\beta})UU^*$ and computing $Uc_{\opn{inv}}U^*\chi_{[0,t]}$, we obtain an approximation of $\int_0^t\tilde{k}_{\opn{inv}}(s,r)dr$ for $s\in[0,1]$.
Here, $\chi_{E}:[0,1]\to\{0,1\}$ is the indicator function for a Borel set $E$ on $[0,1]$.
Therefore, by numerically differentiating $Uc_{\opn{inv}}U^*\chi_{[0,t]}$ by $t$, we obtain an approximation of $\tilde{k}_{\opn{inv}}$.
Figure~\ref{fig:river_vvrkhs} shows the heat map representing the approximation of $\tilde{k}_{\opn{inv}}$.

Around the upstream stations, there are many branches and the flow is affected by them.
Thus, the similarity between flows at two points would change along time.
While, around the downstream stations, the flow is supposed not to be affected by other rivers.
Thus, the similarity between flows at two points would be invariant along time.
The values around the diagonal part of Figure~\ref{fig:river_rkhm} (RKHM) become small as $s$ and $t$ become large (as going up the river).
On the other hand, those of Figure~\ref{fig:river_vvrkhs} (vvRKHS) are also large for large $s$ and $t$.
Therefore, RKHM captures the aforementioned fact more properly.

\begin{figure}[t]
    \centering
    \subfigure[RKHM]{\includegraphics[scale=0.6]{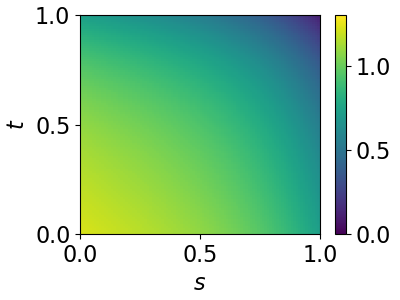}\label{fig:river_rkhm}}
    \subfigure[vvRKHS]{\includegraphics[scale=0.6]{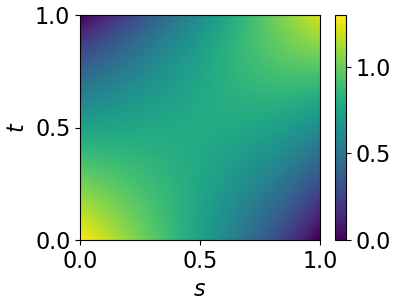}\label{fig:river_vvrkhs}}
    \caption{Heat maps representing time-invariant similarities}
\end{figure}

\color{black}
\subsection{Analysis of interaction effects}\label{subsec:interacting}
Polynomial regression is a classical problem in statistics~\citep{hastie09} and analyzing interacting effects by the polynomial regression has been investigated (for its recent improvements, see, for example, ~\citet{suzumura17}).
Most of the existing methods focus on the case of finite dimensional (discrete) data.
However, in practice, we often encounter situations where we cannot fix the dimension of data.
For example, observations are obtained at multiple locations and the locations are not fixed. It may be changed depending on time.
Therefore, analysing interaction effects of infinite dimensional (continuous) data is essential.
We show the KMEs of $\alg$-valued measures in RKHMs provide us with a method for the analysis of infinite dimensional data by setting $\alg$ as an infinite dimensional space such as $\blin{\mcl{W}}$.
Moreover, the proposed method does not need the assumption that interaction effects are described by a polynomial.
We first develop the analysis in RKHMs for the case of finite dimensional data in Subsection~\ref{subsec:interacting_finite}.
Then, we show the analysis is naturally generalized to the infinite dimensional data in Subsection~\ref{subsec:interacting_infinite}.

\paragraph{Applying $\alg$-valued measures and KME in RKHMs}
Using $\alg$-valued measures, we can describe the measure corresponding to each point of functional data as functions or operators.
For example, let $\mcl{X}$ be a locally compact Hausdorff space and let $x_1,x_2,\ldots\in C([0,1],\mcl{X})$ be samples.
Let $\alg=L^{\infty}([0,1])$ and let $\mu$ be the $\alg$-valued measure defined as $\mu(t)=\tilde{\mu}_{t}$, where $\tilde{\mu}_{t}$ is the distribution which samples $x_1(t),x_2(t),\ldots$ follow.
Then, $\mu$ describes continuous behaviors of the distribution of samples $x_1(t),x_2(t),\ldots$ with respect to $t$.
Moreover, let $\alg=\blin{L^2([0,1])}$ and let $\mu$ be the $\alg$-valued measure defined as $(\mu(E)v)(s)=\int_{t\in [0,1]}\tilde{\mu}(E)_{s,t}v(t)dt$ for a Borel set $E$, where $\tilde{\mu}_{s,t}$ is the joint distribution of the distributions which samples $x_1(s),x_2(s),\ldots$ and samples $x_1(t),x_2(t),\ldots$ follow.
Then, $\mu$ describes continuous dependencies of samples $x_1(s),x_2(s),\ldots$ and samples $x_1(t),x_2(t),\ldots$ with respect to $s$ and $t$.
Using the KME in RKHMs, we can embed $\alg$-valued measures into RKHMs, which enables us to compute inner products between $\alg$-valued measures.
Then, we can generalize algorithms in Hilbert spaces to $\alg$-valued measures.

\color{black}
\subsubsection{The case of finite dimensional data}\label{subsec:interacting_finite}
In this subsection, we assume $\alg=\mat$.
Let $\mcl{X}$ be a locally compact Hausdorff space and let $x_1,\ldots, x_n\in\mcl{X}^{m\times m}$ and $y_1,\ldots,y_n\in\alg$ be given samples. 
We assume there exist functions $f_{j,l}:\mcl{X}\to\alg$ such that 
\begin{equation*}
y_i=\sum_{j,l=1}^mf_{j,l}((x_i)_{j,l})    
\end{equation*}
for $i=1,\ldots,n$.
For example, the $(j,l)$-element of each $x_i$ describes an effect of the $l$-th element on the $j$-th element of $x_i$ and $f_{j,l}$ is a nonlinear function describing an impact of the effect to the value $y_i$.
If the given samples $y_i$ are real or complex-valued, we can regard them as $y_i1_{\alg}$ to meet the above setting.
Let $\mu_{x}\in\measuremat$ be a $\mat$-valued measure defined as $(\mu_{x})_{j,l}=\tilde{\delta}_{x_{j,l}}$, where $\tilde{\delta}_x$ for $x\in\mcl{X}$ is the standard (complex-valued) Dirac measure centered at $x$.
Note that the $(j,l)$-element of $\mu_x$ describes a measure regarding the element $x_{j,l}$.
Let $k$ be an $\alg$-valued $c_0$-kernel (see Definition~\ref{def:c0_kernel}), let $\modu_k$ be the RKHM associated with $k$, and let $\Phi$ be the KME defined in Section~\ref{subsec:kme}.
In addition, let $\mcl{V}$ be the submodule of $\modu_k$ spanned by $\{\Phi(\mu_{x_1}),\ldots,\Phi(\mu_{x_n})\}$, and let $P_f:\mcl{V}\to\mat$ be a $\mat$-linear map (see Definition~\ref{def:a_lin_op}) which satisfies
\begin{equation*}
P_f\Phi(\mu_{x_i})=\sum_{j,l=1}^m{f_{j,l}((x_i)_{j,l})} 
\end{equation*}
for $i=1,\ldots,n$.
Here, we assume the vectors $\Phi(\mu_{x_1}),\ldots,\Phi(\mu_{x_n})$ are $\mat$-linearly independent (see Definition~\ref{def:a_lin_indep}).

\subsubsection{Generalization to the continuous case}\label{subsec:interacting_infinite}
We generalize the setting mentioned in Subsection~\ref{subsec:interacting_finite} to the case of functional data.
We assume Assumption~\ref{assum:kme} in this subsection.
We set $\alg$ as $\blin{L^2[0,1]}$ instead of $\mat$ in this subsection.
Let $x_1,\ldots, x_n\in C([0,1]\times[0,1],\mcl{X})$ and $y_1,\ldots,y_n\in\alg$ be given samples. 
We assume there exists an integrable function $f:[0,1]\times [0,1]\times \mcl{X}\to\alg$ such that
\begin{equation*}
y_i=\int_0^1\int_0^1f(s,t,x_i(s,t))dsdt
\end{equation*}
for $i=1,\ldots,n$.
We consider an $\alg$-valued positive definite kernel $k$ on $\mcl{X}$, the RKHM $\modu_{k}$ associated with $k$, and the KME $\Phi$ in $\modu_{k}$.
Let $\mu_{x}\in\mcl{D}(\mcl{X},\blin{L^2([0,1])})$ be a $\blin{L^2([0,1])}$-valued measure defined as $\mu_{x}(E)v=\blacket{\chi_E(x(s,\cdot)),v}_{L^2([0,1])}$ for a Borel set $E$ on $\mcl{X}$.
Here, $\chi_E:\mcl{X}\to\{0,1\}$ is the indicator function for $E$.
\textcolor{black}{Note that $\mu_x(E)$ is an integral operator whose integral kernel is $\chi_E(x(s,t))$, which corresponds to the Dirac measure $\tilde{\delta}_{x(s,t)}(E)$.} 
Let $\mcl{V}$ be the submodule of $\modu_k$ spanned by $\{\Phi(\mu_{x_1}),\ldots,\Phi(\mu_{x_n})\}$, and let $P_f:\mcl{V}\to \blin{L^2([0,1])}$ be a $\blin{L^2([0,1])}$-linear map (see Definition~\ref{def:a_lin_op}) which satisfies
\begin{equation*}
P_f\Phi(\mu_{x_i})=\int_0^1\int_0^1f(s,t,x_i(s,t))dsdt 
\end{equation*}
for $i=1,\ldots,n$.
Here, we assume the vectors $\Phi(\mu_{x_1}),\ldots,\Phi(\mu_{x_n})$ are $\blin{L^2([0,1])}$-linearly independent (see Definition~\ref{def:a_lin_indep}).

We estimate $P_f$ by restricting it to a submodule of $\mcl{V}$.
For this purpose, we apply the PCA in RKHMs proposed in Section~\ref{subsec:pca_rkhm} and obtain principal axes $p_1,\ldots,p_r$ to construct the submodule.
We replace $\phi(x_i)$ in the problem~\eqref{eq:pca_min} with $\Phi(\mu_{x_i})$ and consider the problem
\begin{equation}
\inf_{\{p_j\}_{j=1}^r\subseteq \modu_k\mbox{\footnotesize : ONS}}\;\sum_{i=1}^n\bigg\vert \Phi(\mu_{x_i})-\sum_{j=1}^rp_j\blacket{p_j,\Phi(\mu_{x_i})}_{\modu_k}\bigg\vert_{\modu_k}^2.\label{eq:pca_min_kme}
\end{equation}
The projection operator onto the submodule spanned by $p_1,\ldots,p_r$ is represented as $QQ^*$, where $Q=[p_1,\ldots,p_r]$.
Therefore, we estimate $P_f$ by $P_fQQ^*$.
We can compute $P_fQQ^*$ as follows.
\begin{prop}\label{prop:interaction_max}
The solution of the problem~\eqref{eq:pca_min_kme} is represented as $p_j=\sum_{i=1}^n\Phi(\mu_{x_i})c_{i,j}$ for some $c_{i,j}\in\alg$.
Let $C=[c_{i,j}]_{i,j}$. Then, the estimation $P_fQQ^*$ is computed as 
\begin{equation*}
P_fQQ^*=[y_1,\ldots,y_n]CQ^*.
\end{equation*}
\end{prop}
The following proposition shows we can obtain a vector which attains the largest transformation by $P_f$.
\begin{prop}\label{prop:pf_max}
Let $u\in\modu_k$ be a unique vector satisfying for any $v\in\modu_k$, $\blacket{u,v}_{\modu_k}=P_fQQ^*v$.
For $\epsilon>0$, let $b_{\epsilon}=(\vert u\vert_{\modu_k}+\epsilon 1_{\alg})^{-1}$ and let $v_{\epsilon}=ub_{\epsilon}$.
Then, $P_fQQ^*v_{\epsilon}$ converges to
\begin{equation}
\sup_{v\in\modu_k,\ \Vert v\Vert_{\modu_k}\le 1}P_fQQ^*v\label{eq:pf_max_problem}
\end{equation}
as $\epsilon\to 0$, where the supremum is taken with respect to a (pre) order in $\alg$ (see Definition~\ref{def:sup}).
If $\alg=\mat$, then the supremum is replaced with the maximum.
In this case, let $\vert u\vert_{\modu_k}^2=a^*da$ be the eigenvalue decomposition of the positive semi-definite matrix $\vert u\vert_{\modu_k}^2$ and let $b=a^*d^+a$, where the $i$-th diagonal element of $d^+$ is $d_{i,i}^{-1/2}$ if $d_{i,i}\neq 0$ and $0$ if $d_{i,i}=0$.
Then, $ub$ is the solution of the maximization problem.
\end{prop}
\begin{proof}
By the Riesz representation theorem (Proposition~\ref{thm:riesz}), there exists a unique $u\in\modu_k$ satisfying for any $v\in\modu_k$, $\blacket{u,v}_{\modu_k}=P_fQQ^*v$.
Then, for $v\in\modu_k$ which satisfies $\Vert v\Vert_{\modu_k}=1$, by the Cauchy--Schwarz inequality (Lemma~\ref{lem:c-s}), we have
\begin{equation}
P_fQQ^*v=\blacket{u,v}_{\modu_k}\le_{\alg} \vert u\vert_{\modu_k}\Vert v\Vert_{\modu_k}\le_{\alg}\vert u\vert_{\modu_k}.\label{eq:p-fmax}
\end{equation}
The vector $v_{\epsilon}$ satisfies $\Vert v_{\epsilon}\Vert_{\modu_k}\le 1$. 
In addition, we have 
\begin{align*}
\vert u\vert_{\modu_k}^2-(\vert u\vert_{\modu_k}^2-\epsilon^21_{\alg})\ge_{\alg} 0.
\end{align*}
By multiplying $(\vert u\vert_{\modu_k}+\epsilon 1_{\alg})^{-1}$ on the both sides, we have $\blacket{u,v_{\epsilon}}_{\modu_k}+\epsilon 1_{\alg}-\vert u\vert_{\modu_k}\ge_{\alg} 0$,
which implies $\Vert \vert u\vert_{\modu_k}-\blacket{u,v_{\epsilon}}_{\modu_k}\Vert_{\alg}\le\epsilon$, and $\lim_{\epsilon\to 0}P_fQQ^*v_{\epsilon}=\lim_{\epsilon\to 0}\blacket{u,v_{\epsilon}}_{\modu_k}=\vert u\vert_{\modu_k}$.
Since $\blacket{u,v_{\epsilon}}_{\modu_k}\le_{\alg} d$ for any upper bound $d$ of $\{\blacket{u,v}_{\modu_k}\ \mid\ \Vert v\Vert_{\modu_k}\le 1\}$, $\vert u\vert_{\modu_k}\le_{\alg} d$ holds.
As a result, $\vert u\vert_{\modu_k}$ is the supremum of $P_fQQ^*v$.
In the case of $\alg=\mat$, the inequality~\eqref{eq:p-fmax} is replaced with the equality by setting $v=ub$.
\end{proof}
The vector $ub_{\epsilon}$ is represented as $ub_{\epsilon}=QC^*[y_1,\ldots,y_n]^Tb_{\epsilon}=\sum_{i=1}^n\Phi(\mu_{x_i})d_i$, where $d_i\in\alg$ is the $i$-th element of $CC^*[y_1,\ldots,y_n]^Tb_{\epsilon}\in\alg^n$, and $\Phi$ is $\alg$-linear (see Proposition~\ref{prop:kme_lin}).
Therefore, the vector $ub_{\epsilon}$ corresponds to the $\alg$-valued measure $\sum_{i=1}^n\mu_{x_i}d_i$, and if $\Phi$ is injective (see Example~\ref{ex:injective1}), the corresponding measure is unique.
This means that if we transform the samples $x_i$ according to the measure $\sum_{i=1}^n\mu_{x_i}d_i$, then the transformation makes a large impact to $y_i$.

\color{\exp}
\subsubsection{Numerical examples}
We applied our method to functional data $x_1,\ldots,x_n\in C([0,1]\times [0,1],[0,1])$, where $n=30$, $x_i$ are polynomials of the form $x_i(s,t)=\sum_{j,l=0}^5\eta_{j,l}s^jt^l$.
The coefficients $\eta_{j,l}$ of $x_i$ are randomly and independently drawn from the uniform distribution on $[0,0.1]$.
Then, we set $y_i\in \mathbb{R}$ as  
\begin{align*}
y_i=\int_0^1\int_0^1x_i(s,t)^{-\alpha+\alpha\vert s+t\vert}dsdt
\end{align*}
for $\alpha=3,0.5$.
We set $\alg=\blin{L^2([0,1])}$ and $k(x_1,x_2)=\tilde{k}(x_1,x_2)1_{\alg}$, where $\tilde{k}$ is a complex-valued positive definite kernel on $[0,1]$ defined as $\tilde{k}(x_1,x_2)=e^{-\Vert x_1-x_2\Vert_2^2}$.
We applied the PCA proposed in Subsection~\ref{subsec:pca_trace} with $r=3$, and then computed $\lim_{\epsilon\to 0} ub_{\epsilon}\in\modu_k$ in Proposition~\ref{prop:interaction_max}, which can be represented as $\Phi(\sum_{i=1}^n\mu_{x_i}d_i)$ for some $d_i\in\alg$.
The parameter $\lambda$ in the objective function of the PCA was set as $0.5$.
Figure~\ref{fig:interaction} shows the heat map representing the value related to the integral kernel of the $\alg$-valued measure $\sum_{i=1}^n\mu_{x_i}(E)d_i$ for $E=[0,0.1]$.
We denote $\sum_{i=1}^n\mu_{x_i}(E)d_i$ by $\nu(E)$ and the integral kernel of the integral operator $\nu(E)$ by $\tilde{k}_{\nu(E)}$.
As we stated in Section~\ref{subsec:interacting_infinite}, if we transform the samples $x_i$ according to the measure $\nu$, then the transformation makes a large impact to $y_i$.
Moreover, the value of $\tilde{k}_{\nu(E)}$ at $(s,t)$ corresponds to the measure at $(s,t)$.
Therefore, the value of $\tilde{k}_{\nu(E)}$ at $(s,t)$ describes the impact of the effect of $t$ on $s$ to $y_i$.
To additionally take the effect of $s$ on $t$ into consideration, we show the value of $\tilde{k}_{\nu(E)}(s,t)+\tilde{k}_{\nu(E)}(t,s)$ in Figure~\ref{fig:interaction}.
The values for $\alpha=3$ are larger than those for $\alpha=0.5$, which implies the overall impacts to $y_i$ for $\alpha=3$ are larger than that for $\alpha=0.5$.
Moreover, the value is large if $s+t$ is small.
This is because for $x_i(s,t)\in[0,0.1]$, $x_i(s,t)^{-\alpha+\alpha\vert s+t\vert}$ is large if $s+t$ is small.
Furthermore, the values around $(s,t)=(1,0)$ and $(0,1)$ are also large since $x_i$ has the form $x_i(s,t)=\sum_{j,l=0}^5\eta_{j,l}s^jt^l$ for $\eta_{j,l}\in[0,0.1]$ and $x_i(s,t)$ itself is large around $(s,t)=(1,0)$ and $(0,1)$, which results in $x_i(s,t)^{-\alpha+\alpha\vert s+t\vert}\approx x_i(s,t)$ being large.

\begin{figure}[t]
    \centering
    \subfigure[$\alpha=3$]{\includegraphics[scale=0.6]{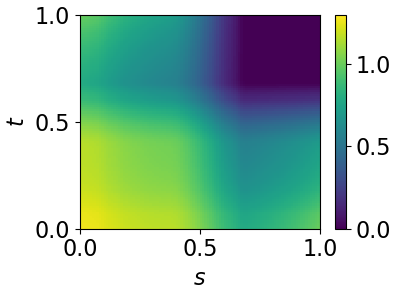}}
    \subfigure[$\alpha=0.5$]{\includegraphics[scale=0.6]{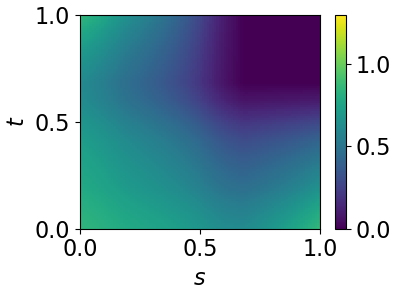}}
    \caption{Heat map representing the value the integral kernel of $\nu([0,1])$}
    \label{fig:interaction}
\end{figure}

\color{black}
\subsection{Other applications}\label{subsec:others}
\subsubsection{Maximum mean discrepancy with kernel mean embedding}\label{sec:mmd}
Maximum mean discrepancy (MMD) is a metric of measures according to the largest difference in means over a certain subset of a function space. 
It is also known as integral probability metric (IPM). 
For a set $\mcl{U}$ of real-valued bounded measurable functions on $\mcl{X}$ and two real-valued probability measures $\mu$ and $\nu$,
MMD $\gamma(\mu,\nu,\mcl{U})$ is defined as follows~\citep{muller97,gretton12}:
\begin{equation*}
\sup_{u\in\mcl{U}}\bigg\vert\int_{x\in\mcl{X}}u(x)d\mu(x)-\int_{x\in\mcl{X}}u(x)d\nu(x)\bigg\vert.
\end{equation*}
For example, if $\mcl{U}$ is the unit ball of an RKHS, denoted as $\mcl{U}_{\opn{RKHS}}$, the MMD can be represented using the KME $\tilde{\Phi}$ in the RKHS as $\gamma(\mu,\nu,\mcl{U}_{\opn{RKHS}})=\Vert\tilde{\Phi}(\mu)-\tilde{\Phi}(\nu)\Vert_{\rkhs}$.
In addition, let $\mcl{U}_{\opn{K}}=\{u\mid\ \Vert u\Vert_L\le 1\}$ and let $\mcl{U}_{\opn{D}}=\{u\mid\ \Vert u\Vert_{\infty}+\Vert u\Vert_L\le 1\}$, where, $\Vert u\Vert_L:=\sup_{x\neq y}\vert u(x)-u(y)\vert/\vert x-y\vert$, and  $\Vert u\Vert_{\infty}$ is the sup norm of $u$. 
The MMDs with $\mcl{U}_{\opn{K}}$ and $\mcl{U}_{\opn{D}}$ are also discussed in~\citet{ratchev85,dudley02,sriperumbudur12}.

Let $\mcl{X}$ be a locally compact Hausdorff space, let $\mcl{U}_{\alg}$ be a set of $\alg$-valued bounded and measurable functions, and let $\mu,\nu\in\measure$.
We generalize the MMD to that for $\alg$-valued measures as follows:
\begin{equation*}
 \gamma_{\alg}(\mu,\nu,\mcl{U}_{\alg}):=\sup_{u\in\mcl{U}_{\alg}}\bigg\vert\int_{x\in\mcl{X}}u(x)d\mu(x)-\int_{x\in\mcl{X}}u(x)d\nu(x)\bigg\vert_{\alg},
\end{equation*}
where the supremum is taken with respect to a (pre) order in $\alg$ (see Definition~\ref{def:sup}). 
Let $k$ be an $\alg$-valued positive definite kernel and let $\modu_k$ be the RKHM associated with $k$.
We assume Assumption~\ref{assum:kme}.
Let $\Phi$ be the KME defined in Section~\ref{subsec:kme}.
The following theorem shows that similar to the case of RKHS, if $\mcl{U}_{\alg}$ is the unit ball of an RKHM, the generalized MMD $\gamma_{\alg}(\mu,\nu,\mcl{U}_{\alg})$ can also be represented using the proposed KME in the RKHM.
\begin{prop}\label{prop:mmd}
Let $\mcl{U}_{\opn{RKHM}}:=\{u\in\modu_k\mid\ \Vert u\Vert_{\modu_k}\le 1\}$.
Then, for $\mu,\nu\in\mcl{D}(\mcl{X},\alg)$, we have
 \begin{equation*}
\gamma_{\alg}(\mu,\nu,\mcl{U}_{\opn{RKHM}})=\vert \Phi(\mu)-\Phi(\nu)\vert_{\modu_k}.
\end{equation*}
\end{prop}
\begin{proof}
By the Cauchy--Schwarz inequality (Lemma~\ref{lem:c-s}), we have
 \begin{align*}
  &\bigg\vert \int_{x\in\mcl{X}}d\mu^*u(x)-\int_{x\in\mcl{X}}d\nu^*u(x)\bigg\vert_{\alg}=\vert \blacket{\Phi(\mu-\nu),u}_{\modu_k}\vert_{\alg}\\
&\qquad\le_{\alg} \Vert u\Vert_{\modu_k}\vert \Phi(\mu-\nu)\vert_{\modu_k}\;\le_{\alg}\vert\Phi(\mu-\nu)\vert_{\modu_k}
 \end{align*}
for any $u\in\modu_k$ such that $\Vert u\Vert_{\modu_k}\le 1$.
Let $\epsilon>0$.
We put $v=\Phi(\mu-\nu)$ and $u_{\epsilon}=v(\vert v\vert_{\modu_k}+\epsilon 1_{\alg})^{-1}$.
In the same manner as Proposition~\ref{prop:pf_max}, $\vert \Phi(\mu-\nu)\vert_{\modu_k}$ is shown to be the supremum of $\vert \int_{x\in\mcl{X}}d\mu^*u(x)-\int_{x\in\mcl{X}}d\nu^*u(x)\vert_{\alg}$.
\end{proof}

Various methods with the existing MMD of real-valued probability measures are generalized to $\alg$-valued measures by applying our MMD. 
Using our MMD of $\alg$-valued measures instead of the existing MMD allows us to evaluate  discrepancies between measures regarding each point of structured data such as multivariate data and functional data. 
For example, the following existing methods can be generalized:

{\bf Two-sample test:}\quad In two-sample test, samples from two distributions (measures) are compared by computing the MMD of these measures~\citep{gretton12}. 

{\bf Kernel mean matching for generative models:}\quad In generative models, MMD is used in finding points whose distribution is as close as that of input points~\citep{jitkrittum19}. 

{\bf Domain adaptation:}\quad In domain adaptation, MMD is used in describing the difference between the distribution of target domain data and that of source domain data~\citep{Li19}. 

\subsubsection{Time-series data analysis with random noise}
Recently, random dynamical systems, which are (nonlinear) dynamical systems with random effects, have been extensively researched. 
Analyses of them by generalizing the discussion mentioned in Subsection~\ref{subsec:pf_rkhs_review} using the existing KME in RKHSs have been proposed~\citep{klus17,hashimoto19}.
We can apply our KME of $\alg$-valued measures to generalize the analysis proposed in Subsection~\ref{subsec:pf_rkhm} to random dynamical systems.
Then, we can extract continuous behaviors of the time evolution of functions with consideration of random noise.

\section{Connection with existing methods}\label{sec:existing}
In this section, we discuss connections between the proposed methods and existing methods.
We show the connection with the PCA in vvRKHSs in Subsection~\ref{subsec:pca_vvrkhs} and an existing notion in quantum mechanics.

\subsection{Connection with PCA in vvRKHSs}\label{subsec:pca_vvrkhs}
We show that PCA in vvRKHSs is a special case of the proposed PCA in RKHMs.
Let $\mcl{W}$ be a Hilbert space and we set $\alg=\blin{\mcl{W}}$.
Let $k:\mcl{X}\times\mcl{X}\to\blin{\mcl{W}}$ be a $\blin{\mcl{W}}$-valued positive definite kernel.
In addition, let $x_1,\ldots,x_n\in\mcl{X}$ be given data and $w_{1,1},\ldots,w_{1,N},\ldots,w_{n,1},\ldots,w_{n,N}\in\mcl{W}$ be fixed vectors in $\mcl{W}$.
The following proposition shows that we can reconstruct principal components of PCA in vvRKHSs by using the proposed PCA in RKHMs. 
\begin{prop}
Let $W_j:\mcl{X}\to\mcl{W}$ be a map satisfying $W_j(x_i)=w_{i,j}$ for $j=1,\ldots,N$, let $W=[W_1,\ldots,W_N]$,
and let $\hat{k}:\mcl{X}\times\mcl{X}\to\mathbb{C}^{N\times N}$ be defined as $\hat{k}(x,y)=W(x)^*k(x,y)W(y)$.
Let $\{q_1,\ldots,q_r\}\subseteq\mcl{F}_{\hat{k}}$ is a solution of the minimization problem
\begin{equation}
\min_{ \{q_j\}_{j=1}^r\subseteq\mcl{F}_{\hat{k}}\mbox{\footnotesize : ONS}}\;\sum_{i=1}^n\opn{tr}\Big(\big\vert \phi(x_i)-\sum_{j=1}^rq_j\blacket{q_j,\phi(x_i)}_{\modu_{\hat{k}}}\big\vert_{\modu_{\hat{k}}}^2\Big),\label{eq:pca_rkhm}
\end{equation}
where $\mcl{F}_k=\{v\in\modu_k\mid\ v(x)\mbox{ is a rank $1$ operator for any }x\in\mcl{X}\}$.
In addition, let $p_1,\ldots,p_r\in\hil_k^{\opn{v}}$ be the solution of the minimization problem
\begin{equation}
\min_{\{p_j\}_{j=1}^r\subseteq \hil_k^{\opn{v}}\mbox{\footnotesize : ONS}}\;\sum_{i=1}^{n}\sum_{l=1}^N\bigg\Vert \phi(x_i)w_{i,l}-\sum_{j=1}^{r}p_j\blacket{p_j,\phi(x_i)w_{i,l}}_{\vvrkhs}\bigg\Vert_{\vvrkhs}^2.\label{eq:pca_vvrks}
\end{equation}
Then, $\Vert (\sblacket{q_j,\hat{\phi}(x_i)}_{\modu_{\hat{k}}})_l\Vert_{\mathbb{C}^{N}}=\blacket{p_j,\phi(x_i)w_{i,l}}_{\vvrkhs}$ for $i=1,\ldots,n$, $j=1,\ldots,r$, and $l=1,\ldots,N$.
Here, $(\sblacket{q_j,\hat{\phi}(x_i)}_{\modu_{\hat{k}}})_l$ is the $l$-th column of the matrix $\sblacket{q_j,\hat{\phi}(x_i)}_{\modu_{\hat{k}}}\in\mathbb{C}^{N\times N}$.
\end{prop}
\begin{proof}
Let $\bG\in(\mathbb{C}^{N\times N})^{n\times n}$ be defined as $\bG_{i,j}=\hat{k}(x_i,x_j)$.
By Proposition~\ref{prop:min_trace}, any solution of the problem~\eqref{eq:pca_rkhm} is represented as $q_j=\sum_{i=1}^n\hat{\phi}(x_i)c_{i,j}$, where $j=1,\ldots,r$ and $[c_{1,j},\ldots,c_{n,j}]^T=\lambda_j^{-1/2}\mathbf{v}_ju^*$ for any normalized vector $u\in\mathbb{C}^{N}$.
Here, $\lambda_j$ are the largest $r$ eigenvalues and $\mathbf{v}_j$ are the corresponding orthonormal eigenvectors of the matrix $\bG$.
Therefore, by the definition of $\hat{k}$, the principal components are calculated as 
\begin{equation*}
\sblacket{q_j,\hat{\phi}(x_i)}_{\modu_{\hat{k}}}^*=\lambda_j^{-1/2}W(x_i)^*[k(x_i,x_1)W(x_1),\ldots,k(x_i,x_n)W(x_n)]\mathbf{v}_ju^*.
\end{equation*}
On the other hand, in the same manner as Proposition~\ref{prop:min_trace}, the solution of the problem~\eqref{eq:pca_vvrks} is shown to be represented as $p_j=\sum_{i=1}^n\sum_{l=1}^N\phi(x_i)w_{i,l}\alpha_{(i-1)N+l,j}$, where $j=1,\ldots,r$ and $[\alpha_{1,j},\ldots,\alpha_{Nn,j}]^T=\lambda_j^{-1/2}\mathbf{v}_j$.
Therefore, the principal components are calculated as 
\begin{equation*}
\overline{\blacket{p_j,\phi(x_i)w_{i,l}}_{\vvrkhs}}=\lambda_j^{-1/2}W_l(x_i)^*[k(x_i,x_1)W(x_1),\ldots,k(x_i,x_n)W(x_n)]\mathbf{v}_j,
\end{equation*}
which completes the proof of the proposition.
\end{proof}

\subsection{Connection with quantum mechanics}\label{subsec:quantum}
Positive operator-valued measures play an important role in quantum mechanics.
A positive operator-valued measure 
is defined as an $\alg$-valued measure $\mu$ such that $\mu(\mcl{X})=I$ and $\mu(E)$ is positive for any Borel set $E$.
It enables us to extract information of the probabilities of outcomes from a state~\citep{peres04,holevo11}.
We show that the existing inner product considered for quantum states~\citep{balkir14,prasenjit16} is generalized with our KME of positive operator-valued measures. 

Let $\mcl{X}=\mathbb{C}^m$ and $\alg=\mat$.
Let $\rho\in\alg$ be a positive semi-definite matrix with unit trace, called a density matrix.  
A density matrix describes the states of a quantum system,  
and information about outcomes is described as measure $\mu\rho\in\mcl{D}(\mcl{X},\alg)$.
We have the following proposition.
Here, we use the bra-ket notation, i.e., 
$\ket{\alpha}\in\mcl{X}$ represents a (column) vector in $\mcl{X}$, and $\bra{\alpha}$ is defined as $\bra{\alpha}=\ket{\alpha}^{*}$:
\begin{prop}\label{prop:inprod_equiv}
Assume $\mcl{X}=\mathbb{C}^m$, $\alg=\mat$, and $k:\mcl{X}\times\mcl{X}\to\alg$ is a positive definite kernel defined as $k(\ket{\alpha},\ket{\beta})=\ket{\alpha}\blackett{\alpha}{\beta}\bra{\beta}$.
If $\mu$ is represented as $\mu=\sum_{i=1}^m\delta_{\ket{\psi_i}}\ket{\psi_i}\bra{\psi_i}$ for an orthonormal basis $\{\ket{\psi_1},\ldots,\ket{\psi_m}\}$ of $\mcl{X}$, then for any $\rho_1,\rho_2\in\alg$, $\opn{tr}(\blacket{\Phi(\mu\rho_1),\Phi(\mu\rho_2)}_{\modu_k})=\blacket{\rho_1,\rho_2}_{\opn{HS}}$ holds. 
Here, $\blacket{\cdot,\cdot}_{\opn{HS}}$ is the Hilbert--Schmidt inner product.
\end{prop}
\begin{proof}
Let $M_i=\ket{\psi_i}\bra{\psi_i}$ for $i=1,\ldots,m$.
The inner product between $\Phi(\mu\rho_1)$ and $\Phi(\mu\rho_2)$ is calculated as follows:
\begin{align*}
 \blacket{\Phi(\mu\rho_1),\Phi(\mu\rho_2)}_{\modu_k}&=\int_{x\in\mcl{X}}\int_{y\in\mcl{X}}\rho_1^*\mu^*(x)k(x,y)\mu\rho_2(y)
=\sum_{i,j=1}^m\rho_1^*M_ik(\ket{\psi_i},\ket{\psi_{j}})M_j\rho_2.
\end{align*}
Since the identity $k(\ket{\psi_i},\ket{\psi_j})=M_iM_j$ holds and $\{\ket{\psi_1},\ldots,\ket{\psi_m}\}$ is orthonormal, we have $ \blacket{\Phi(\mu\rho_1),\Phi(\mu\rho_2)}_{\modu_k}=\sum_{i=1}^m\rho_1^*M_i\rho_2$.
By using the identity $\sum_{i=1}^mM_i=I$, we have 
\begin{equation*}
\opn{tr}\bigg(\sum_{i=1}^m\rho_1^*M_i\rho_2\bigg)=\opn{tr}\bigg(\sum_{i=1}^mM_i\rho_2\rho_1^*\bigg)=\opn{tr}(\rho_2\rho_1^*),
\end{equation*}
which completes the proof of the proposition.
\end{proof}
In previous studies~\citep{balkir14,prasenjit16}, the Hilbert--Schmidt inner product between density matrices was considered to represent similarities between two quantum states. 
\citet{liu18} considered the Hilbert--Schmidt inner product between square roots of density matrices.
Theorem~\ref{prop:inprod_equiv} shows that these inner products are represented via our KME in RKHMs.


\section{Conclusions and future works}\label{sec:concl}
In this paper, we proposed a new data analysis framework with RKHM and developed a KME in RKHMs for analyzing distributions.
We showed the theoretical validity for applying those to data analysis.
Then, we applied it to kernel PCA, \textcolor{black}{time-series data analysis}, and analysis of interaction effects in finite or infinite dimensional data.
RKHM is a generalization of RKHS in terms of $C^*$-algebra, and we can extract rich information about structures in data such as  functional data by using $C^*$-algebras.
For example, we can reduce multi-variable functional data to functions of single variable by considering the space of functions of single variables as a $C^*$-algebra and then by applying the proposed PCA in RKHMs.
Moreover, we can extract information of interaction effects in continuously distributed spatio data by considering the space of bounded linear operators on a function space as a $C^*$-algebra.

As future works, we will address $C^*$-algebra-valued supervised problems on the basis of the representer theorem (Theorem~\ref{thm:representation}) and apply the proposed KME in RKHMs to quantum mechanics.

\section*{Acknowledgments}
\red{We would like to thank Dr. Tomoki Mihara, whose comments improve the mathematical rigorousness of this paper.}
This work was partially supported by JST CREST Grant Number JPMJCR1913.

\appendix
\section{Proofs of the lemmas and propositions in Section~\ref{sec:rkhm_review}}
\label{ap:proof_3_1}
\subsection*{Proof of Proposition~\ref{prop:reproducing}}
(Existence) For $u,v\in\mcl{M}_k$, there exist $u_i,v_i\in\mcl{M}_{k,0}\ (i=1,2,\ldots)$ such that $v=\lim_{i\to\infty}v_i$ and $w=\lim_{i\to\infty}w_i$.
By the Cauchy-Schwarz inequality (Lemma~\ref{lem:c-s}), the following inequalities hold:
\begin{align*}
\Vert \blacket{u_i,v_i}_{\modu_k}-\blacket{u_j,v_j}_{\modu_k}\Vert_{\alg}
&\le\Vert \blacket{u_i,v_i-v_j}_{\modu_k}\Vert_{\alg}+\Vert \blacket{u_i-u_j,u_j}_{\modu_k}\Vert_{\alg}\\
&\le \Vert u_i\Vert_{\modu_k}\;\Vert v_i-v_j\Vert_{\modu_k}+\Vert u_i-u_j\Vert_{\modu_k}\;\Vert v_j\Vert_{\modu_k}\\
&\to 0\ (i,j\to\infty),
\end{align*}
which implies $\{\blacket{u_i,v_i}_{\modu_k}\}_{i=1}^{\infty}$ is a Cauchy sequence in $\mcl{A}$.
By the completeness of $\mcl{A}$, there exists a limit $\lim_{i\to\infty}\blacket{u_i,v_i}_{\modu_k}$.
\medskip

\noindent(Well-definedness) Assume there exist $u'_i,v'_i\in\mcl{M}_{k,0}\ (i=1,2,\ldots)$ such that $u=\lim_{i\to\infty}u_i=\lim_{i\to\infty}u'_i$ and $v=\lim_{i\to\infty}v_i=\lim_{i\to\infty}v'_i$.
By the Cauchy-Schwarz inequality (Lemma~\ref{lem:c-s}), we have
\begin{equation*}
\Vert\blacket{u_i,v_i}_{\modu_k}-\blacket{u'_i,v'_i}_{\modu_k}\Vert_{\alg}\le\Vert u_i\Vert_{\modu_k}\Vert v_i-v'_i\Vert_{\modu_k}+\Vert u_i-u'_i\Vert_{\modu_k}\Vert v'_i\Vert_{\modu_k}\to 0\ (i\to\infty),
\end{equation*}
which implies $\lim_{i\to\infty}\blacket{u_i,v_i}_{\modu_k}=\lim_{i\to\infty}\blacket{u'_i,v'_i}_{\modu_k}$.

\medskip
\noindent(Injectivity) For $u,v\in\mcl{M}_k$, we assume $\blacket{\phi(x),u}_{\modu_k}=\blacket{\phi(x),v}_{\modu_k}$ for $x\in\mcl{X}$.
By the linearity of $\blacket{\cdot,\cdot}_{\modu_k}$, $\blacket{p,u}_{\modu_k}=\blacket{p,v}_{\modu_k}$ holds for $p\in\mcl{M}_{k,0}$.
For $p\in\modu_{k}$, there exist $p_i\in\modu_{k,0}\ (i=1,2,\ldots)$ such that $p=\lim_{i\to\infty}p_i$.
Therefore, $\blacket{p,u-v}_{\modu_k}=\lim_{i\to\infty}\blacket{p_i,u-v}_{\modu_k}=0$.
As a result, $\blacket{u-v,u-v}_{\modu_k}=0$ holds by setting $p=u-v$, which implies $u=v$.

\subsection*{Proof of Proposition~\ref{prop:rkhm_unique}}
We define $\Psi:\mcl{M}_{k,0}\to\mcl{M}$ as an $\mcl{A}$-linear map that satisfies $\Psi(\phi(x))=\psi(x)$.
We show $\Psi$ can be extended to a unique $\mcl{A}$-linear bijection map on $\modu_{k}$ , which preserves the inner product.

\noindent(Uniqueness) The uniqueness follows by the definition of $\Psi$.

\medskip
\noindent(Inner product preservation) For $x,y\in\mcl{X}$, we have
\begin{equation*}
\blacket{\Psi(\phi(x)),\Psi(\phi(y))}_{\modu_k}=\blacket{\psi(x),\psi(y)}_{\mcl{M}}=k(x,y)=\blacket{\phi(x),\phi(y)}_{\modu_k}.
\end{equation*}
Since $\Psi$ is $\alg$-linear, $\Psi$ preserves the inner products between arbitrary $u,v\in\modu_{k,0}$.

\medskip
\noindent(Well-definedness) Since $\Phi$ preserves the inner product, if  $\{v_i\}_{i=1}^{\infty}\subseteq\mcl{M}_k$ is a Cauchy sequence, $\{\Psi(v_i)\}_{i=1}^{\infty}\subseteq\mcl{M}$ is also a Cauchy sequence.
Therefore, by the completeness of $\mcl{M}$, $\Psi$ also preserves the inner product in $\mcl{M}_k$, and for $v\in\mcl{M}_k$, $\Vert\Psi(v)\Vert_{\mcl{M}}=\Vert v\Vert_{\modu_k}$ holds.
As a result, for $v\in\mcl{M}_k$, if $v=0$, $\Vert \Psi(v)\Vert_{\mcl{M}}=\Vert v\Vert_{\modu_k}=0$ holds.
This implies $\Psi(v)=0$.

\medskip
\noindent(Injectivity) 
For $u,v\in\mcl{M}_k$, if $\Psi(u)=\Psi(v)$, then 
$0=\Vert\Psi(u)-\Psi(v)\Vert_{\mcl{M}}=\Vert u-v\Vert_{\modu_k}$ holds since $\Psi$ preserves the inner product, which implies $u=v$.

\medskip
\noindent(Surjectivity) It follows directly by the condition
$\overline{\{\sum_{i=0}^n\psi(x_i)c_i\mid\ x_i\in\mcl{X},\ c_i\in\mcl{A}\}}=\mcl{M}$.

\subsection*{Proof of Lemma~\ref{lem:pdk_equiv}}
Let $k$ be an $\alg$-valued positive definite kernel defined in Definition~\ref{def:pdk_rkhm}.
Let $w\in\mcl{W}$.
For $n\in\mathbb{N}$, $w_1,\ldots,w_n\in\mcl{W}$, let $c_i\in\mcl{B}(\mcl{W})$ be defined as $c_ih:=\blacket{w,h}_{\mcl{W}}/\blacket{w,w}_{\mcl{W}}w_i$ for $h\in\mcl{W}$.  
Since $w_i=c_iw$ holds, the following equalities are derived for $x_1,\ldots,x_n\in\mcl{X}$:
\begin{align*}
\sum_{i,j=1}^n\blacket{w_i,k(x_i,x_j)w_j}_{\mcl{W}}
&=\sum_{i,j=1}^n\blacket{c_iw,k(x_i,x_j)c_jw}_{\mcl{W}}
=\Bblacket{w,\sum_{i,j=1}^nc_i^*k(x_i,x_j)c_iw}_{\mcl{W}}.
\end{align*}
By the positivity of $\sum_{i,j=1}^nc_i^*k(x_i,x_j)c_j$, $\sblacket{w,\sum_{i,j=1}^nc_i^*k(x_i,x_j)c_jw}_{\mcl{W}}\ge 0$ holds, which implies $k$ is an operator valued positive definite kernel defined in Definition~\ref{def:pdk_vv-rkhs}.

On the other hand, let $k$ be an operator valued positive definite kernel defined in Definition~\ref{def:pdk_vv-rkhs}.
Let $v\in\mcl{W}$.
For $n\in\mathbb{N}$, $c_1,\ldots,c_n\in\alg$ and $x_1,\ldots,x_n\in\mcl{X}$, the following equality is derived:
\begin{align*}
\Bblacket{w,\sum_{i,j=1}^nc_i^*k(x_i,x_j)c_jw}_{\mcl{W}}\!\!\!\!
=\sum_{i,j=1}^n\blacket{c_iw,k(x_i,x_j)c_jw}_{\mcl{W}}.
\end{align*}
By Definition~\ref{def:pdk_vv-rkhs}, $\sum_{i,j=1}^n\blacket{c_iw,k(x_i,x_j)c_jw}_{\mcl{W}}\ge 0$ holds, which implies $k$ is an $\alg$-valued positive definite kernel defined in Definition~\ref{def:pdk_rkhm}.

\section{$\alg$-valued measure and integral}\label{subsec:vv_measure}
We introduce $\alg$-valued measure and integral in preparation for defining a KME in RKHMs.
\textcolor{black}{$\alg$-valued measure and integral are special cases of vector measure and integral~\citep{dinculeanu67,dinculeanu00}, respectively.
Here, we review these notions especially for the case of $\alg$-valued ones.}
The notions of measures and the Lebesgue integrals are generalized to $\alg$-valued.
The {\em left and right integral of an $\alg$-valued function $u$ with respect to an $\alg$-valued measure $\mu$} is defined through $\alg$-valued step functions.
\begin{defin}[$\alg$-valued measure]
Let $\varSigma$ be a $\sigma$-algebra on $\mcl{X}$.
\begin{enumerate}
\item An $\alg$-valued map $\mu:\varSigma\to\alg$ is called a {\em (countably additive) $\alg$-vaued measure} if $\mu(\bigcup_{i=1}^{\infty}E_i)=\sum_{i=1}^{\infty}\mu(E_i)$ for all countable collections $\{E_{i}\}_{i=1}^{\infty}$ of pairwise disjoint sets in $\varSigma$.
\item An $\alg$-valued measure $\mu$ is said to be finite if $\vert\mu\vert (E):=\sup\{\sum_{i=1}^n\Vert\mu(E_i)\Vert_{\alg}\mid\ n\in\mathbb{N},\ \{E_{i}\}_{i=1}^{n}\mbox{ is a finite partition of }E\in\varSigma\}<\infty$.
We call $\vert\mu\vert$ the total variation of $\mu$.
\item  An $\alg$-valued measure $\mu$ is said to be regular if for all $E\in\varSigma$ and $\epsilon>0$, there exist a compact set $K\subseteq E$ and an open set $G\supseteq E$ such that $\Vert\mu(F)\Vert_{\alg}\le\epsilon$ for any $F\subseteq G\setminus K$.
The regularity corresponds to the continuity of $\alg$-valued measures.
\item An $\alg$-valued measure $\mu$ is called a {\em Borel measure} if $\varSigma=\mcl{B}$, where $\mcl{B}$ is the Borel $\sigma$-algebra on $\mcl{X}$ ($\sigma$-algebra generated by all compact subsets of $\mcl{X}$).
\end{enumerate}
The set of all $\alg$-valued finite regular Borel measures is denoted as $\measure$.
\end{defin}
\begin{defin}[$\alg$-valued Dirac measure]\label{def:dirac}
For $x\in\mcl{X}$, we define $\delta_{x}\in\measure$ as $\delta_x(E)=1_{\alg}$ for $x\in E$ and $\delta_x(E)=0$ for $x\notin E$.
The measure $\delta_x$ is referred to as the {\em $\alg$-valued Dirac measure} at $x$.
\end{defin}
Similar to the Lebesgue integrals, an integral of an $\alg$-valued function with respect to an $\alg$-valued measure is defined through $\alg$-valued step functions.
\begin{defin}[Step function]\label{def:integral}
 An $\alg$-valued map $s:\mcl{X}\to\alg$ is called a {\em step function} if 
$s(x)=\sum_{i=1}^nc_i\chi_{E_i}(x)$
for some $n\in\mathbb{N}$, $c_i\in\alg$ and finite partition $\{E_{i}\}_{i=1}^{n}$ of $\mcl{X}$, where $\chi_E:\mcl{X}\to\{0,1\}$ is the indicator function for $E\in\mcl{B}$.
The set of all $\alg$-valued step functions on $\mcl{X}$ is denoted as $\mcl{S}(\mcl{X},\alg)$.
\end{defin}
\begin{defin}[Integrals of functions in $\simple$]
For $s\in\simple$ and $\mu\in\measure$, the {\em left and right integrals of $s$ with respect to $\mu$} are respectively defined as 
\begin{equation*}
\int_{x\in\mcl{X}}s(x)d\mu(x):=\sum_{i=1}^nc_i\mu(E_i),\quad \int_{x\in\mcl{X}}d\mu(x)s(x):=\sum_{i=1}^n\mu(E_i)c_i.
\end{equation*}
\end{defin}

As we explain below, the integrals of step functions are extended to those of ``integrable functions''.
For a real positive finite measure $\nu$, 
let $\bochner{\nu}$ be the set of all $\alg$-valued $\nu$-Bochner integrable functions on $\mcl{X}$, i.e., if $u\in\bochner{\nu}$,
there exists a sequence $\{s_i\}_{i=1}^{\infty}\subseteq\simple$ of step functions such that 
$\lim_{i\to\infty}\int_{x\in\mcl{X}}\Vert u(x)-s_i(x)\Vert_{\alg}d\nu(x)=0$~\cite[Chapter IV]{diestel84}.
Note that $u\in\bochner{\nu}$ if and only if $\int_{x\in\mcl{X}}\Vert u(x)\Vert_{\alg}d\nu(x)<\infty$, and
$\bochner{\nu}$ is a Banach $\alg$-module (i.e., a Banach space equipped with an $\alg$-module structure) with respect to the norm defined as $\Vert u\Vert_{\bochner{\nu}}=\int_{x\in\mcl{X}}\Vert u(x)\Vert_{\alg}d\nu(x)$.
\begin{defin}[Integrals of functions in $\bochner{\vert\mu\vert}$]
For $u\in\bochner{\vert\mu\vert}$, the {\em left and right integrals of $u$ with respect to $\mu$} is respectively defined as 
\begin{equation*}
\lim_{i\to\infty}\int_{x\in\mcl{X}}d\mu(x)s_i(x),\quad
\lim_{i\to\infty}\int_{x\in\mcl{X}}s_i(x)d\mu(x),   
\end{equation*}
where $\{s_i\}_{i=1}^{\infty}\subseteq\simple$ is a sequence of step functions whose $\bochner{\nu}$-limit is $u$.
\end{defin}
Note that since $\alg$ is not commutative in general, the left and right integrals do not always coincide.

There is also a stronger notion for integrability. 
An $\alg$-valued function $u$ on $\mcl{X}$ is said to be totally measurable if it is a uniform limit of a step function, i.e., 
there exists a sequence $\{s_i\}_{i=1}^{\infty}\subseteq\mcl{S}(\mcl{X},\alg)$ of step functions such that $\lim_{i\to\infty}\sup_{x\in\mcl{X}}\Vert u(x)-s_i(x)\Vert_{\alg}=0$.
We denote by $\total$ the set of all $\alg$-valued totally measurable functions on $\mcl{X}$.
Note that if $u\in\total$, then $u\in\bochner{\vert\mu\vert}$ for any $\mu\in\measure$.
%
In fact, the continuous functions in $\clch$ is totally measurable (see Definition~\ref{def:c0_function} for the definition of $\clch$).
\begin{prop}
The space $C_0(\mcl{X},\alg)$ is contained in $\mcl{T}(\mcl{X},\alg)$.
Moreover, for any real positive finite regular measure $\nu$, it is dense in $\bochner{\nu}$ with respect to $\Vert\cdot\Vert_{\bochner{\nu}}$. 
\end{prop}

For further details, refer to~\citet{dinculeanu67,dinculeanu00}.

\section{Proofs of the propositions and theorem in Section~\ref{subsec:injectivity}}\label{ap:universal}
Before proving the propositions and theorem, we introduce some definitions and show fundamental properties which are related to the propositions and theorem.
\begin{defin}[$\alg$-dual]
For a Banach $\alg$-module $\modu$, the {\em $\alg$-dual of $\modu$} is defined as $\modu':=\{f:\modu\to\alg\mid\ f\mbox{ is bounded and $\alg$-linear}\}$.
\end{defin}
Note that for a right Banach $\alg$-module $\modu$, $\modu'$ is a left Banach $\alg$-module.
\begin{defin}[Orthogonal complement]
For an $\alg$-submodule $\modu_0$ of a Banach $\alg$-module $\modu$, the {\em orthogonal complement of $\modu_0$} is defined as a closed submodule $\modu_0^{\perp}:=\bigcap_{u\in\modu_0}\{f\in\modu'\mid\ f(u)=0\}$ of $\modu'$.
In addition, for an $\alg$-submodule $\mcl{N}_0$ of $\modu'$, the {\em orthogonal complement of $\mcl{N}_0$} is defined as a closed submodule $\mcl{N}_0^{\perp}:=\bigcap_{f\in\mcl{N}_0}\{u\in\modu\mid\ f(u)=0\}$ of $\modu$.
\end{defin}
Note that for a von Neumann $\alg$-module $\modu$, by Proposition~\ref{thm:riesz}, $\modu'$ and $\modu$ are isomorphic.
%
The following lemma shows a connection between an orthogonal complement and the density property.
\begin{lem}\label{lem:orthocompequiv2}
For a Banach $\alg$-module $\modu$ and its submodule $\modu_0$,  
$\modu_0^{\perp}=\{0\}$ if 
$\modu_0$ is dense in $\modu$.
\end{lem}
\begin{proof}
We first show $\overline{\modu_0}\subseteq (\modu_0^{\perp})^{\perp}$.
Let $u\in\modu_0$. By the definition of orthogonal complements, $u\in(\modu_0^{\perp})^{\perp}$.
Since $(\modu_0^{\perp})^{\perp}$ is closed, $\overline{\modu_0}\subseteq (\modu_0^{\perp})^{\perp}$.
If $\modu_0$ is dense in $\modu$, $\modu\subseteq (\modu_0^{\perp})^{\perp}$ holds, which means $\modu_0^{\perp}=\{0\}$.
\end{proof}
Moreover, in the case of $\alg=\mat$, a generalization of the Riesz--Markov representation theorem for $\measure$ holds.
\begin{prop}[Riesz--Markov representation theorem for $\mat$-valued measures]\label{prop:representation_finitedim}
Let $\alg=\mat$.
There exists an isomorphism between $\measure$ and $\clch'$.
\end{prop}
\begin{proof}
For $f\in\clch'$, let $f_{i,j}\in{C}_0(\mcl{X},\mathbb{C})'$ be defined as $f_{i,j}(u)=(f(u1_{\alg}))_{i,j}$ for $u\in{C}_0(\mcl{X},\mathbb{C})$.
Then, by the Riesz--Markov representation theorem for complex-valued measure, there exists a unique finite complex-valued regular measure $\mu_{i,j}$ such that $f_{i,j}(u)=\int_{x\in\mcl{X}}u(x)d\mu_{i,j}(x)$.
Let $\mu(E):=[\mu_{i,j}(E)]_{i,j}$ for $E\in\mcl{B}$.
Then, $\mu\in\measure$, and we have
\begin{align*}
f(u)&=f\bigg(\sum_{l,l'=1}^m u_{l,l'}e_{l,l'}\bigg)
=\sum_{l,l'=1}^m[f_{i,j}(u_{l,l'})]_{i,j}e_{l,l'}\\
&=\sum_{l,l'=1}^m\bigg[\int_{x\in\mcl{X}}u_{l,l'}(x)d\mu_{i,j}(x)\bigg]_{i,j}e_{l,l'}
=\int_{x\in\mcl{X}}d\mu(x)u(x),
\end{align*}
where $e_{i,j}$ is an $m\times m$ matrix whose $(i,j)$-element is $1$ and all the other elements are $0$.
Therefore, if we define $h':\clch'\to\measure$ as $f\mapsto \mu$, $h'$ is the inverse of $h$, which completes the proof of the proposition.
\end{proof}

\subsection{Proofs of Propositions~\ref{thm:characteristic} and \ref{thm:characteristic2}}
To show Propositions~\ref{thm:characteristic} and \ref{thm:characteristic2}, the following lemma is used.
\begin{lem}\label{lem:injective_equiv}
 $\Phi:\mcl{D}(\mcl{X},\alg)\to\modu_k$ is injective if and only if $\blacket{\Phi(\mu),\Phi(\mu)}_{\modu_k}\neq 0$ for any nonzero $\mu\in\mcl{D}(\mcl{X},\alg)$.
\end{lem}
\begin{proof}
($\Rightarrow$) Suppose there exists a nonzero $\mu\in\mcl{D}(\mcl{X},\alg)$ such that $\blacket{\Phi(\mu),\Phi(\mu)}_{\modu_k}=0$.
Then, $\Phi(\mu)=\Phi(0)=0$ holds, and thus, $\Phi$ is not injective.
\medskip

\noindent($\Leftarrow$) Suppose $\Phi$ is not injective. Then, there exist $\mu,\nu\in\mcl{D}(\mcl{X},\alg)$ such that $\Phi(\mu)=\Phi(\nu)$ and $\mu\neq\nu$, which implies $\Phi(\mu-\nu)=0$ and $\mu-\nu\neq 0$.
\end{proof}
We now show Propositions~\ref{thm:characteristic} and \ref{thm:characteristic2}.
\begin{myproof}{Proof of Theorem~\ref{thm:characteristic}}
Let $\mu\in\mcl{D}(\mcl{X},\alg)$, $\mu\neq 0$.
We have
\begin{align*}
 \blacket{\Phi(\mu),\Phi(\mu)}&=\int_{x\in\mathbb{R}^d}\int_{y\in\mathbb{R}^d}d\mu^*(x)k(x,y)d\mu(y)\\
&=\int_{x\in\mathbb{R}^d}\int_{y\in\mathbb{R}^d}d\mu^*(x)\int_{\omega\in\mathbb{R}^d}e^{-\sqrt{-1}(y-x)^T\omega}d\lambda(\omega)d\mu(y)\\
&=\int_{\omega\in\mathbb{R}^d}\int_{x\in\mathbb{R}^d}e^{\sqrt{-1}x^T\omega}d\mu^*(x)d\lambda(\omega)\int_{y\in\mathbb{R}^d}e^{-\sqrt{-1}y^T\omega}d\mu(y)\\
&=\int_{\omega\in\mathbb{R}^d}\hat{\mu}(\omega)^*d\lambda(\omega)\hat{\mu}(\omega).
\end{align*}
Assume $\hat{\mu}=0$.
Then, $\int_{x\in\mcl{X}}u(x)d\mu(x)=0$ for any $u\in\clch$ holds, which implies $\mu\in\clch^{\perp}=\{0\}$ by Proposition~\ref{prop:representation_finitedim} and Lemma~\ref{lem:orthocompequiv2}.
Thus, ${\mu}=0$.
In addition, by the assumption, $\opn{supp}(\lambda)=\mathbb{R}^d$ holds.
As a result, $\int_{\omega\in\mathbb{R}^d}\hat{\mu}(\omega)^*d\lambda(\omega)\hat{\mu}(\omega)\neq 0$ holds.
By Lemma~\ref{lem:injective_equiv}, $\Phi$ is injective.
\end{myproof}
\begin{myproof}{Proof of Theorem~\ref{thm:characteristic2}}
Let $\mu\in\mcl{D}(\mcl{X},\alg)$, $\mu\neq 0$.
We have
\begin{align}
\blacket{\Phi(\mu),\Phi(\mu)}
&=\int_{x\in\mathbb{R}^d}\int_{y\in\mathbb{R}^d}d\mu^*(x)k(x,y)d\mu(y)\nn\\
&=\int_{x\in\mathbb{R}^d}\int_{y\in\mathbb{R}^d}d\mu^*(x)\int_{t\in[0,\infty)}e^{-t\Vert x-y\Vert^2}d\eta(t)d\mu(y)\nn\\
&=\int_{x\in\mathbb{R}^d}\int_{y\in\mathbb{R}^d}d\mu^*(x)\int_{t\in[0,\infty)}\frac{1}{(2t)^{d/2}}\int_{\omega\in\mathbb{R}^d}e^{-\sqrt{-1}(y-x)^T\omega-\frac{\Vert\omega\Vert^2}{4t}}d\omega d\eta(t)d\mu(y)\nn\\
&=\int_{\omega\in\mathbb{R}^d}\hat{\mu}(\omega)^*\int_{t\in[0,\infty)}\frac{1}{(2t)^{d/2}}e^{\frac{-\Vert\omega\Vert^2}{4t}}d\eta(t)\hat{\mu}(\omega)d\omega,\label{eq:radial}
\end{align}
where we applied a formula $e^{-t\Vert x\Vert^2}={(2t)^{-d/2}}\int_{\omega\in\mathbb{R}^d}e^{-\sqrt{-1}x^T\omega-\Vert\omega\Vert^2/(4t)}d\omega$ in the third equality.
In the same manner as the proof of Theorem~\ref{thm:characteristic}, $\hat{\mu}\neq 0$ holds.
In addition, since $\opn{supp}(\eta)\neq\{0\}$ holds, $\int_{t\in[0,\infty)}(2t)^{-d/2}e^{-\Vert\omega\Vert^2/(4t)}d\eta(t)$ is positive definite.
As a result, the last formula in Eq.~\eqref{eq:radial} is nonzero. 
By Lemma~\ref{lem:injective_equiv}, $\Phi$ is injective.
\end{myproof}

\subsection{Proofs of Proposition~\ref{thm:universal_finitedim} and Theorem~\ref{thm:universal}}
Let $\regular$ be the set of all real positive-valued regular measures, and $\bdmeasure{\nu}$ the set of all finite regular Borel $\alg$-valued measures $\mu$ whose total variations are dominated by $\nu\in\regular$ (i.e., $\vert\mu\vert\le\nu$).
We apply the following representation theorem to derive Theorem~\ref{thm:universal}.
\begin{prop}
\label{prop:represention}
For $\nu\in\regular$, 
there exists an isomorphism between $\bdmeasure{\nu}$ and $\bochner{\nu}'$.
\end{prop}
\begin{proof}
For $\mu\in\bdmeasure{\nu}$ and $u\in\bochner{\nu}$, we have
\begin{equation*}
\bigg\Vert\int_{x\in\mcl{X}}d\mu(x)u(x)\bigg\Vert_{\alg}
\le\int_{x\in\mcl{X}}\Vert u(x)\Vert_{\alg}d\vert\mu\vert(x)
\le\int_{x\in\mcl{X}}\Vert u(x)\Vert_{\alg}d\nu(x).
\end{equation*}
Thus, we define $h:\bdmeasure{\nu}\to\bochner{\nu}'$ as $\mu\mapsto (u\mapsto\int_{x\in\mcl{X}}d\mu(x)u(x))$.

Meanwhile, for $f\in\bochner{\nu}'$ and $E\in\mcl{B}$, we have
\begin{equation*}
\Vert f(\chi_E 1_{\alg})\Vert_{\alg}\le C\int_{x\in\mcl{X}}\Vert \chi_E 1_{\alg}\Vert_{\alg}d\nu(x)
=C\nu(E)
\end{equation*}
for some $C>0$ since $f$ is bounded.
Here, $\chi_E$ is an indicator function for a Borel set $E$. 
Thus, we define $h':\bochner{\nu}'\to\bdmeasure{\nu}$ as $f\mapsto(E\mapsto f(\chi_E 1_{\alg}))$.

By the definitions of $h$ and $h'$, $h(h'(f))(s)=f(s)$ holds for $s\in\simple$.
Since $\simple$ is dense in $\bochner{\nu}$, $h(h'(f))(u)=f(u)$ holds for $u\in\bochner{\nu}$.
Moreover, $h'(h(\mu))(E)=\mu(E)$ holds for $E\in\mcl{B}$.
Therefore, $\bdmeasure{\nu}$ and $\bochner{\nu}'$ are isomorphic.
\end{proof}
\begin{myproof}{Proof of Theorem~\ref{thm:universal}}
Assume $\modu_k$ is dense in $\clch$.
Since $\clch$ is dense in $\bochner{\nu}$ for any $\nu\in\regular$, $\modu_k$ is dense in $\bochner{\nu}$ for any $\nu\in\regular$.
By Proposition~\ref{lem:orthocompequiv2}, $\modu_k^{\perp}=\{0\}$ holds.
Let $\mu\in\measure$.
There exists $\nu\in\regular$ such that $\mu\in\bdmeasure{\nu}$.
By Proposition~\ref{prop:represention}, if $\int_{x\in\mcl{X}}d\mu(x)u(x)=0$ for any $u\in\modu_k$, $\mu=0$.
Since $\int_{x\in\mcl{X}}d\mu(x)u(x)=\blacket{u,\Phi(\mu)}_{\modu_k}$, $\int_{x\in\mcl{X}}d\mu(x)u(x)=0$ means $\Phi(\mu)=0$.
Therefore, by Lemma~\ref{lem:injective_equiv}, $\Phi$ is injective.
\end{myproof}

For the case of $\alg=\mat$, we apply the following extension theorem to derive the converse of Theorem~\ref{thm:universal}.
\begin{prop}[c.f. Theorem in~\cite{helemskii94}]
\label{prop:hahn_banach}
Let $\alg=\mat$. 
Let $\modu$ be a Banach $\alg$-module, $\modu_0$ be a closed submodule of $\modu$, and $f_0:\modu_0\to\alg$ be a bounded $\alg$-linear map.
Then, there exists a bounded $\alg$-linear map $f:\modu\to\alg$ that extends $f_0$ (i.e., $f(u)=f_0(u)$ for $u\in\modu_0$).
\end{prop}
\begin{proof}
Von Neumann-algebra $\alg$ itself is regarded as an $\alg$-module and is normal.
Also, $\mat$ is Connes injective.
By Theorem in~\cite{helemskii94}, $\alg$ is an injective object in the category of Banach $\alg$-module.
The statement is derived by the definition of injective objects in category theory. 
\end{proof}
%
We derive the following lemma and proposition by Proposition~\ref{prop:hahn_banach}.
\begin{lem}\label{lem:hahn_banach2}
Let $\alg=\mat$.
Let $\modu$ be a Banach $\alg$-module and $\modu_0$ be a closed submodule of $\modu$.
For $u_1\in\modu\setminus \modu_0$, there exists a bounded $\alg$-linear map $f:\modu\to\alg$ such that $f(u_0)=0$ for $u_0\in\modu_0$ and $f(u_1)\neq 0$.
\end{lem}
\begin{proof}
Let $q:\modu\to\modu/\modu_0$ be the quotient map to $\modu/\modu_0$, and $\,\mcl{U}_1:=\{q(u_1)c\mid\ c\in\alg\}$.
Note that $\modu/\modu_0$ is a Banach $\alg$-module and $\,\mcl{U}_1$ is its closed submodule.
Let $\mcl{V}:=\{c\in\alg\mid\ q(u_1)c=0\}$, which is a closed subspace of $\alg$.
Since $\mcl{V}$ is orthogonally complemented~\cite[Proposition 2.5.4]{manuilov00}, 
$\alg$ is decomposed into $\alg=\mcl{V}+\mcl{V}^{\perp}$.
Let $p:\alg\to\mcl{V}^{\perp}$ be the projection onto $\mcl{V}^{\perp}$ and
$f_0:\mcl{U}_1\to\alg$ defined as $q(u_1)c\mapsto p(c)$.
Since $p$ is $\alg$-linear, $f_0$ is also $\alg$-linear.
Also, for $c\in\alg$, we have
\begin{align*}
&\Vert q(u_1)c\Vert_{\modu/\modu_0}=\Vert q(u_1)(c_1+c_2)\Vert_{\modu/\modu_0} 
=\Vert q(u_1)c_1\Vert_{\modu/\modu_0}\\
&\qquad \ge \inf_{d\in\mcl{V}^{\perp},\Vert d\Vert_{\alg}=1}\Vert q(u_1)d\Vert_{\modu/\modu_0}\; \Vert c_1\Vert_{\alg}
=\inf_{d\in\mcl{V}^{\perp},\Vert d\Vert_{\alg}=1}\Vert q(u_1)d\Vert_{\modu/\modu_0}\; \Vert p(c)\Vert_{\alg},
\end{align*}
where $c_1=p(c)$ and $c_2=c_1-p(c)$.
Since $\inf_{d\in\mcl{V}^{\perp},\Vert d\Vert_{\alg}=1}\Vert q(u_1)d\Vert_{\modu/\modu_0}\; \Vert p(c)\Vert_{\alg}>0$, $f_0$ is bounded.
By Proposition~\ref{prop:hahn_banach}, $f_0$ is extended to a bounded $\alg$-linear map $f_1:\modu/\modu_0\to\alg$.
Setting $f:=f_1\circ q$ completes the proof of the lemma.
\end{proof}
Then we prove the converse of Lemma \ref{lem:orthocompequiv2}.
\begin{prop}\label{prop:orthcompequiv}
Let $\alg=\mat$.
For a Banach $\alg$-module $\modu$ and its submodule $\modu_0$, 
$\modu_0$ is dense in $\modu$ if $\modu_0^{\perp}=\{0\}$.
\end{prop}
\begin{proof}
Assume $u\notin\overline{\modu_0}$. 
We show $\overline{\modu_0}\supseteq (\modu_0^{\perp})^{\perp}$.
By Lemma~\ref{lem:hahn_banach2}, there exists $f\in\modu'$ such that $f(u)\neq 0$ and $f(u_0)=0$ for any $u_0\in\overline{\modu_0}$.
Thus, $u\notin (\modu_0^{\perp})^{\perp}$.
As a result, $\overline{\modu_0}\supseteq (\modu_0^{\perp})^{\perp}$.
Therefore, if $\modu_0^{\perp}=\{0\}$, then $\overline{\modu_0}\supseteq\modu$, which implies $\modu_0$ is dense in $\modu$.
\end{proof}
As a result, we derive Proposition~\ref{thm:universal_finitedim} as follows.
\begin{myproof}{Proof of Proposition~\ref{thm:universal_finitedim}}
Let $\mu\in\measure$. Then, ``$\Phi(\mu)=0$'' is equivalent to ``$\int_{x\in\mcl{X}}d\mu^*(x)u(x)=\blacket{\Phi(\mu),u}_{\modu_k}=0$ for any $u\in\modu_k$''.
Thus, by Proposition~\ref{prop:representation_finitedim}, ``$\Phi(\mu)=0\Rightarrow \mu=0$'' is equivalent to ``$f\in\clch'$, $f(u)=0$ for any $u\in\modu_k$ $\Rightarrow$ $f=0$''.
By the definition of $\modu_k^{\perp}$ and Proposition~\ref{prop:orthcompequiv}, $\modu_k$ is dense in $\clch$.
\end{myproof}

\section{Derivative on Banach spaces}\label{ap:gataux}
\begin{defin}[Fr\'{e}chet derivative]\label{def:derivative}
Let $\modu$ be a Banach space.
Let $f:\modu\to\alg$ be an $\alg$-valued function defined on $\modu$.
The function $f$ is referred to as {\em (Fr\'{e}chet) differentiable} at a point $\bc\in\modu$ if there exists a continuous $\mathbb{R}$-linear operator $l$ such that
\begin{equation*}
\lim_{u\to 0,\ u\in \modu\setminus\{0\}}\frac{\Vert f(\bc+u)-f(\bc)-l(u)\Vert_{\alg}}{\Vert u\Vert_{\modu}}=0
\end{equation*}
for any $u\in\modu$.
In this case, we denote $l$ as $Df_{\bc}$.
\end{defin}

\end{document}